\documentclass[12pt]{article}
\usepackage{amsmath}
\usepackage{graphicx}
\usepackage{enumerate}
\usepackage{hyperref}
\usepackage{natbib}
\usepackage{dsfont}
\usepackage{amsfonts} 
\usepackage{algorithm}
\usepackage{algorithmic}
\usepackage{color}
\usepackage[dvipsnames]{xcolor}
\usepackage{ulem}
\usepackage{url} % not crucial - just used below for the URL 
\usepackage{xr}
\newcommand{\blind}{1}
\newcommand{\indep}{\perp\!\!\!\!\perp} 

\definecolor{cadmiumgreen}{rgb}{0.0, 0.42, 0.24}

\usepackage[bottom]{footmisc}
\usepackage{mathrsfs}
\newcommand{\ber}{\mathrm{Ber}}

\usepackage{yk}

\begin{document}

\def\spacingset#1{\renewcommand{\baselinestretch}%
{#1}\small\normalsize} \spacingset{1}

%%%%%%%%%%%%%%%%%%%%%%%%%%%%%%%%%%%%%%%%%%%%%%%%%%%%%%%%%%%%%%%%%%%%%%%%%%%%%%

\if1\blind
{
  \title{\bf Transfer Learning Under High-Dimensional Network Convolutional Regression Model}
  % \author{Liyuan Wang 
  % % \thanks{
  % %   The authors gratefully acknowledge \textit{please remember to list all relevant funding sources in the unblinded version}} \hspace{.2cm}\\
  %   Center for Applied Statistics and School of Statistics, Renmin University of China\\
  %   and \\
  %   Author 2 \\
  %   Department of ZZZ, University of WWW}
\author{ \small
Liyuan Wang$^{1\dagger}$, Jiachen Chen$^{2\dagger}$, Kathryn L. Lunetta$^{2}$,\\
\small Danyang Huang$^{1\ast}$,  Huimin Cheng$^{2\ast}$, Debarghya Mukherjee$^{3\ast}$ \\
\small $^{1}$Center for Applied Statistics and School of Statistics, Renmin University of China \\
\small $^{2}$Department of Biostatistics, Boston University, Boston, MA, USA \\
\small $^{3}$Department of Mathematics and Statistics, Boston University, Boston, MA, USA 
}   
\date{}
  \maketitle
  
} \fi

\if0\blind
{
  \bigskip
  \bigskip
  \bigskip
  \begin{center}
    {\LARGE\bf  Transfer Learning Under High-Dimensional  Network Convolutional Regression Model}
\end{center}
  \medskip
} \fi

\begingroup
\renewcommand\thefootnote{}
\footnotetext{
	$^{\dagger}$ These authors contributed equally to this work.
	
	\quad $^{\ast}$ Co-Corresponding authors: Danyang Huang, Huimin Cheng, Debarghya Mukherjee.
}
\endgroup
\bigskip
\begin{abstract}

Transfer learning enhances model performance by utilizing knowledge from related domains, particularly when labeled data is scarce. While existing research addresses transfer learning under various distribution shifts in independent settings, handling dependencies in networked data remains challenging. To address this challenge, we propose a high-dimensional transfer learning framework based on network convolutional regression (NCR), inspired by the success of graph convolutional networks (GCNs). The NCR model incorporates random network structure by allowing each node's response to depend on its features and the aggregated features of its neighbors, capturing local dependencies effectively. Our methodology includes a two-step transfer learning algorithm that addresses domain shift between source and target networks, along with a source detection mechanism to identify informative domains. Theoretically,  we analyze the lasso estimator in the context of a random graph based on the Erdős-Rényi model assumption, demonstrating that transfer learning improves convergence rates when informative sources are present. Empirical evaluations, including simulations and a real-world application using Sina Weibo data,  demonstrate substantial improvements in prediction accuracy, particularly when labeled data in the target domain is limited. 

% show that our approach effectively captures network dependencies, manages distribution shifts, and enhances predictive accuracy, especially with limited target domain data. This work advances transfer learning for high-dimensional networked data, providing robust methods for complex dependency structures.

\end{abstract}

\noindent%
\vfill

\newpage
\spacingset{1.9} % DON'T change the spacing!
\section{Introduction}

Transfer learning has emerged as a powerful tool in modern data analysis, enabling models to leverage knowledge gained from one task or domain to enhance performance in another, particularly when labeled data is limited \citep{pan2009survey,olivas2009handbook}.  The remarkable success of Large Language Models (LLMs) exemplifies this approach, where models like GPT-3 are first pre-trained on vast amounts of unlabeled text data to learn general language representations, then fine-tuned on specific tasks with smaller datasets to achieve state-of-the-art performance across various natural language processing applications.  The efficacy of transfer learning in NLP \cite{daume2009frustratingly,pan2013transfer}, computer vision \citep{ganin2015unsupervised,zoph2018learning}, and medical diagnosis \citep{shin2016deep, zhernakova2009detecting}, further highlights its potential to address challenges posed by limited data availability in various fields.  

% Recent advancements in theoretical research within the statistics have further extended its scope \citep{li2022transfer,cai2021transfer,tian2023transfer}.  These studies provide a solid theoretical foundation for understanding the potential and limitations of transfer learning in complex tasks.} 

In statistical learning theory, transfer learning problems are often categorized based on the relationships between the source and target domains. Two primary scenarios are commonly considered:
a) \emph{Covariate shift}: when the input distribution changes between the source and target domains, but the conditional distribution of the output given the input remains the same, i.e.,  $P_S(X) \neq P_T(X)$ and $P_S(Y|X) = P_T(Y|X)$. 
b) \emph{Posterior drift}: when the conditional distribution of the output given the input changes between the source and target domains, whereas the marginal distribution remains the same, i.e.,  $P_S(Y|X) \neq P_T(Y|X)$ and $P_S(X) = P_T(X)$.  
% c) \emph{Label shift}: when the distribution of the output labels changes between the source and target domains, but the conditional distribution of the input given the output remains the same, i.e., $P_S(Y) \neq P_T(Y)$ and  $P_S(X|Y) = P_T(X|Y)$.
In practical applications,  covariate shift and posterior drift can occur simultaneously, i.e., the joint distribution of the input and the output variable changes across the domains, often termed as \emph{distribution shift}.

Extensive research has addressed transfer learning under various distribution shift scenarios. In fixed-dimensional  settings,
\citet{cai2021transfer} and \citet{reeve2021adaptive} developed minimax rate-optimal classifiers for nonparametric classification under posterior drift, while \cite{kpotufe2018marginal} focused on covariate shift. Building on these efforts, \citet{pathak2022new, maitylabelshift2022, cai2024transfer}  extended to nonparametric regression. 
% under covariate shift, posterior drift, and label shift respectively. 
In high-dimensional settings, transfer learning poses additional challenges due to the curse of dimensionality and the need for regularization to handle sparsity. \citet{li2022transfer} analyzed high-dimensional sparse linear regression under distribution shift, establishing minimax-optimal estimators for the target domain parameters, 
% \cite{bastani2018predicting} analyzed high-dimensional transfer learning, 
% \cite{bastani2018predicting} used techniques from high-dimensional statistics to construct an estimator that efficiently combines large amounts of proxy data with limited true data to improve predictive accuracy guiding decision-making in real applications. 
This framework was further extended to generalized linear models \citep{tian2023transfer,li2023estimation}, gaussian graphical models \citep{li2023transfer}, and unified through a performance gap framework \citep{wang2023gap}.

While these studies have advanced the theoretical understanding of transfer learning, they primarily focus on settings where observations are assumed to be independent and identically distributed (i.i.d.). However, various applications frequently encounter data exhibiting complex dependencies induced by network structures, where interconnected nodes demonstrate mutual influence. 
For instance, in social networks, users' decisions to adopt new technologies, share content, or make purchases can be significantly influenced by their friends \citep{bakshy2012role, chen2004impact}. These networked structures introduce three critical challenges: (1) Dependence Structure: Unlike i.i.d. data, nodes in a network are influenced by their neighbors, creating dependencies that violate the independence assumptions.
(2) Dependence Shift: When source and target networks differ, the relationships among nodes may change, leading to shifts in the network structure itself. (3) High Dimensionality: Networked data often involve high-dimensional features where the number of covariates exceeds the number of nodes, necessitating regularization techniques to manage sparsity. To the best of our knowledge, there is no existing work that can address all of these challenges.

{\bf Our contributions. } In this work, we address these challenges by developing a high-dimensional transfer learning framework under a network convolutional regression (NCR) framework (see Section \ref{sec:ncr} for details). Specifically, we propose a high-dimensional NCR model inspired by the success of graph convolutional networks (GCNs) \citep{kipf2017semi, wu2019simplifying}.  The fundamental insight of GCNs, as elucidated by \cite{wu2019simplifying}, lies in their graph convolution operations, which aggregate information from neighboring nodes to capture structural dependencies. 
% Unlike traditional regression, which assumes independence among observations, 
Motivated by GCN, our proposed NCR model incorporates network structure by allowing each node’s response  $Y_i$ to depend not only on its features $X_i$, but also on an aggregated representation of its neighbors’ features.  
% $\sum_{a_{ij}=1}X_j$, where $a_{ij}=1$ means $i$ and $j$ have a connection, otherwise $a_{ij}=0$. 
This model captures the local dependencies via neighborhood aggregation, bridging the gap between traditional regression and graph-based learning.

Under the NCR model, we propose a transfer learning algorithm to address both posterior drift (differences in conditional distributions) and dependence shift (variations in network structure) between source and target domains. Specifically, our algorithm includes two steps: (1) a transferring step that efficiently combines information across source and target domains through careful parameter estimation and (2) a debiasing step that leverages target-specific data to correct for potential bias induced by domain differences. We further enhance this framework with a source detection algorithm that automatically identifies informative source networks. This step ensures that only relevant information is transferred, avoiding the possibility of negative transfer.

We establish theoretical guarantees for our methodology in the challenging setting of high-dimensional network data. 
(1) Theoretical properties of the lasso estimator under network dependencies: We initiate our theoretical analysis by assuming an Erdős–Rényi (ER) model for network generation. Under this framework, we derive the theoretical properties of the lasso estimator for the high-dimensional NCR model. Our results characterize how network randomness influences estimation error and convergence rates, offering a rigorous foundation for regression with networked data.
(2) Transfer learning improves convergence rates: When informative source networks are available, we show that transfer learning significantly improves the convergence rates of the estimator. By leveraging shared patterns between source and target domains, the proposed algorithm achieves faster parameter recovery and enhanced predictive accuracy, providing theoretical justification for the benefits of knowledge transfer in network settings.
To empirically demonstrate the advantages of our method, we conducted extensive experiments, including simulations and a real-world application using Sina Weibo, China's largest Twitter-like platform. 
Results show that the NCR model and transfer learning algorithm effectively capture network dependencies and improve prediction accuracy.

The remainder of this paper is organized as follows. Section 2 establishes notation and formally develops the NCR model in the high-dimensional setting. Section 3 describes the proposed transfer learning framework with known and unknown transferable sources. 
Section 3 presents our transfer learning framework, including both the estimation procedure and source detection algorithm. Section 4 provides detailed theoretical analysis. Sections 5 and 6 present simulation studies and real data analysis, respectively. We conclude with a discussion of future research directions.

\section{Model and Notations}
\label{sec:meth}

\subsection{Basic Notations}\label{sec:notations}
% \subsection{Notations and problem setup}
We begin by introducing the notations that will be utilized throughout this article. Consider a network consisting of $n_0$ nodes/individuals. For each node $i$ ($1 \leq i \leq n_0$), let $Y_i \in \mathbb{R}$ denote the response, and let $X_i \in \mathbb{R}^d$ represent the covariates, which are independently and identically drawn from an unknown distribution $P_X$ with  mean $\mathbf{0}\in\mR^{d}$ and covariance matrix $\Sigma_X\in \mR^{d\times d}$, where $d$ represents the feature dimension.
We investigate the model within the \emph{high dimensional} framework, where $d$ is significantly larger than $n_0$. Accordingly, let $\by=(Y_1,\cdots,Y_{n_0})^\top\in\mathbb{R}^{n_0}$ represent the vector of responses from all individuals, and let $\bX=(X_1^\top,\cdots,X_{n_0}^\top)^\top\in\mathbb{R}^{{n_0}\times d}$ denote the matrix aggregating all covariate information. Let $A=(A_{ii_1})\in\mathbb{R}^{{n_0}\times {n_0}}$ denote the adjacency matrix  which captures the network dependence among individuals, where $A_{ii_1}=1$ if there is an edge from node $i$ to node $i_1$, and $A_{ii_1}=0$ otherwise. Although our methodology is agnostic to the specific structure of the distribution of $A_{ii_1}$, we assume $A$ is generated from an Erd{\"o}s-R{\'e}nyi (ER) model \citep{erdds1959random} to derive theoretical guarantees. Specifically, we assume $A_{ii_1} \overset{i.i.d.}{\sim}\mathrm{Ber}(p_0)$ for $i \neq i_1$, where $0 < p_0 < 1$ is the edge generation probability, and that there are no self-loops, i.e., $A_{ii} = 0$ for all $1 \le i \le {n_0}$. In the following Section \ref{sec:ncr}, we will model the relationship between  $\by$ and $\bX$, $A$.

For a positive semi-definite matrix $\Sigma \in \mathbb{R}^{d \times d}$, let $\lambda_{\max}(\Sigma)$ and $\lambda_{\min}(\Sigma)$ denote the largest and smallest eigenvalues of $\Sigma$, respectively. Let $e_j$ be a vector where the $j-$th element is 1 and all others are 0. Define $a \vee b$ as $\max \{a, b\}$ and $a \wedge b$ as $\min \{a, b\}$. Let $a_n = O\left(b_n\right)$ denote that $\left|a_n / b_n\right| \leq C$ and $a_n = O_P\left(b_n\right)$ denote that $\mathbb{P}\left(\left|a_n / b_n\right| \leq C\right) \rightarrow 1$ for some constant $C$.

\subsection{Network Convolutional Regression Model}
\label{sec:ncr}

% In this part, we first introduce the network convolutional regression model in one network data with  $n$ nodes. 
To account for the complex dependency relationships inherent in network-linked data, one prominent approach is  Graph Convolutional Neural Network (GCN) \citep{kipf2016semi}. The success of GCN lies in its ability to predict outcomes by aggregating covariate information from neighboring nodes.
Inspired by the effectiveness of GCN, we introduce a network convolutional regression model that similarly aggregates covariate information from neighboring nodes, but within a statistical regression framework.

% In the field of statistics, \citep{lunde2023conformal} proposes network-assisted regression, which considers various neighborhood effects, such as degree and neighborhood average. 
Specifically,  for each individual $i$, the {\it network convolutional features}  are defined as $\sum_{i_1}A_{ii_1}X_{i_1}\in\mR^{d}$ with $1\leq i_1\neq i\leq n_0$. We assume that the response $Y_i$ is affected by both the {\it convolutional features} and the {\it self-nodal features $X_{i}$}. Note that the network convolutional features have much more variability than the self-nodal features as it is a summation of $d_i = \sum_{ii_1} A_{ii_1}$ random vectors. 
This disparity can lead to an imbalanced influence in the regression model and unequal convergence rates for the parameter estimates associated with the network convolutional coefficients and the self-nodal coefficients.
To address this issue, we normalize the network convolutional features as $\sum_{i_1}A_{ii_1}X_{i_1}/\sqrt{(n_0-1)p_0}$ so that convolutional features and nodal features have similar variability. 
% This scaling adjusts their variance to match that of the self-nodal features  $X_i$, ensuring that the estimators for  $\beta_0$ and  $\beta_1$  have comparable convergence properties.

% Thus, it is crucial to scale them appropriately to ensure the convergence rate for corresponding parameter estimate comparable with that of the self-nodal features.  That means to make the covariance matrix of the resaling vector comparable with that of $X_i$. Then, by central limit theorem, 
%  we normalize the network convolutional features as $\sum_{i_1}a_{ii_1}X_{i_1}/\sqrt{(n_0-1)p_0}$.

Mathematically, we define the normalized adjacency matrix as $A^*=A/\sqrt{(n_0-1)p_0}$, and present the {\it network convolutional regression model} (NCR) as,
\begin{equation}
\label{eq:NRM}
\by = A^* \bX \beta_0 + \bX \beta_1 + \epsilon,
\end{equation}
where $\beta_0=(\beta_{0j})\in\mR^{p}$ is the {\it network convolutional coefficient}, $\beta_1=(\beta_{1j})\in\mR^{p}$ is the {\it self-nodal coefficient}, and $\epsilon=(\epsilon_1,\cdots,\epsilon_{n_0})^\top$ is the error term.

Defining $\bZ=(A^* \bX,\bX)\in\mR^{n_0 \times 2d}$, $\gamma = (\beta_0^\top,\beta_1^\top)^\top\in\mR^{2d}$, the regression model \eqref{eq:NRM} can be rewritten as $\by=\bZ \gamma+\epsilon$.
In the classical regime, when $d$ is assumed to be fixed and $n_0 \uparrow \infty$, one can estimate $\gamma$ by minimizing the squared error loss: 
% To estimate the model, for $d\le n_0$, we can obtain an ordinary least squares type estimator for $\gamma$ as,
 \begin{equation}
 \label{eq:OLS_est}
\hat \gamma_{\rm OLS} = \left(\bZ^\top \bZ\right)^{-1}\bZ^\top \by.
\end{equation}
However, the covariates $\{Z_i\}$ are no longer independent as they are interlinked via the (random) network matrix $A$. 
% However, note that the network structure $A$ is a random matrix. The involvement of $A$ makes the network convolutional features $\sum_{i i_1}a_{i i_1}X_{i_1}$ no longer independent. 
Consequently, the statistical properties of the estimator $\hat\gamma_{\rm OLS}$ must be carefully reestablished, serving as a foundation for understanding the behavior of the estimator in high-dimensional settings.
The following theorem shows that under mild conditions, $\hat \gamma^{\rm OLS}$ is $\sqrt{n_0}$-consistent and asymptotically normal.
% For theoretical analysis of this new estimator, it is necessary to consider the randomness of the elements in matrix $A$. In this way, we introduce the following theorem for the statistical properties of the OLS estimator for the model of \eqref{eq:NRM}.

 \begin{theorem}\label{thm:LSE}
 Assume the observation $(\bX, A^*, \by)$ are generated from the model in the form of \eqref{eq:NRM}, where
 the random noises $\epsilon_{i}$s follow sub-gaussian distribution with mean zero and covariance matrix $\sigma^2 I_{n_0}$, where $I_{n_0}\in\mR^{{n_0}\times {n_0}}$ is an identity matrix. Then we have the asymptotic distribution of the OLS estimator in \eqref{eq:OLS_est} as
 $$\sqrt{n_0}(\hat \gamma_{\rm OLS} - { \gamma }) \overset{\mathscr{L}}{\implies} \cN(0,  \sigma^2 I_2 \otimes \Sigma_{X}^{-1}).$$ 
 \end{theorem}
 \noindent
\begin{remark}
   The convergence rate of $\hat \gamma_{\rm OLS}$ is $\sqrt{n_0}$, which is the same as in the classical linear model. Notably, network density \( p \) does not appear in the convergence rate because we use the scaled version \( A^* \) instead of \( A \), effectively absorbing the influence of \( p \). This makes the theory applicable to both sparse and dense networks. The key technical distinction between proving the asymptotic normality of the standard OLS estimator and the proof of the above theorem lies in carefully handling the dependence among covariates induced by random network edges.
   The sub-gaussianity of the errors is assumed for some technical simplicity and is not necessary. 
   The proof is deferred to section S.4.4.
\end{remark}

 %\textcolor{brown}{Note that the convergence rate of $\hat \beta_{\rm OLS}$ is $n\sqrt{p}$, which is different from a typical convergence rate of $\sqrt{n}$. \textbf{(Huimin: No definition of $\hat \beta_{\rm OLS}$)}} This is because each element $\sum_{ij}a_{ij}x_{ij}$ in $A\bX$ is roughly a sum of $(n-1)p$ independent random variables. This makes the variance of each element in $A\bX$ is $n-1$ times that in $X$. 
 
%In general, we can define a {\it network convolutional regression model} (NCR), which considers both self features $\bX$ and the covariates of neighbors $A\bX$ in the model.  To make the parameters comparable, for simplicity, we consider to normalize $A\bX$ with $\sqrt{np}$, since $n/(n-1)\rightarrow1$. As a result, define $\widetilde{A}=A/\sqrt{np}$, the NCR model could be defined as,
%\begin{equation}\label{eq:GNRM}
%\by = \widetilde{A}\bX \beta_1 + \bX %\beta_0 + \epsilon,
%\end{equation}
%where $\beta_1$ is the {\it network coefficient}, $\beta_0$ is the {\it self-feature coefficient}, and $\epsilon=(\epsilon_1,\cdots,\epsilon_n)^\top$ is error term. Further define $\bZ=(\widetilde{A}\bX,\bX)\in\mR^{n\times 2d}$, $\beta^*=(\beta_1^\top,\beta_0^\top)^\top\in\mR^{2d}$. Then the model could be rewritten as $\by=\bZ\beta^*+\epsilon$. \DM{Put a comment on the normalization to balance the SNR.}

%\textbf{High-dimensional NCR}. 
In this work,  we consider a  high-dimensional setting where $d$  could be much larger than $n_0$.
% Given the limited number of observations (i.e., small $n_0$) in the target dataset, we leverage additional information from source datasets through transfer learning to improve estimation.
When $d > n_0$, it is not possible to estimate $\gamma$ consistently without further constraints. A typical assumption in the literature of high dimensional statistics is that of \emph{strong sparsity}, i.e., $\|\gamma\|_0=s\ll d$, where the $\ell_0$ norm of a vector denotes its number of non-zero elements. In other words, we assume that there are few co-ordinates ($s$-many) of $\bZ$ that actually influence the response variable.   
Inspired by the estimation procedure of LASSO \citep{tibshirani1996regression}, an $\ell_1$ penalty could be incorporated into the objective function of the least squares method to obtain an estimator as follows: 
\begin{equation}
\hat \gamma = \argmin_{\gamma} \ \left\{\frac{1}{2n_0}\left\|\mathbf{y}-\bZ \gamma \right\|_{2}^{2}+\lambda_{n_0}\left\| \gamma \right\|_{1}\right\}.
\end{equation}
Similarly to the previous discussion, the involvement of the random adjacency matrix  $A$ makes the theoretical analysis of LASSO more challenging in establishing the asymptotic properties of \( \hat{\gamma} \). One of the main technical challenges lies in establishing a condition equivalent to the standard restricted strong convexity (RSC) or restricted eigenvalue (RE) condition in the presence of network dependency (e.g., see \cite{candes2005decoding,bickel2009simultaneous,negahban2012restricted, rudelson2012reconstruction}). We will address this challenge in Section \ref{sec:theory}.

% The RE condition will be discussed in the following theoretical part. This is the first theoretical contribution that we are going to make in this work. The detailed discussion will be presented in Section 3.

% \subsection{Modelling Target and Source Network Data }
% We define the {\it similarity level} $ h = \operatorname{\max}_k \left\| \delta_k \right\|_1 $, which serves as an upper bound on the dissimilarity between the target and source tasks. While $h$ can be any positive value, a smaller $h$ generally leads to better transfer learning performance, as it suggests that the source domains provide relevant information for the target task.

 % We will discuss building upon known important source datasets $\mathcal{S}=\{1,\cdots,K\}$. Define $n_\mathcal{S}=\sum_{k\in \mathcal{S}}n_k$ to be the total sum of sample sizes across all source data.  Then detection method of important source datasets is discussed.

% As a result, we define a between $\beta^*_0$ and $\beta^*_k$ measuring the similarity between $\beta^*_0$ and $\beta^*_k$ and denoted as . Intuitively the smaller
% the magnitude of $\delta_k$, the higher the similarity between the $k$th source task and the target one. 

% In general,
% $h$ could be any positive value. However, for better transfer learning performance, the difference between the coefficients of the target and source should be reasonably small, which means
% $h$ should also be reasonably small.

\section{ Transfer Learning  Algorithms}\label{sec:alg}

\subsection{Transfer Learning  Setup Under High-Dimensional NCR}

% Now, we outline the problem setup in this paper.
We consider a transfer learning framework involving one target domain and $K$ source domains. Following notations in Table \ref{tab:notation}, we denote the sample sizes of $k$-th domain by $n_k$ for $0 \le k \le K$, with $k = 0$ denoting the target domain. 
Let $\bX_k \in \reals^{n_k \times d}, A_k \in \reals^{n_k \times n_k}$ and $\by_k \in \reals^{n_k}$ represent the design matrix, observed adjacency matrix, and response vector of $k^{th}$ domain. We \emph{do not} assume that the ER edge probabilities of all $A_k$'s are identical: $A_{k, ij} \overset{i.i.d.}{\sim} \mathrm{Ber}(p_k)$, where $p_0, p_1, \dots, p_k$ can be possibly be different from each other. The total number of samples is denoted by $n = \sum_{k=0}^K n_k$. The response $\by_k$ of the $k^{th}$ domain (for $0 \le k \le K$) is assumed to be generated from a high-dimensional NCR model: 
\begin{equation}
\label{eq:NCR_model}
\mathbf{y}_k=A^*_k \mathbf{X}_{k} \beta_{0k}+\mathbf{X}_{k} \beta_{1k} + \epsilon_k \triangleq  \mathbf{Z}_k \gamma_k +\epsilon_k.
\end{equation}
where $A^*_k = A_k/\sqrt{(n_k - 1) p_k}$,  $\beta_{0k}$ and $\beta_{1k}$ represent the network coefficient and the self-feature coefficient respectively, $\bZ_0 = (A^*_0 \mathbf{X}_{0}, \bX_0)$, and $\gamma_0=(\beta_{0k}^\top,\beta_{1k}^\top)^\top$. The errors $\{\eps_{ki}\}$ assume to be i.i.d. centered sub-gaussian random variable.
The primary objective of this work is to optimally estimate and make statistical inferences for the target network coefficient \(\beta_{00}\) and the self-feature coefficient \(\beta_{01}\), by  effectively transferring knowledge from the source domains to the target domain. For transfer learning to be effective, the target model and some of the source models must exhibit a degree of similarity. To quantify this, we introduce the concept of a {\it contrast vector}  $\delta_k= \gamma_k - \gamma_0$, which measures the difference between the coefficients of the $k$-th source task and the target task. The magnitude of $\|\delta_k\|_1$  
% {\DY$(Check the main text)} 
provides an indication of how similar a source task is to the target task: the smaller the magnitude, the greater the similarity. 
A source domain is defined as $h-$level transferable if its $\|\delta_k\|_1$ is lower than a threshold $h$.  The set of $h-$level transferable source data is $\mathcal{A}_h = \{k :  \|\delta_k\|_1 \leq h \}$.  To ease notation, we 
will abbreviate the notation $\mathcal{A}_h$
as $\mathcal{A}$ in the following analysis. 

In the following subsection, we present a method to estimate $\gamma_0$ under the assumption that the set of transferable source domains $\mathcal{A}$ is known. In the subsequent subsection, we extend this method to the case where $\mathcal{A}$ is unknown and needs to be estimated from the data.
% show the transfer learning algorithm when the set of transferable source datasets $\mathcal{A}$ is known.   We then show an algorithm for detecting transferable source domains, i.e.,  estimating $\mathcal{A}$.

\subsection{Transfer Learning Algorithm with Known Transferable Sources}

In this subsection, we present a transfer learning methodology under the NCR model \eqref{eq:NCR_model} when the transferable source set $\mathcal{A}$ is known, i.e., we have prior knowledge of which source data to utilize.  
Our goal is to leverage information from both the target and source datasets to estimate $\gamma_0$, the parameter of interest in our target domain. To achieve this, we propose a method (summarized in Algorithm \ref{al:trans_alg}), which is inspired by the recent work of \cite{li2022transfer, tian2023transfer, li2023estimation}.
Our method involves two key steps:  (1)  \textit{Transferring step}: we compute an initial estimator $\hat\gamma^\mathcal{A}$ by minimizing $\ell_1$ penalized squared error loss aggregating observations from the target domain and all informative source domains. 
% the target sample $(\bX_0,\by_0,A_0)$ and all the auxiliary source samples $(\bX_k,\by_k,A_k), k\in \mathcal{A}$. 
(2) \textit{Debiasing step}: The estimator obtained in (1) may be biased since the parameters of the source domain $\gamma_k$, while similar to those of the target domain $\gamma_0$, are not exactly the same. We then correct this bias using the target dataset only.  
% However, its probabilistic limit could be biased from the original $\gamma_0$, since we allow $\gamma_k$ to differ from $\gamma_0$, thus we need a debiasing step. 
% (2) \textit{Debiasing step}:  we correct this bias using the target dataset only.  
We refer to Algorithm \ref{al:trans_alg} as the Oracle Trans-NCR algorithm.  In what follows, we will show the details of Algorithm \ref{al:trans_alg}.

% The detailed estimation procedure is presented in Algorithm \ref{al:trans_alg}. 
\begin{algorithm}[h]
\caption{Oracle Trans-NCR}
\begin{algorithmic}
\STATE \textbf{Input}: Target data ($\by_0,\bX_0,A_0$), source data $\{\by_k,\bX_k,A_k\}_{k \in \mathcal{A}}$, and tuning parameters $\lambda_\gamma$ and $\lambda_\delta$. Set $\hat \bZ_k = (\hat{A}_k \mathbf{X}_k, \mathbf{X}_k)$, $\hat A_k = A_k/\sqrt{(n_k - 1)\hat p_k}$ where $\hat p_k$ is an estimator of $p_k$;\
\STATE \underline{Step 1.} \textbf{(Transferring Step)} Compute $\hat\gamma ^\mathcal{A}$ as
\begin{equation}
\label{eq:transferstepobjectfunction}
\hat\gamma^\mathcal{A}=\argmin_{\gamma} \ \left\{\frac{1}{2\sum_{k\in\{0\}\bigcup\mathcal {A}}n_k}\sum_{k\in\{0\}\bigcup\mathcal {A}}
\left\|\mathbf{y}_k-\hat{\mathbf{Z}}_k \gamma \right\|_{2}^{2}+\lambda_\gamma \left\|{\gamma}\right\|_{1}\right\},
\end{equation}
where $$\lambda_{\gamma} = C_1\sqrt{\frac{\log d}{n}} + C_2 h \max\left\{\frac{\sqrt{\log d \sum_k (n_k + n_k^2 p_k^2)}}{n}, \frac{\log d \max_k\left\{n_k p_k\right\}}{n}\right\}$$ with some constant $C_1$ and $C_2$.
\STATE \underline{Step 2.} \textbf{(Debiasing Step)} Compute $\hat\delta^\mathcal{A}$ as
\begin{equation}
\hat\delta^\mathcal{A}=\argmin_{\delta} \ \left\{\frac{1}{2n_0}\left\|\mathbf{y}_0-\hat{\mathbf{Z}}_0{(\hat\gamma^\mathcal{A} +\delta)} \right\|_{2}^{2} + \lambda_\delta\left\|{\delta}\right\|_{1}\right\},
\end{equation}
where $\lambda_{\delta} = C_3\sqrt{\log d/n_0}$ with some constant $C_3$.\\
% Let $\hat \gamma_0 = \hat\gamma^\mathcal{A} + \hat\delta^\mathcal{A}$.
\STATE \textbf{Output}: $\hat \gamma_0 = \hat\gamma^\mathcal{A} + \hat\delta^\mathcal{A}$.
\end{algorithmic}\label{al:trans_alg}
\end{algorithm}
In practice, $p_k$ is unknown and must be estimated from the observed data. Consequently, we use $\hat{p}_k$ to normalize the adjacency matrix.
% Specifically,  in practice, since we do not know $p_k$ in advance, we use   $\hat{p}_k$ to replace $p_k$ for normalizing $A_k$. Therefore,
% our normalized adjacency matrix is $\hat{A}_k = A_k/\sqrt{(n_k-1)\hat{p}_k}$, and the corresponding design matrix is 
% $\hat{\mathbf{Z}}_k = (\hat{A}_k \mathbf{X}_k, \mathbf{X}_k)$
% in Algorithm \ref{al:trans_alg}.
In the transferring step, we obtain an estimator $\hat\gamma^\mathcal{A}$  by combining data from the target and transferable source domains.  Under certain condition (see Section \ref{sec:theory} for details), it can be shown that $\hat\gamma^\mathcal{A}$ approximates $\gamma^\mathcal{A}$, defined through the solution of the following moment equation: 
\begin{equation}\label{eq:moment}
\mathbb{E}\left\{ \sum_{k\in\{0\}\cup \mathcal{A}} \mathbf{Z}_k^{\top} \left( \by_k - \mathbf{Z}_k \gamma^\mathcal{A} \right)  \right\} = 0. 
\end{equation}
This $\gamma^\mathcal{A}$ is generally different from the true parameter of interest, $\gamma_0$, and this difference arises because the source domains may not be perfectly aligned with the target domain. Define the bias of $\gamma^\cA$ to be $\delta^\cA = \gamma^\cA - \gamma_0$. Assuming $\var(X) = \Sigma_X$ to be the same for all the domains, we can relate $\delta^\cA$ to $\delta_k = \gamma_k - \gamma_0$ by rewriting Equation \eqref{eq:moment} as follows: 
% To quantify this bias, we define the overall bias term as $\delta^\mathcal{A} = \gamma^\mathcal{A} - \gamma_0$, which is different from $\delta_k = \gamma_k - \gamma_0$. 
% In what follows, we analyze the relationship between $\delta^\mathcal{A}$ and $\delta_k$. Equation
% \eqref{eq:moment} can be rewritten as 
$
% \begin{aligned}
\mathbb{E}\left\{ \sum_{k}  \mathbf{Z}_k^{\top} \left(\mathbf{Z}_k \delta_k  + \epsilon_k -\mathbf{Z}_k  \delta^\mathcal{A}  \right)   \right\} = 0.
% \end{aligned}
$
It could be verified that $(n^k)^{-1} \bbE\left( \bZ_k^\top \bZ_k\right) = I_2 \otimes \Sigma_X \triangleq \Sigma_{Z} $.  See Appendix S.4.4 for detailed proof. 
%$$
%\frac{1}{n_k}\bbE\left( \bZ_k^\top %\bZ_k\right) = \begin{pmatrix}
%    \Sigma_X & \b0 \\
%    \b0 &   \Sigma_X 
%\end{pmatrix} = \begin{pmatrix}
%        1 &  0 \\
%        0 & 1
%    \end{pmatrix} \otimes \Sigma_X \triangleq \Sigma_{Z} .
%$$
By $\mathbb{E} \epsilon_k = 0$ and the invertibility of the covariance matrix $\Sigma_X$, we have 
$$
\delta^\mathcal{A} = \left( \sum_{k} n_k \Sigma_{Z}\right) ^{-1} \left(\sum_{k} n_k \Sigma_{Z} \delta_k \right) = \sum_{k} \frac{n_k}{n} \delta_k.
$$
Thus, the overall bias $\delta^\mathcal{A}$ can be expressed as a weighted average of the source biases $\delta_k$,  where the weight is defined as sample size ratio $n_k /n$. Since each $\delta^\mathcal{A}$  bounded by $h$, the overall bias $\delta^\mathcal{A}$ is also controlled  by $h$.

Our debiasing step aims to correct this bias $\delta^\cA$ incurred in the transferring step. 
% We then estimate the bias term $\delta^\mathcal{A}$. 
As elaborated in Algorithm \ref{al:trans_alg}, we reparametrize the target parameter $\gamma_0$  as $\hat{\gamma}^A + \delta^\mathcal{A}$, and estimate $\delta^\mathcal{A}$ using only target data. We then adjust the initial estimator to obtain the final estimator for the target parameter: $\hat \gamma_0 = \hat\gamma^\mathcal{A} + \hat\delta^\mathcal{A}$.
Note that Algorithm \ref{al:trans_alg} involves two hyperparameters: $\lambda_\gamma$  and $\lambda_\delta$. 
We will discuss the choice of these tuning parameters in the theoretical analysis of the estimator.

% Furthermore, because that convergence of ${\hat \gamma}^\mathcal{A}$ in the first step is determined by total sample size $n$, it is relatively fast. Although Step 2 uses only the target sample, where $\delta^\mathcal{A}$ is a high-dimensional vector with $\ell_1$-norm controlled by $h$. Thus the error for Step 2 could be under controlled for a small $h$. 

\begin{figure}[h]
    \makebox[\textwidth][c]{
    \includegraphics[width=\linewidth]{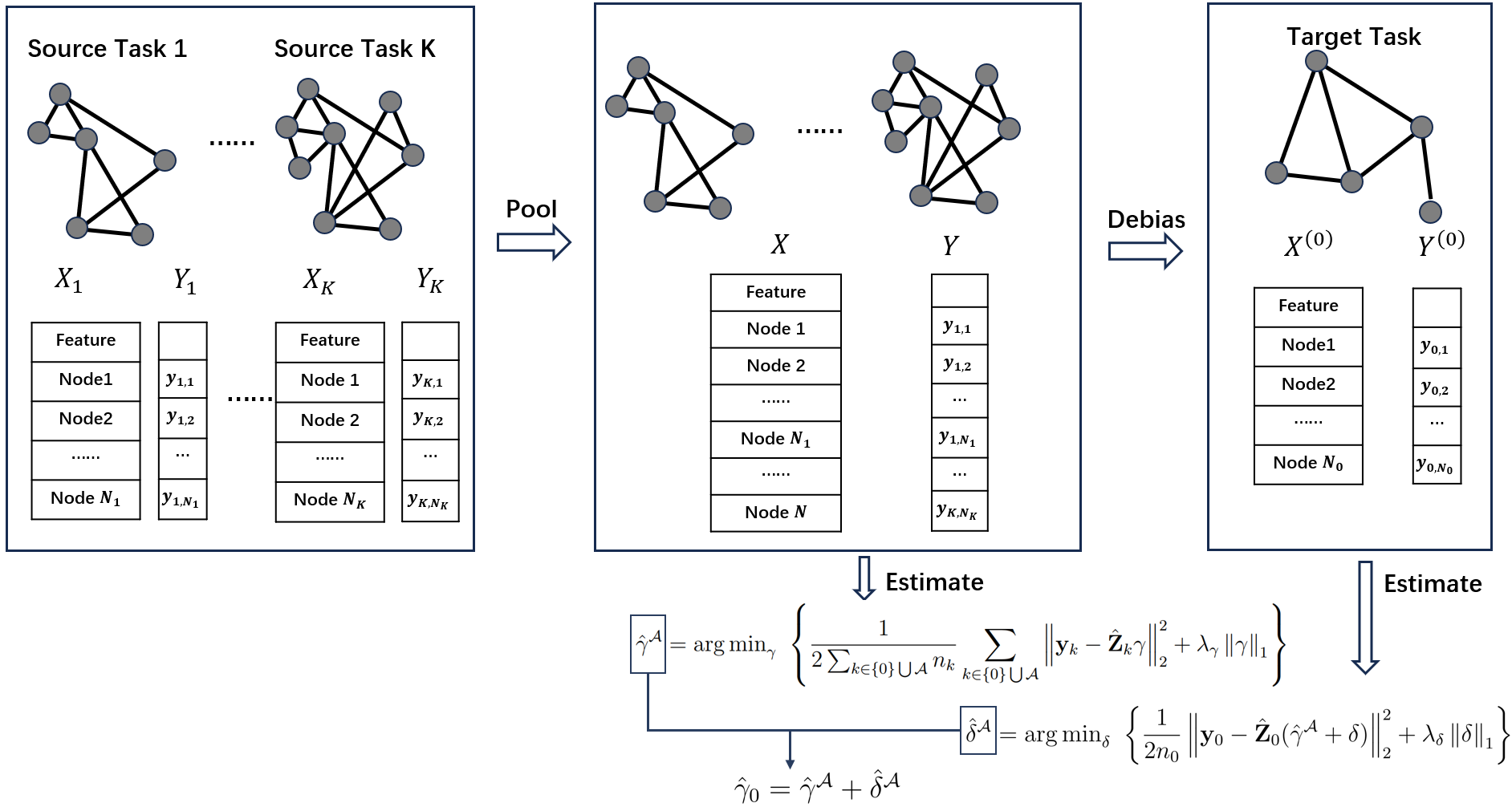}
    }
    \vspace{-50pt}
    \caption{Schematic diagram of high-dimensional NCR with known sources.}
    \label{fig:nrm_schematic}
\end{figure}
\subsection{Model Aggregation with Unknown Transferable Sources }
In the previous analysis, we assume that the set of transferable source domains $\mathcal{A}$ is already known. However, in practical scenarios, $\mathcal{A}$  can be unknown. To address this challenge, we propose the Trans-NCR algorithm, which first constructs possible transferable source domain candidates and then aggregates the results from these candidates.  
% Let $\widehat{G}_l$ denote a candidate set.  We apply Algorithm \ref{al:trans_alg} to this candidate set, treating $\widehat{G}_l$ as if it were the true transferable source dataset $\mathcal{A}$. The resulting estimator from this process is denoted as $\hat\gamma(\widehat{G}_l)$. For completeness, we define $\widehat{G}_0=\emptyset$ to include the case where no source domains are used. 

\textbf{Candidate Set Construction.} 
Here, we present how to construct a possible candidate set for transferable source domains. 
We split observations in target data into two parts: (1) $\mathcal{I}$ which is a random set of $\{1,\ldots, n_0\}$, 
and (2) \(\mathcal{I}^{c}=\{1, \ldots, n_{0}\} \backslash \mathcal{I}\). We then use data $\{\mathbf{Z}_{0, \mathcal{I}}, \mathbf{y}_{0, \mathcal{I}}\}$ and $\{\mathbf{Z}_{k}, \mathbf{y}_{k}\}_{k=1}^K$ to construct candidate sets. 
Given $K$ source domains, there are $2^K$ 
possible combinations of these domains, representing all potential candidate sets.
 Analyzing all these  $2^K$  candidates is computationally infeasible for large $K$. To overcome this,
we adopt a strategy inspired by \cite{li2022transfer}, leveraging the $\ell_1$ sparsity of the discrepancy between the source and target domains. Specifically, for each source domain $k$, we quantify the difference between its parameter $\gamma_k$ and the target parameter $\gamma_0$ using the metric 
$R_k=\|\Sigma_X\delta_k\|_2^2$, where $\delta_k=\gamma_k - \gamma_0$ and $\Sigma_X$
 A smaller 
$R_k$
  indicates that the source domain is more similar to the target domain and, thus, more transferable.
We estimate $\Sigma\delta_k$ by its sample version as
\[
\widehat\Delta_k=(n_k)^{-1}\sum_{i=1}^{n_k} \begin{pmatrix} x_{ki}y_{ki} \\ \sum_j^{n_k} A_{ij} x_{kj} y_{ki} \end{pmatrix} - |\mathcal{I}|^{-1}\sum_{i\in \mathcal{I}} \begin{pmatrix} x_{0i}y_{0i} \\ \sum_{j \in \mathcal{I} } A_{ij} x_{0j} y_{0i} \end{pmatrix},
\]
where 
$\mathcal{I}$ is a subset of the target data.
Note that $\hat{R}_k=\|\widehat\Delta_k\|_2^2$ is a sum of $2d$ random variables, and may contain significant noise due to high dimensionality.  To address this issue, we apply a sure independence screening procedure \citep{fan2008sure} to select a subset of features that exhibit the most substantial discrepancies.
For each source domain $k$, we  select the top $t_*$ features with the largest absolute values of $\widehat\Delta_{kj}$, $j=1,\ldots,{2d}$. The selected features in $k$th data 
are denoted as $\hat{T}_k$. The estimated discrepancy measure is then calculated as  $\widehat{R}_k=\left\|\widehat{\Delta}_{k\widehat{T}_{k}}\right\|_{2}^{2}$. By defining \( t^* = \lceil \gamma n^* \rceil \) with $0<\gamma<1$, the dimensionality is reduced to a manageable scale, ensuring the retention of all relevant features with high probability and satisfying the sure screening property \citep{fan2008sure}. Empirically, $t^*$ is selected to be $\lceil n^*/3  \rceil $ \citep{li2022transfer}.

Using the calculated measures \(\widehat{R}_k\), we rank the source domains in ascending order of their \(\widehat{R}_k\) values—smaller values indicate greater similarity to the target domain. To construct candidate sets without exhaustively analyzing all \(2^K\) possible combinations, we focus on the most promising domains based on this ranking.
Let \( L \) be a predefined hyperparameter that controls the maximum number of source domains included in any candidate set. We construct \( L + 1 \) candidate sets, denoted by \(\{\widehat{G}_l\}_{l=0}^{L}\), as follows: (1) Null Set (\( l = 0 \)): \(\widehat{G}_0 = \emptyset\), the set containing no source domains; (2)  Top-\( l \) Sets (\( 1 \leq l \leq L \)): For each \( l \), \(\widehat{G}_l\) consists of the indices of the top \( l \) source domains with the smallest \(\widehat{R}_k\) values. Formally,
\begin{equation}\label{eq:def_Gl}
    \widehat{G}_l=\left\{1\leq k\leq K:\hat R_k \text{ is among the first } l \text{ smallest} \right\}.
\end{equation}
 This construction ensures that each candidate set \(\widehat{G}_l\) includes the source domains most similar to the target domain up to the \( l \)-th ranked domain. 
%{Note that $\|\delta_k\|^2_2$ is a sum of $d$ random variables containing noise. To further reduce the noise magnitude, a sure independence screening procedure can be conducted to screen  the features with larger $|\hat\delta_{kj}|$, where $\hat\delta_{kj}$ is the $j$th ($1\leq j\leq d$) entry in vector $\hat\delta_k$. Then $R_k$ can be estimated based on the screened set. See Algorithm \ref{al:source_selection_alg}  for details.}

\textbf{Model Aggregation.} After constructing the candidate sets, we aim to combine models derived from these sets to enhance predictive performance.
Let $\hat{\gamma}\left(\widehat{G}_{l}\right)$ denote the estimated regression coefficients for target data obtained 
by leveraging transfer learning from the source domains in $\hat{G}_l$ for $l=1,\ldots, L$, and 
$\hat{\gamma}\left(\widehat{G}_{0}\right)$ represents the estimated regression coefficients for target data obtained by using target data only. 
% $\mathcal{H}=\{M_0, M_1, \cdots, M_L\}$ denote  a set of models, where $M_0$ is the model trained solely on the target data without any source domains, and $M_l$ is the model trained ​
%  using candidate set $\hat{G}_l$ for $l=1,\ldots, L$, obtained 
% using $\hat{G}_l$ as the source task for transfer learning. 
 To effectively combine these models, we employ the Q-aggregation method \citep{dai2012deviation}, i.e., 
$\sum_{l=0}^{L} \theta_{l} \hat{\gamma}\left(\widehat{G}_{l}\right)$, where
$\theta_l$ represents the aggregation weight for the estimate $\hat{\gamma}\left(\widehat{G}_{l}\right)$
obtained from the $l$th candidate.
These aggregation weights belong to  an $L-$dimensional simplex 
$\Theta=\left\{\theta=(\theta_0,\ldots,\theta_L)^\top\in\mathbb{R}^{L+1}:\theta_l\geq0,\sum_{l=0}^{L}\theta_l=1\right\}$. 
 
Following \cite{dai2012deviation},  we calculate the optimal weights by minimizing a penalized empirical risk over a validation set $\mathcal{I}^c$. We define the empirical risk function as
{$\hat{Q}(\mathcal{I}^c, \gamma ) = \sum_{i\in\mathcal{I}^c}\|{Y_{0,i}}-(Z_{0,i})^\intercal \gamma\|_2^2 \, ,$
where $Y_{0,i}$ and $Z_{0,i}$ represent the $i$-th row of $\mathbf{y}_0$ and $\mathbf{Z}_0$, respectively.} The optimal aggregation weights $\hat{\theta}_{l}$ are obtained by solving the following optimization problem:
\begin{equation}\label{eq:thetahat}
\hat{\theta}= \underset{\theta \in \Theta}{\arg \min }\left\{\widehat{Q}\left(\mathcal{I}^{c}, \sum_{l=0}^{L} \hat{\gamma}\left(\widehat{G}_{l}\right) \theta_{l}\right)+\sum_{l=0}^{L} \theta_{l} \widehat{Q}\left(\mathcal{I}^{c}, \hat{\gamma}\left(\widehat{G}_{l}\right)\right)+\frac{2 \lambda_{\theta}}{n_{0}}\sum_{l=0}^L \theta_l \log(\theta_l)\right\},
\end{equation}
where $\lambda_\theta$ is a regularization parameter controlling the trade-off between fit and complexity. The first term in  \eqref{eq:thetahat} represents the empirical risk of the aggregated model, whereas the second term accounts for the weighted risks of individual models. The third term  is an entropy-based regularization term that encourages balanced weighting and prevents overfitting.
The final aggregated estimate is then obtained by
$\hat{\gamma}^{\hat{\theta}}=\sum_{l=0}^{L} \hat{\theta}_{l} \hat{\gamma}\left(\widehat{G}_{l}\right).$
We summarize the procedure in  Algorithm \ref{al:source_selection_alg}, which is referred to as Trans-NCR Algorithm.

\section{Theoretical Analysis}\label{sec:theory}
\subsection{Technical Conditions}
In this section, we focus on the theoretical properties of the transfer learning method for high-dimensional NCR with known sources. 
Before presenting our theorems, we first outline the assumptions required to establish the theoretical results:
% Before theoretical investigation, the following conditions are needed.
\begin{itemize}
	\item[(C1)]\textbf{(Sub-Gaussian Features)} For each $k \in \mathcal{A} \cup\{0\}$, each row of $\mathbf{X}_k$ is an independent
	$d$-dimensional Sub-Gaussian vector with zero mean and covariance matrix $\Sigma_X$. Moreover, assume $c_1\leq \lambda_{\min}(\Sigma_X)\leq \lambda_{\max}(\Sigma_X)\leq c_2$ as $d$ goes to infinity, where $\lambda_{\min}(\Sigma_X)$ and $\lambda_{\max}(\Sigma_X)$ denote the smallest and largest eigenvalues of $\Sigma_X$, respectively, and $c_1$, $c_2$ are positive constants.
	\item[(C2)]\textbf{(Sub-Gaussian Noise)} For each $k \in \mathcal{A} \cup\{0\}$, the random noises $\epsilon_k$  follows sub-gaussian distribution with mean zero and covariance matrix $\sigma_k^2 I_k$.
	\item[(C3)](\textbf{Network Density and Feature Dimension}): The feature dimension $d$, sparsity parameter $s$, and network density $p_k$ satisfies $n_k p_k \gg \{\log n_k \vee s (\log^2 d)\}$ for $0 \leqslant k\leqslant K$, and $p_k \log{d} = o(1)$.
     { \item[(C4)] \textbf{(Network Parameter Norm)} The network effect coefficient $\beta_{0k}$ needs to satisfy $\|\beta_{0k}\|_1\lesssim\sqrt{{n}/{\sum_k \left( 1+ n_k p_k^2\right) }}$.}
% \item[(C5)] {\magenta\textbf{(Average Nodal Degree)} It is assumed that the network average nodal degree of the $k$th dataset to be 
     %$n_k p_k \gg \{\log n_k \vee s (\log^2 d)\}$ for $0 \leqslant k\leqslant K$.}
\end{itemize}
\noindent
Condition (C1) and (C2) are standard assumptions in the analysis of ultra-high-dimensional models. Although we assume the variance of $X$ to be the same across all the domains, this assumption can be relaxed with more detailed technical considerations. However, as this does not significantly contribute to understanding the core difficulty of the problem, we choose not to pursue it here. Furthermore, the sub-gaussian conditions can be relaxed by using a robust loss function (e.g., Huber loss function) instead of squared-error loss. 
% Condition (C1) is a typical assumption regarding the covariance matrices of explanatory variables, ensuring that the smallest and largest eigenvalues of the covariance matrix $\Sigma_X$ remain uniformly bounded. Regarding the noise terms, Condition (C2) states that only sub-Gaussian distributions are required, allowing for different variances across different source datasets.
Condition (C3) constrains the relationship among feature dimension $d$, sparsity $s$, network sample size $n_k$, and the network density $p_k$ of different data sources. It comprises two parts. 
%The first part aligns with the assumptions typically required by the Lasso method \citep{tibshirani1996regression, goldsmith2015lasso}. 
The first part imposes requirements on the average nodal degree.  It should exceed the logarithmic growth rate of the number of nodes, which is a typical assumption in network analysis \citep{rohe2011spectral, lei2015consistency}. 
Furthermore, it should also exceed $s \log^2{d}$. This assumption is related to the standard assumption $s\log{d} = o(n)$ in high dimensional regression literature, with an additional factor $\log{d}/p_k$ arising due to the interdependence among the observations. 
The second part imposes restrictions on the relationship between network density and feature dimension, implying that $p_k$ should be of smaller order than $(\log{d})^{-1}$. As $\log{d}$ typically grows very slowly, this assumption effectively implies $p_k \downarrow 0$. 
%Condition (C3) restricts the growth of the feature dimension, which is a common constraint in high-dimensional statistics. This constraint ensures that the feature space can grow with the sample size but that the complexity of the model (i.e., sparsity level) remains manageable. It prevents the dimensionality from growing too fast relative to the available data, which could otherwise lead to overfitting and poor generalization.}
Condition (C4) is a mild condition that allows the $\ell_1$ norm of the network effect coefficients $\beta_{0k}$ to grow with the total sample size $n$. More specifically, if the sample sizes are the same for both source and target datasets, the order is ${\{k/(1/n+p_k^2)\}}^{-1/2}$.

%Condition (C5) collectively establish the lower and upper bound on the rate at which network sparsity decays. Specifically, Condition (C5) requires that the network density decays no slower than $(s\log d)^{-1}$. ensuring that the network remains sufficiently sparse as $d$ and $s$ increases.}   
%Condition (C6) imposes that the network density  does not decay faster than $\log n_k/n_k$, as $n_k$ increases. This ensures that the network does not become overly sparse, maintaining a minimum average degree that is essential for the effective estimation of network effects. 
% This condition, along with Condition (C3), ensures the RE condition of the data matrix, which can be carefully discussed in the following theorem.
%Condition (C6) is a standard assumption commonly used in the literature 
%\cite{}. 
% Condition (C5) requires the network density to be relatively sparse as $o((s\log d)^{-1})$. This condition, along with Condition (C3), ensures the RE condition of the data matrix, which can be carefully discussed in the following theorem. Condition (C6) imposes requirements on the average degree of the network, which is diverging with a speed of $\Omega(\log n)$. Sufficiently large nodal degree ensures the effective estimation of the network effect.

Based on the above notations, we carefully established the  RSC condition that needed in the following theorem.

	%Not that condition (C3) requires the connection probability of each source task is not too small. This is to ensure that the operator norm of the adjacency matrix of each source task, which is $\left\|\widetilde{A}_k \right\|_{op} $, can be controlled by $2\sqrt{n_k p_k}$ with high probability. This is because we need to ensure $e^{\log n_k - c n_k p_k}  \to 0$.  (Check Reference)
	
\begin{theorem}[Restricted strong convexity condition]
    \label{thm:RSC}
    For any vector $u$ and $\bZ$ satisfying assumptions (C1)-(C4), we have
	$$\frac{u^\top \bZ^\top \bZ u}{n} \geqslant \kappa \|u\|_2^2 - C_5 \log{d} \sqrt{\frac{\Psi(p)}{n}} \|u\|_1 \|u\|_2$$
    with probability $1 - d^{-\alpha}$, where $\Psi(p) \sim 1/(-4p\log{p})$ for $p$ close to $0$ and $\alpha > 0$. Furthermore, if we replace $\bZ$ by $\hat \bZ$, then we have: 
    $$
    \frac{u^\top \hat \bZ^\top \hat \bZ u}{n} \geqslant \frac{\kappa}{2} \|u\|_2^2 - C_5 \log{d} \sqrt{\frac{\Psi(p)}{n}} \|u\|_1 \|u\|_2 - 3K \frac{\sqrt{2\log{d}}}{n}\|u_1\|_1^2$$
 with probability $1 - d^{-\alpha}$, where $u_1$ is the last $d$ coordinates of $u$.   
\end{theorem}
\noindent
The form of Theorem \ref{thm:RSC} is similar to Theorem 1 in \cite{raskutti2010restricted}, both of which describe the property of the quadratic form $\|\mathbf{Z} u \|_2$ of the random design matrix. Although they are not strongly convex in themselves, they exhibit a behavior similar to strong convexity as the sample size increases. The key difference, in the network scenario, is the involvement of $\Psi(p)$, which is a function of the network density $p$ in the following proof of the theorems. A larger $p$ implies a smaller gap between the neighborhood behavior of the quadratic form $\|\mathbf{Z} u \|_2$ and strong convexity. This means we can better approximate $n^{-1}\|\mathbf{Z} u \|_2$ by $\kappa \|u\|_2^2$.

\subsection{Theoretical Properties of Transfer Learning based on NCR}

With this condition, we could discuss the estimation error of the NCR model without transfer learning in the following theorem.
\begin{theorem}[Estimation Error for NCR]
    \label{thm: general model upper bound without transfer learning}
    Assume $\mathcal{A} = \emptyset$ and Conditions (C1)--(C4) hold, we take $\lambda = c_1\sqrt{\frac{\log d}{n_0}}$ for some sufficiently large positive constant $c_1$ and positive constant $c_2$, then there exists an upper bound on the estimation error
    $$\Pr\left(\left\| \hat{\gamma} - \gamma \right\|_2^2\leqslant c_1(s_1 + s_0) \frac{\log d}{n_0}\right)\ge 1-d^{-1}-n^{-1}-e^{\log n-np/c_2},$$
    where $s_0$ and $s_1$ represent the number of non-zero elements in $\beta_0$ and $\beta_1$ respectively.
\end{theorem}

\noindent
The upper bound of $\|\hat \gamma - \gamma\|$ established in Theorem \ref{thm: general model upper bound without transfer learning} does not depend on $p$ because we have appropriately scaled $A^* = A/\sqrt{(n-1)p}$ by $p$, as we have previously discussed. 
However, the convergence rate is still related to the network density $p$. Specifically, compared with the result in the classical lasso estimator for a linear regression model, the difference lies in the extra term of $-n^{-1}-e^{\log n-np/c_2}$. 
This comes from the restriction on the $\ell_2$-norm and the Frobenius norm of $A^*$. In fact, as long as the expected degree of the network satisfies Condition (C3), this term tends to zero. This leads to that the estimation error regarding $\hat{\gamma}$ is small with probability approaching 1.

%\textbf{Remark 1. (Estimation of network density).} {\red Notably, theoretically, we normalized the adjacency matrix by $p$ as \( \tilde A = A/\sqrt{np} \), where \( p \) is the unknown network density. While in real application, we estimate \( p \) by its consistency estimator as \(\hat{p}=\{n(n-1)\}^{-1}\sum_{i,j}A_{ij}\).}

\begin{theorem}[Estimation Error for Oracle-Trans-NCR]
	\label{thm: general model upper bound with transfer learning}
	Assume that Conditions (C1)--(C4) hold true. Suppose that $\mathcal{A}$ is known, we take 
\begin{equation}
\label{eq:lambda_beta}
    \lambda_{\gamma} = C_1\sqrt{\frac{\log d}n} + C_2 h \max\left\lbrace \frac{\sqrt{\log d \sum_k (n_k + n_k^2 p_k^2) }}{n} , \frac{\log d \max_k\left\lbrace n_k p_k \right\rbrace }{n}\right\rbrace
\end{equation}
and $\lambda_{\delta} = C_3 \sqrt{\frac{\log d}{n_0}}$ for some sufficiently large constants $C_1$, $C_2$ and $C_3$. Further assume $h\sqrt{\log{d} \ \Psi(p_0)} = o(1)$ and $s\lambda_\gamma^2 \log{d} \ \Psi(p_0) = o(1)$. Then, there exists an upper bound on the estimation error
	$$
	\begin{aligned}
		\Pr \left(\left\|\hat{\gamma_0} - \gamma_0 \right\|_{2}^2 \leq C\left[{s \lambda_{\gamma}^2} + \left(  h^2 \wedge \lambda_{\gamma}  h \right)  + \left(  h^2 \wedge \lambda_{\delta}  h \right) \right] \right) \geq 1- d^{-1} - \sum_k n_k^{-1} - \sum_k e^{\log{n_k} -\frac{n_k p_k}{c}}
	\end{aligned}
	$$
\end{theorem}
\noindent
This theorem shows that the upper bound of the estimation error is not only related to $d$ and $n$, but also highly related with the network density $p_k$.  It is remarkable that in Theorem \ref{thm: general model upper bound with transfer learning}, we allow the networks in multiple source tasks and the target task to have entirely different connection probabilities \( p_k \)s. As long as each  \( p_k \) satisfies the requirement in condition (C3)--(C4), the probability of $1- d^{-1} - \sum_k n_k^{-1} - \sum_k \exp\{\log{n_k} -c^{-1}(n_k p_k)\}$ tends to 1. In our empirical analysis, we could observe the network density for one of the sources dataset  is  ten times that of the target network dataset. However, the leading term of $\lambda_\gamma$ could be different due to the variations in network density, which further leads to different estimation error upper bound. Specifically, we discuss two scenarios.

	First, when all the networks of source and target datasets are extremely sparse with $\max_k\left\lbrace n_k^2 p_k^2 \right\rbrace  \log d / n = O(1) $. It is remarkable that this implies $\sum_k n_k^2 p_k^2 / n = o(1)$ due to fixed $K$. In this case, $\lambda_{\gamma}$ is dominated by the first part $\sqrt{\log d/n}$. In this way, we could simplify the estimation error into a more straightforward form, as
	$$
	\begin{aligned}
		\Pr \left( \left\|\hat{\gamma_0} - \gamma_0 \right\|_{2}^2 \leq C\left[{(s_1 + s_0) \frac{\log d}{n}} + \left(  h^2 \wedge \frac{\log d}{n_0}  h \right) \right] \right)  \geq 1- d^{-1} - \sum_k n_k^{-1} - \sum_k e^{\log{n_k} -\frac{n_k p_k}{c}}.
	\end{aligned}
	$$
 In this case, the estimation error of NCR is similar to that of the classical lasso estimator.
 	
	Second, when the connection probability of the source task network is relatively high but still satisfies the Condition (C3) with $\max_k\left\lbrace n_k^2 p_k^2 \right\rbrace  \log d / n \to \infty $, the second term of $\lambda_{\beta}$ will dominate the first term. This leads to a different estimation error upper bound with $\lambda_{\gamma} = C hn^{-1} \log d \max_k\left\lbrace n_k p_k \right\rbrace$ as,
	$$
	  \begin{aligned}
	  		\left\|\hat{\gamma_0} - \gamma_0 \right\|_{2}^2 \leq C\left[(s_1 + s_0) h^2 \left( \frac{\log d \max_k\left\lbrace n_k p_k \right\rbrace }{n}\right) ^2 + \left(  h^2 \wedge \frac{\log d}{n_0}  h \right) \right]
	  \end{aligned}
	  $$
	  of high probability.

     {It is remarkable that when the connection probabilities of the source tasks satisfying 
$$
h \max\left\lbrace \sqrt{ \frac{\sum_k (n_k + n_k^2 p_k^2)}{n}} , \sqrt{\frac{\log d }{n}} \max_k\left\lbrace n_k p_k \right\rbrace\right\rbrace = O(1),
$$
we have the same asymptotic rate of convergence for coefficient estimation in network models as in linear models. When $h$ is relatively small, the convergence rate of transfer learning estimation is $s\log d /n $, faster than $s \log d / n_0$ without transfer learning. This is precisely the advantage of transfer learning.
}

\def \our{\mbox{NCR}}
\def \bX{\textbf{X}}
\def \bA{\textbf{A}}
\def \bY{\textbf{y}}
\def \dims{d}
\def \bbeta{\mathbf{\beta}}
\def \one{\mathbf{1}}
\def \zero{\mathbf{0}}
\def \sk{\tilde{s}}
\def \mA{\mathcal{A}}
\section{Simulation Studies}
\label{sec:verify}
In this section, we evaluate the empirical performance of five methods in various scenarios: (1) Our Oracle-Trans-$\our$, which uses the oracle transferable source data for transfer learning,  as described in Algorithm \ref{al:trans_alg}. 
(2) Trans-$\our$, which uses a data-driven way to aggregate all candidate source data, assigning each source a learned weight as detailed in Algorithm \ref{al:source_selection_alg}. 
(3)  Trans-Lasso \citep{li2022transfer}, which employs transfer learning within a conventional regression framework, fitting both source and target datasets without considering the network structure.
(4) Target-only network regression lasso ($\our$), which fits the model using only target data within the network regression framework, without transfer learning.
(5) Target-only lasso (Lasso), which is a conventional regression model using only target data, without considering network structure or transfer learning.
We evaluate the performance by calculating the sum of squared estimation errors (SSE) between the estimated coefficients and the true target coefficient, i.e., $\mbox{SSE}=||\hat{\gamma}_{0} - \gamma_{0}||_F^2$, where  $\gamma_0 =(\beta^{\top}_{00}, \beta^{\top}_{10})^{\top}$. All experiments are replicated 100 times to calculate the averaged SSE.

\subsection{Simulation Setup}
% We first show the general process to generate simulation data. 
\textbf{Data generation.}
We consider ten candidate source domains ($k \in \{1, \ldots, K\}, K=10$), and one target domain ($k=0$). 
 % Let $k=0$ denote the target data and $k \in \{1, \ldots, 10\}$ denote the source data.
 We generate each $k$th simulation data  as follows. 
(1) Node Features: Generate 
$\bX_k \in \mathbb{R}^{n_k \times 500}$ from i.i.d. Gaussian with mean zero and 
 a covariance matrix $\Sigma_X$, where the $ij$th element in $\Sigma_X$ is $0.8^{|i-j|}$, $i,j\in\{1,\ldots,500\}$, implying an exponential correlation structure.
(2) Adjacency Matrix:  Generate adjacency matrix $A_k \in \{0,1\}^{n_k \times n_k}$ considering two different random graph models: 
% \begin{itemize}
 ER model with parameter $p_k$, i.e., $A_{k,ij} \sim \Ber(p_k)$.
 Two-block balanced stochastic block model (SBM),  i.e., $A_{k,ij} \sim \Ber(p^{(k)}_{in})$ if $i$ and $j$ are in the same community, otherwise $A_{k,ij} \sim \Ber(p^{(k)}_{out})$.
% \end{itemize}
We  normalize $A_k$ as  $\hat{A}_k=\frac{A_k}{\sqrt{(n-1) \hat{p}_k}}$, where $\hat{p}_k=\frac{\sum_{i<j} A_{k,ij}}{n_1(n_k-1)/2}$ is the network density. 
(3) Response Generation:   Generate  $\bY_k \in \mathbb{R}^{n_k}$ using \eqref{eq:NRM}, i.e., 
 $\mathbf{y}_k=A^*_k \mathbf{X}_{k} \beta_{0k}+\mathbf{X}_{k} \beta_{1k} + \epsilon_k$, where
 where  $\epsilon_k$ are from i.i.d normal distribution $\cN(0, 1)$.
Following the setting in  \cite{li2022transfer}, 
 we set the coefficient vector in the target data as
$\bbeta_{00}=(0.3 \cdot \one_{16}, \zero_{484})$,  $\bbeta_{10}=(0.4 \cdot \one_{16}, \zero_{484})$,
where $\one_{16}$
has all $16$ elements 1 and $\zero_{484}$ has all $484$ elements 0. For $k$th source data, we set 
$\bbeta_{0k}=\bbeta_{1k} = \bbeta_{00} - (\delta_k \cdot \one_{8}, \zero_{488} )$.

% where $\delta_k$ denotes the difference between the $j$th true coefficient of the target data and that of the $k$th source data, $k=1,\ldots, K$.

 % We assume $\beta_{0k}$ and $\beta_{1k}$ are sparse, and the number of nonzero elements of $\beta_{0k}$ and $\beta_{1k}$  are both $50$. 

  % We set the number of non-zero covariates in all  dataset as  $s_0=s_1=\ldots=s_K=50$.

% First, we show the specific parameter setup for the target data. 

% Following the setting in  \cite{li2022transfer}, 
% we set non-zero  model coefficients for the first $s_k$ covariates. In particular,  
% For the $k$th source data, we let $\bbeta_{0k}=\bbeta^{(0)}_{0j}-h_k$, $\textcolor{red}{\beta^{(k)}_{1j}=\beta^{(0)}_{1j}-h_k}$, $j=1,\ldots, \textcolor{red}{s2}$,

\textbf{Transferable source set.} 
Among these ten candidate source domains, we set the first five as transferable source domains, i.e., $\mA=\{1,\ldots,5\}$, and the last five as non-transferable source domains. 
Specifically,
for the first five sources, we set 
$\delta_1=\ldots=\delta_5=\delta$ where $\delta$ is a small value and we will consider varying $\delta$ to study its effect.  We consider a large domain shift for the subsequent sources, i.e., $\delta_6=\ldots=\delta_{10}=10$.

\subsection{Simulation Results Under ER Random Graph Model}\label{sec:simu_er}

In this simulation,  
% we set
% target sample size  $n_0=150$, the ER connecting probability $p_0=0.05$. We let  all the source domains have the same sample size, $n_1=\ldots=n_K=n$, and the same ER connecting probability, $n_1=\ldots=n_K=p$.
we aim to answer three key questions: 
(1) Impact of Source Sample Size: What's the performance of our method as the sample size in the source domain increases?  
(2) Impact of Domain Shift:  What's the performance of our method when the domain shift between the transferable source set and the target data increases?
(3) Impact of Network Parameters: What's the performance of our method when the source and target networks are generated from ER models with differing parameters?
% (4) Comparison between 
% Trans-$\our$ and Oracle-Trans-$\our$: Does Trans-$\our$  achieve comparable performance as Oracle-Trans-$\our$? 

 \begin{figure}[!htp]
\centering
\includegraphics[width=1\textwidth]{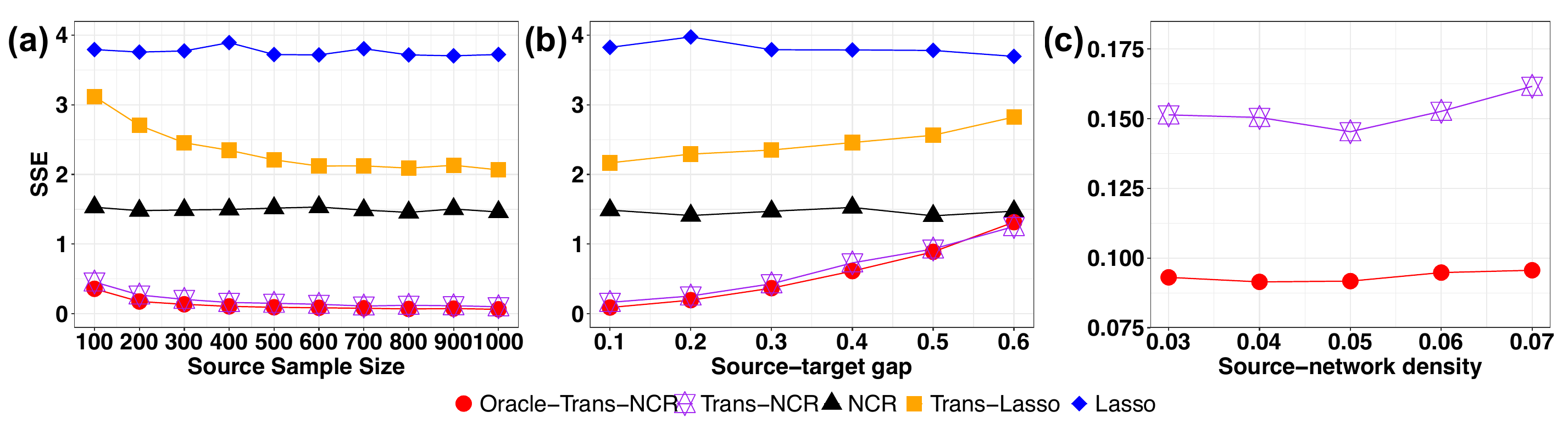}
\vspace{-0.6in}
\caption{Performance evaluation (SSE) under ER graph model of Oracle-Trans-$\our$ (red), Trans-$\our$ (purple), $\our$ (black), Trans-Lasso (orange), and Lasso (blue) under varying (a) source sample size, (b) domain shift $\delta$, and (c) source network  ER connecting probability
% SBM between-community probability $p_{out}$ 
(0.05 matches the target data setting).  In (c), only Oracle-Trans-$\our$ and Trans-$\our$ are shown for clarity, as the significantly larger SSE values of the other methods obscure the U-shaped performance pattern of Oracle-Trans-$\our$ and Trans-$\our$.  }
\label{fig:er}
\end{figure}

\textbf{Asymptotic performance.}
To investigate the impact of source sample sizes, 
 we  vary the sample size of source data $n_1=\ldots=n_{10}=n \in \{100,200,300,400,\ldots,1000\}$ while fixing other parameters:  size of the target data { $n_0=150$}, 
 % number of non-zero coefficients  {$s=16$, $\sk=12$},  
 source-target domain shift $\delta=0.1$, edge probability $p_0=\ldots=p_{10}=0.05$.
Figure \ref{fig:er}(a) shows that 
(1) Trans-$\our$ achieves comparable performance as Oracle-Trans-$\our$ under various source sample size, indicating Trans-$\our$  can effectively identify and leverage the most relevant source data for transfer learning.
(2) Our method  Trans-$\our$ and Oracle-Trans-$\our$ demonstrates a substantial decrease in SSE as the source data sample size increases.  Although Trans-Lasso also shows a decreasing trend in SSE, its convergence rate is considerably slower compared to Trans-$\our$.
(3) Across all different values of the source data sample size $n$, Trans-$\our$ consistently achieves the lowest SSE among the compared methods. This demonstrates the superior efficacy of our approach in utilizing the available source data to enhance the target task performance.
(4)  As expected, the performance of the Lasso and $\our$ methods does not change with the source data sample sizes since they don't use the source data information. 

% This trend is consistent across different values of $p_1$, indicating that our method's asymptotic performance is robust to various edge probabilities in the source dataset.
% (4) The Trans-Lasso shows a slight decrease with increasing source sample size when $p_1=0.01$. However, this asymptotic behavior of Trans-Lasso notably diminishes and eventually vanishes as $p_1$ increases, particularly at $p_1=0.1$.
% This is because Trans-Lasso is tailored for transfer learning scenarios where network structure does not exist. When source networks approach sparse densities like $p_1=0.01$, this resembles the non-existent network case, allowing for the effectiveness of Trans-Lasso. 
%  As $p_1$ rises, introducing more network complexity, TransLasso's performance deteriorates due to its inability to account for the network structure.

\textbf{Impact of source-target domain shift.} 
To investigate the impact of the  domain shift between  the transferable source domain and target domain,  we vary $\delta \in \{0.1,0.2,\ldots, 0.6\}$,  while fixing other parameters:  ${n_0=150}$,  $n_1=\ldots=n_{10}=500$,   and 
$p_0=\ldots=p_{10}=0.05$.
Figure \ref{fig:er}(b) reveals that the   SSE  of Trans-$\our$  and Trans-Lasso increases gradually as the source-target domain shift grows, which is expected since a larger shift implies reduced transferability between source and target domains. When the domain shift is relatively small ($\delta < 0.6$), Trans-$\our$ demonstrates superior performance compared with the other methods. However, as the domain shift becomes more significant ($\delta \geq 0.6$), the benefit of transfer learning in Trans-$\our$ diminishes. In this case, the performance of Trans-$\our$ becomes comparable to that of the target-only methods, such as target-only network lasso ($\our$) and baseline lasso (Lasso). In addition, we also observe that Trans-$\our$ achieves comparable performance as Oracle-Trans-$\our$ under various $\delta$.

\textbf{Impact of source-target ER model parameter discrepancy.}
 To investigate the impact of the ER model parameter difference between the source and target networks,  we vary the ER edge probability in the source data $p_1=\ldots=p_{10} \in \{0.03, 0.04,\ldots, 0.07\}$ while fixing  $p_0=0.05$, $n_1=\ldots=n_{10}=500$, $n_0=150$, $\delta=0.1$. Figure \ref{fig:er}(c) reveals that  Trans-$\our$ shows a slight U-shape trend, with the best results achieved when the source and target network densities are similar. As the density difference increases in either direction, the performance degrades.
 While network density discrepancies can impact the performance of transfer learning approaches, Trans-$\our$ demonstrates strong robustness in handling these differences, consistently outperforming Trans-Lasso, $\our$, and Lasso across the range of density variations tested. In addition, we also observe that Trans-$\our$ achieves comparable performance as Oracle-Trans-$\our$ under various ER connecting probabilities in source networks.

\subsection{Simulation Results Under SBM}

While our theorem focuses on results for the ER model, we empirically demonstrate the superiority of our method under the SBM \citep{rohe2011spectral}.

% Similar to the analysis conducted in Section \ref{sec:simu_er} for the ER model, we aim to address three analogous questions under the SBM setting:
% (1) Asymptotic performance: Do we observe similar asymptotic performance improvements as the sample size in the source domain increases when the networks are generated using the SBM?
% (2) Impact of domain shift: How does the domain shift between the source and target domains affect the performance of our method when the underlying networks follow the SBM?
% (3) Performance under different SBM parameters: When the target network and source networks are generated using different SBM parameters, how does our method perform in terms of transfer learning effectiveness?
 
% \textcolor{blue}{within-community density 0.1, between-community density 0.05; for (c) we fix ratio between within-community and between-community to be 2 and present various between-community density as x-axis (label of (c) "Source−network density" might need edits for more accurate description).}

\begin{figure}[htp]
\centering
\includegraphics[width=1\textwidth]{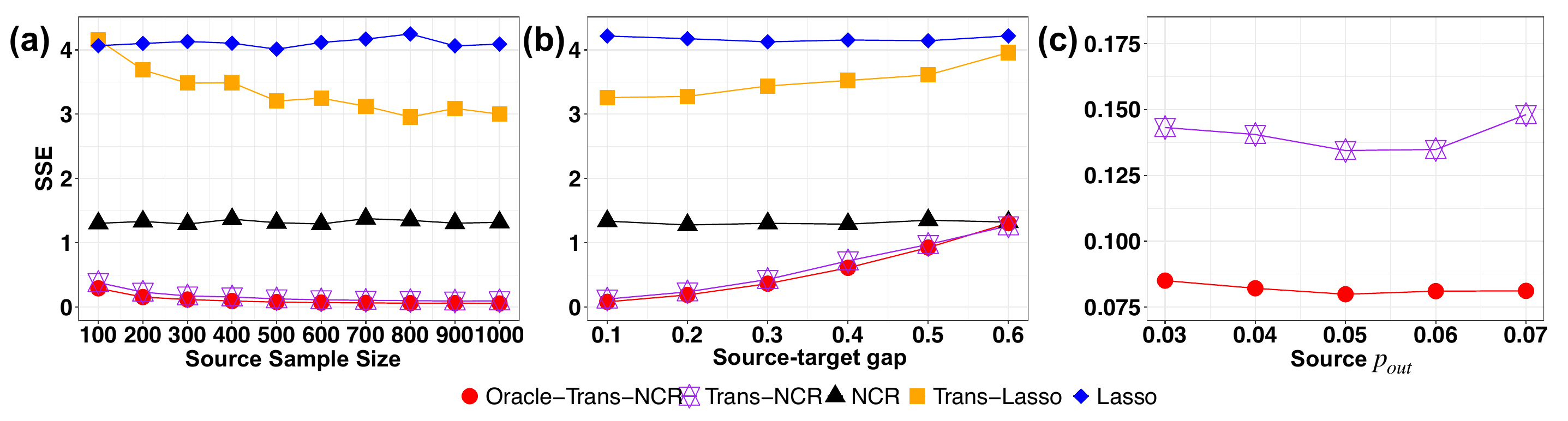}
\vspace{-0.6in}
\caption{SSE under SBM graph model of Oracle-Trans-$\our$ (red), Trans-$\our$ (purple), $\our$ (black), Trans-Lasso (orange), and Lasso (blue) under varying (a) source sample size, (b) domain shift $\delta$, and (c) SBM between-community probability $p_{out}$ (0.05 matches the target data setting) in the source network.
In (c), only Oracle-Trans-$\our$ and Trans-$\our$ are shown for clarity, as the significantly larger SSE values of the other methods obscure the U-shaped performance pattern of Oracle-Trans-$\our$ and Trans-$\our$. }
\label{fig:sbm}
\end{figure}

 \textbf{Asymptotic performance.}
To investigate the impact of source sample sizes under the SBM setting,   we conduct experiments by varying the sample size of the source data $n_1=\ldots=n_{10}=n \in \{100,200,300,400,\ldots,1000\}$ while keeping other parameters fixed. These fixed parameters include the size of the target data {$n_0=150$}, the source-target domain shift $\delta=0.1$, the within-community connecting probabilities $p^{(0)}_{in}=\ldots=p^{(10)}_{in}=0.1$, and the between-community connecting probabilities $p^{(0)}_{out}=\ldots=p^{(10)}_{out}=0.05$. 
% while fixing other parameters:  size of the target data { $n_0=150$},  
%  source-target domain shift $\delta=0.1$, within-community connecting probabilities $p^{0}_{in}=\ldots=p^{5}_{in}=0.1$, between-community connecting probabilities $p^{0}_{out}=\ldots=p^{5}_{out}=0.05$. 
Figure \ref{fig:sbm}(a) reveals similar trends to the ER model: 
(1)  SSE for Oracle-Trans-$\our$ and Trans-$\our$ decreases with larger source sample sizes, reflecting improved performance;  (2)
Oracle-Trans-$\our$ and Trans-$\our$  consistently achieve the lowest SSE across all sample sizes. 
% Notably, Trans-Lasso performs worse under the SBM model compared to the ER model, likely due to its inability to account for SBM's additional structural complexity.

%  One interesting observation is that the performance of Trans-Lasso deteriorates significantly under the SBM model compared to the ER model. This may be attributed to the additional complexity inherent in the network structure of SBM, which Trans-Lasso does not account for adequately.

\textbf{Impact of source-target domain shift.} 
To investigate the impact of the  domain shift between  the transferable source domain and target domain under SMB settings,  we vary $\delta \in \{0.1,0.2,\ldots, 0.6\}$,  while fixing other parameters:  ${n_0=150}$,  $n_1=\ldots=n_{10}=500$,   and 
within-community connecting probabilities $p^{(0)}_{in}=\ldots=p^{(10)}_{in}=0.1$, between-community connecting probabilities $p^{(0)}_{out}=\ldots=p^{(10)}_{out}=0.05$. 
Figure \ref{fig:sbm}(b) reveals that when the domain shift $\delta$ increases from 0.1 to 0.6, the performance of all methods deteriorates, as evidenced by the increasing SSE values, similar to their performance under the ER model. Nevertheless, Oracle-Trans-$\our$ and Trans-$\our$ still outperform the other methods.

\textbf{Impact of source-target SBM parameter discrepancy.}
 To investigate the impact of the SBM parameter difference between the source and target networks, we conduct experiments by 
varying the between-community connecting probabilities in the source domains as $p^{(1)}_{out}=\ldots=p^{(5)}_{out}=p_{out} \in \{0.03,0.04,\ldots, 0.07\}$, and setting the within-community connecting probabilities in the source domains as $p^{(1)}_{in}=\ldots=p^{(10)}_{in}=2p_{out}$, maintaining a constant ratio of 2 between the within-community and between-community connecting strengths. 
We keep the parameters in the target data fixed at $p^{(0)}_{out}=0.05$, $p^{(0)}_{in}=0.1$.
Other parameters remain constant, with $n_1=\ldots=n_{10}=500$, $n_0=150$, $\delta=0.1$. 
As shown in Figure \ref{fig:sbm}(c), Trans-$\our$ exhibits a slight U-shaped trend as $p_{out}$ increases from 0.03 to 0.07, with the best performance achieved when the source and target networks have the same parameters (i.e., $p_{out}=0.05$). 
The performance of Trans-$\our$ degrades as $p_{out}$ deviates from the target network density in either direction. As we vary $p_{out}$ in either direction, the community structure signal strength keep the same since we use a constant ratio, however, the change in $p_{out}$ still affects the overall network density, which is the primary factor driving the U-shaped performance pattern.
This observation is similar to the results obtained under the ER model (Figure \ref{fig:er}(c)).

 In sum, these findings suggest that our method maintains its superior performance even when the underlying networks are generated using the SBM, demonstrating its robustness to different random graph models.

\section{Real Application: User Activity Prediction}

Predicting user activity on social media platforms, such as daily tweets, has practical applications in marketing, content recommendation, social influence modeling, and public opinion analysis \citep{steinert2018predicting}. For platforms like Twitter, such predictions inform growth strategies by identifying user groups at risk of reduced activity, enabling targeted interventions like personalized notifications or exclusive content to re-engage them. This approach enhances user experiences, boosts engagement, and supports user base growth.

% Predicting user activity on social media platforms, particularly the number of tweets per day, has significant real-world applications across various fields, including marketing, content recommendation, social influence modeling, and public opinion analysis \citep{steinert2018predicting}. 
% Especially, platforms like Twitter  thrives on user-generated content, and predicting how often users will post helps inform platform growth strategies. For instance, if certain user groups are predicted to decrease their activity, the platform can intervene with targeted incentives, such as personalized notifications or exclusive content, to re-engage them.  By leveraging predictions of user activity, platforms can create more personalized experiences, retain user engagement, and ultimately grow their user base.

\textbf{Data collection. }
% {\color{magenta}
% In this study, we aim to predict a user's tweets per day using the social network structure and individual user attributes, represented by a set of covariates.
Data was collected from Sina Weibo, focusing on MBA students at four universities in Liaoning, Beijing, Shanghai, and Fujian. For each region, we constructed social networks where nodes represent accounts and edges denote follower-followee relationships, captured in adjacency matrices \( A_k = (A_{k,ij}) \in \mathbb{R}^{n_k \times n_k} \).
% For each region, we constructed a social network where nodes represent individual Weibo accounts, and edges represent follower-followee relationships. These relationships, recorded in adjacency matrices \( A_k = (A_{k,ij}) \in \mathbb{R}^{n_k \times n_k} \), provide useful information for prediction. 
Users connected in the network often exhibit similar behaviors due to peer influence, shared interests, or information diffusion. User-specific covariates \( \mathbf{X}_k  \) (66 features)   include gender, number of followers, followees, likes, and 62 binary high-frequency interest tags (e.g., travel, movies, finance). 
The response variable \( y_k \) represents the user's tweets per day.
% {\color{blue}A full list of covariates and their descriptions is provided in Table \ref{xxxx} in the Appendix. (Danyang: I don't think we need all the label interpretations, since they have already shown in the word cloud.)}
The size of the networks \( A_k \) varies significantly, with Liaoning as the smallest (315 nodes), Beijing  at 1,772 nodes, Shanghai  the largest at 2,248 nodes, and Fujian  at 1,051 nodes. 
% Liaoning, the smallest, has 315 nodes, while Beijing's network, with 1,772 nodes, reflects a larger and more active student community. Shanghai's network is even larger, comprising 2,248 nodes, representing the significant online presence of students from this metropolitan area.  Fujian's network is smaller, with 1,051 nodes.
% Liaoning serves as the target dataset for predicting daily tweets, while the larger networks from Beijing, Shanghai, and Fujian act as source datasets. 
% The smaller Liaoning network  have limited data, potentially reducing model reliability for prediction. In contrast, the larger networks from Beijing, Shanghai, and Fujian provide more data, which can enhance prediction for Liaoning through transfer learning. 
In this study, Liaoning is designated as the target dataset (\( k=0 \)), with the goal of predicting user activity (tweets per day) within this network. The larger networks from Beijing (\( k=1 \)), Shanghai (\( k=2 \)), and Fujian (\( k=3 \)) serve as source datasets, showcasing the effectiveness of the proposed transfer learning method in improving predictions for the Liaoning network.
% }

\textbf{Descriptive analysis.} 
{\color{black}
The network densities for the four regions are as follows: 0.9\(\times\)10\(^{-3}\) for Liaoning, 0.7\(\times\)10\(^{-3}\) for Beijing, 2.9\(\times\)10\(^{-3}\) for Shanghai, and 9.1\(\times\)10\(^{-3}\) for Fujian. Despite some variation, all networks are sparse and show similar levels of connectivity.
The histograms of tweets per day \( \mathbf{y} \) across the regions (Upper panel in Figure \ref{fig:hist}) reveal a similar pattern: a highly skewed distribution where most users tweet infrequently, and only a few are highly active. This behavior is typical of social media platforms, where content creation is dominated by a small number of users.
The lower panel of Figure \ref{fig:hist} presents word clouds of high-frequency personalized tags across four regions, illustrating user interests in each region. Common tags like "Fashion", "Food", "Music","Travel" and "Movie" appear prominently across all regions, suggesting shared interests that can aid in transfer learning. Regional differences are also evident, such as higher frequencies of "student" and "Freedom" in certain regions, reflecting localized preferences.  These variations
highlight the importance of considering localized behaviors in social network analysis.

}

% In this section, we demonstrate the application of our proposed transfer learning methods using real-world social network datasets collected from Sina Weibo, China’s largest Twitter-like platform. 

% The objective is to demonstrate the effectiveness of our methods in the context of social network analysis, specifically focusing on the transfer of knowledge between networks from different geographical regions.

% In addition to the network structure data $A_k$, we also collect response variable and  predictors associated with each node.  In this example, we focus on the tweets per day of the users,  which is the total number of posts divided by the registration time. this is our response  $\by$. We define the covariates $\bX$ as gender, number of followers, number of followees, number of favorites, and 62-dimensional high-frequency personalized tags (e.g., travel, movies, finance, etc.) with total dimension of $d=66$. We want to use these covariates and social network information to predict a user's tweets per day of the users. 

%\begin{figure}[htp]
%\centering
%\includegraphics[width=.25\textwidth]{fig_new/file_realdata/BJ.jpg}\hfill
%\includegraphics[width=.25\textwidth]{fig_new/file_realdata/SH.jpg}\hfill
%\includegraphics[width=.25\textwidth]{fig_new/file_realdata/FJ.jpg}\hfill
%\includegraphics[width=.25\textwidth]{fig_new/file_realdata/LN.jpg}\hfill
%\centering
%\caption{}\label{fig:cloud}
%\end{figure}

\textbf{Transfer learning results. }
To assess the effectiveness of our proposed Trans-NCR method, we  compare with  three baseline methods on the target dataset from Liaoning province. The baseline methods are: (1) NCR, which applies only to the target data, using both the network structure \( A \) and the covariate matrix \( \mathbf{X} \) to predict \( \mathbf{y} \). It integrates network information and covariates without incorporating transfer learning. (2) Trans-Lasso, which is a transfer learning approach based on linear regression \citep{li2022transfer}. It leverages data from the source regions to improve predictions on the target data but only uses the covariate matrix \( \mathbf{X} \), without considering the network structure. (3)  Lasso,  which performs lasso regression using only the \( \mathbf{X} \) features from the target data, without incorporating network information or transfer learning.
We use the out-of-sample root mean square error (RMSE) obtained through five-fold cross-validation as our evaluation metric. RMSE is a common measure used to assess the accuracy of regression models, with lower RMSE values indicating better predictive performance. The use of cross-validation ensures the robustness of the results by testing the model’s ability to generalize to unseen data.
The results of this analysis are presented in Table \ref{tab:reg}. 

%\begin{table}[h]
%\caption{ RMSE values of baseline methods and transfer learning approaches with different source data for Liaoning province}\label{tab:reg}
%\begin{center}
%\begin{tabular}{l|l|l|llll}
%  \hline
%  \hline
%  \multicolumn{1}{c|}{Method}& \multicolumn{1}{c|}{Target} &  \multicolumn{1}{c|}{Method}&\multicolumn{4}{c}{Source}\\
%  \hline
%  & {Liaoning} &{} &{Fujian} &  {Beijing} & {Shanghai}&{aggregation}\\
%  \hline
%{NCR }  & 0.5941  & {Trans-NCR}   &  0.5013  & 0.5071  & 0.5646 & 0.4766  \\
%{Lasso}  &  0.6537  & {Trans-Lasso}   & 0.5880  & 0.5714  & 0.5828  & 0.5224 \\
%
%\hline
%\hline
%\end{tabular}
%\end{center}
%\end{table}

% {\color{black}
% \begin{table}[h]
% \caption{ RMSE values of baseline methods and transfer learning approaches with different source data for Liaoning province}\label{tab:reg}
% \begin{center}
% \begin{tabular}{l|l|l|l|l|l}
% \hline
%   \hline
%   Source & Province & NCR & Lasso & Trans-NCR &  Trans-Lasso \\
%   \hline
%   None & Liaoning & 0.5941 &  0.6537 & -- &--\\
%   % \hline
%   Fujian & Liaoning & -- &-- &0.5013 &  0.5880 \\
%   Beijing & Liaoning &  --& --&0.5071 & 0.5714\\
%   Shanghai & Liaoning& -- & -- &0.5646 & 0.5828\\
%   F \& B \& S & Liaoning & -- & -- & 0.4766 & 0.5224\\
% \hline
% \hline
% \end{tabular}
% \end{center}
% \end{table}}

\begin{table}[h]
\caption{Prediction RMSE values (over five-fold cross-validation) under different source data when the target province is Liaoning.  ``F'' is short for Fujian, ``B'' for Beijing, and ``S'' for Shanghai.
}\label{tab:reg}
\begin{center}
\begin{tabular}{l|l|l|l|l}
\hline
\hline
Source  & Trans-NCR & NCR  & Trans-Lasso & Lasso  \\
\hline
% None  & -- & 0.5941 & -- & 0.6537 \\
Fujian  & \textbf{0.5013} & 0.5941 & 0.5880 & 0.6537 \\
Beijing  & \textbf{0.5071} & 0.5941 & 0.5714 & 0.6537 \\
Shanghai  & \textbf{0.5646} & 0.5941 & 0.5828 & 0.6537 \\
F \& B \& S  & \textbf{0.4766} & 0.5941 & 0.5224 & 0.6537 \\
\hline
\hline
\end{tabular}
\end{center}
\end{table}

The Trans-NCR method consistently achieves the lowest RMSE values across all source data configurations, with the aggregation of data from multiple regions providing the best performance (RMSE = 0.4766, which is {\color{black}80\% of that for the NCR based on Liaoning only}). This indicates that the proposed Trans-NCR method, which leverages both transfer learning and network structure, is highly effective in improving prediction accuracy. In addition, NCR, which incorporates both network structure and covariates, outperforms Lasso {\color{black}(91\% of the MSE for both target dataset and aggregation transfer learning)}, which only uses covariates, highlighting the importance of including network information in the predictive model. Trans-Lasso also benefits from transfer learning, but it performs worse than {\color{black}any of the} Trans-NCR, underscoring the added value of incorporating network information alongside the transferred covariates. Overall, the results confirm that combining transfer learning with network structure, especially using data from multiple sources, significantly enhances model performance for the Liaoning dataset.

% From Table \ref{tab:reg}, we can draw the following conclusions. First, among the models that use only Liaoning's own data for estimation, the Lasso estimation with the NLR model has a lower RMSE compared to using Lasso alone. This demonstrates the effectiveness of the NLR model by incorporating network effects. Second, using any single province as the source data for transfer learning results in better model performance with lower RMSE compared to using only the target data. Moreover, the transfer learning with the NLR model achieves even smaller RMSE, supporting the conclusion of Theorem ? in this paper. Third, under the NLR model, using the aggregation of all three source datasets yields the smallest RMSE. This indicates the effectiveness of multi-source network data in assisting the parameter estimation for the target data.

\section{Concluding Remarks}

{\color{black}In this paper, we proposed a novel high-dimensional network convolutional regression  model and established the theoretical properties for parameter estimation. Additionally, we developed a new transfer learning framework tailored for this model based on network data. And  rigorous theoretical guarantees for the proposed transfer learning estimators are established. To validate the effectiveness of our methods, we conducted extensive simulations and real-world data analyses, demonstrating their performance in various scenarios.

To conclude, we discuss several promising directions for future research. First, while our theoretical analysis focuses on the ER model for its analytical tractability, extending these results to more realistic network generation mechanisms, such as Stochastic Block Models (SBM) or graphon models, would provide insights into how community structure and degree heterogeneity affect transfer learning. 
However, extending our results to frameworks like the SBM introduces additional challenges because edges are no longer i.i.d. The dependency in SBM edges requires advanced tools from random matrix theory to handle.
Second, the current framework aggregates information from first-order (immediate neighbor) connections. Expanding to higher-order convolutions, which capture dependencies across multi-hop neighborhoods, could better model complex network dynamics. However, this introduces computational and statistical challenges, including managing increased dependency and ensuring convergence, which warrant careful investigation.
Third, many real-world networks involve multiple layers or heterogeneous types of nodes and edges (e.g., multiplex networks in biology or multimodal networks in social sciences). Adapting the NCR model to multilayer or heterogeneous networks would enable more comprehensive analyses, particularly in interdisciplinary applications where interactions occur across multiple domains.

% Second, while we discussed linear NCR models, extending this framework to incorporate more complex structures, such as graph convolutional networks (GCNs) with nonlinear transformations, and analyzing their theoretical properties would be valuable. Third, we assumed one layer's convolution, which only aggreagtes firts-neighbors' information. How to generalize to higher order is also an interesting directions. Exploring methods to handle higher layers,could further enhance the practical applicability of our framework.
}

\label{sec:conc}

% \newpage
% \begin{center}
% {\large\bf SUPPLEMENTARY MATERIAL}
% \end{center}

\bibliographystyle{apalike}
\bibliography{Mybib}

\newpage

\appendix
% \section{Appendix}

\begin{center}
  {\Large\bf Supplementary Material of ``Transfer Learning Under High-Dimensional  Network Convolutional Regression Model''}  
\end{center}

\linespread{1.3}

\setcounter{page}{1}
\setcounter{section}{0}
\setcounter{table}{0}
\setcounter{subsection}{0}
\renewcommand{\theequation}{\arabic{equation}}
\renewcommand{\thetheorem}{S.\arabic{theorem}}
\renewcommand{\thelemma}{S.\arabic{lemma}}
\renewcommand{\thesection}{S.\arabic{section}}
\setcounter{figure}{0}
\renewcommand{\thefigure}{S.\arabic{figure}}
\renewcommand{\thetable}{S.\arabic{table}}
\renewcommand{\thealgorithm}{S.\arabic{algorithm}}

% \scsection{APPENDICES}
\section{Notation}
\label{sec:notation_table}
We present the detailed expressions of notations frequently referenced in the proposed model and algorithm in Table \ref{tab:notation}.

\renewcommand{\arraystretch}{0.8}
\begin{table}[!htb]
\caption{Notations}\label{tab:notation}
\centering
\begin{tabular}{ll}
\hline \hline
Notations  & Description                        \\
\hline
$n_0$ & number of nodes in target network    \\
$A_0 \in \{0,1\}^{n_0 \times n_0}$  & adjacency matrix of target task without self-loops  \\
$A^*_0$  & normalized adjacency matrix of target task \\
$p_0$ & parameter of Erd\H{o}s--R\'enyi (ER) random graph model for target network \\
$ d $ & feature dimension \\
$\bX_0 \in \mathbb{R}^{n_0 \times d}$      & covariate matrix of target task \\
$\bY_0 \in \mathbb{R}^{n_0}$  & response vector of target task \\
$n_k$ & number of nodes in $k$-th source network \\
$A_k \in \{0,1\}^{n_k \times n_k}$  & adjacency matrix of $k$-th source task \\
$A^*_k$  & normalized adjacency matrix of $k$-th source task \\
$p_k$ & parameter of Erd\H{o}s--R\'enyi (ER) random graph model for $k$-th source network \\
$\bX_k \in \mathbb{R}^{n_k \times d}$ &  covariate matrix of $k$-th source task \\
$\bY_k \in \mathbb{R}^{n_k}$ & response vector of $k$-th source task\\
$\beta_{00} \in \mathbb{R}^{d}$ & true network effect of target task\\
$\beta_{01} \in \mathbb{R}^{d}$ & true individual effect of target task\\
$\gamma_0 \in \mathbb{R}^{2d}$ & true regression coefficient of target task\\
$\beta_{k0} \in \mathbb{R}^{d}$ & true network effect of $k$-th source task\\
$\beta_{k1} \in \mathbb{R}^{d}$ & true individual effect of $k$-th source task\\
$\gamma_k \in \mathbb{R}^{2d}$ & true regression coefficient of $k$-th source task\\
$\delta_k$  &  difference between the  $k$-th source and the target task's coefficient       \\
$\mathcal{A}$  & the set of $h-$level transferable source data\\
$n$  & total number of nodes in target network
and source sample  \\
$\gamma^\mathcal{A}$  & true underlying coefficient related to the pooled dataset  \\
$\hat{\gamma}_0$  &  estimate for $\gamma_{0}$ obtained using our algorithms \\
\hline \hline
\end{tabular}
\end{table}

\renewcommand{\arraystretch}{1.0}
%\section*{Appendix B. Useful Lemmas}
\section{Auxiliary lemmas}

In this subsection, we present four lemmas to support the proof of the theorems. To characterize the deviation of \(\hat{p}\) from \(p\), we utilize Chernoff's inequality, stated in Lemma \ref{lemma:chernoff inequality}, which is derived from Exercise 2.3.5 in \cite{vershynin2018high}. Next, we introduce the Hanson-Wright inequality in Lemma \ref{lemma:hwinequality} and \ref{lemma:hwinequalityforXY}, as provided in Theorem 6.2.1 of \cite{vershynin2018high}. 
% To address the Restricted Strong Convexity (RSC) condition, which is essential in the proof of the Lasso algorithm, we include Theorem \ref{thm:RSC}. 
In Lemma \ref{lemma:ubforA}, we establish upper bounds on the norms of \(A^*\) and \({A^*}^\top A^*\) to support subsequent proofs, and in Lemma \ref{lemma: expectationZTZ}, we provide an identity related to the covariance matrix of the dependent random matrix $\bZ$. 

\begin{lemma}[Chernoff's inequality for small deviations]
    \label{lemma:chernoff inequality}
	Let $X_i$ be independent Bernoulli random variables with parameters $p_i$. Consider their sum $S_N=\sum_{i=1}^N X_i$ and denote its mean by $\mu = \mathbb{E}S_N$. Then for $\delta \in \left(0, 1 \right] $ we have 
    $$
    \bbP(\left|S_N-\mu\right| \geqslant \delta \mu) \leqslant 2\exp\left\lbrace -c\mu\delta^2\right\rbrace \,.
    $$
\end{lemma}

\begin{lemma}[Hanson-Wright inequality]
	\label{lemma:hwinequality}
    Let $X=\left(X_1, \ldots, X_n\right) \in \mathbb{R}^n$ be a random vector with independent, mean zero, sub-gaussian coordinates. Let $B$ be an $n\times n$ matrix. Then, for every $t \geqslant 0$, we have
    $$
    \mathbb{P}\left\{\left|X^{\top} B X-\mathbb{E} X^{\top} B X\right| \geq t\right\} \leq 2 \exp \left[-c \min \left(\frac{t^2}{K^4\|B\|_F^2}, \frac{t}{K^2\|B\|}\right)\right], 
    $$
    where  $K=\max _i\left\|X_i\right\|_{\psi_2}$.
\end{lemma}

The following is a straightforward corollary of Hanson-Wright inequality for quadratic forms of two independent random vectors.
\begin{lemma}[Hanson-Wright inequality for two independent random vector]
	\label{lemma:hwinequalityforXY}
    Let $X=\left(X_1, \ldots, X_{n_1}\right) \in \mathbb{R}^{n_1}$ and $Y=\left(Y_1, \ldots, Y_{n_2}\right) \in \mathbb{R}^{n_2}$ be two independent random vectors with independent, mean zero, sub-gaussian coordinates. Let $B$ be an $n_1\times n_2$ matrix. Then, for every $t \geqslant 0$, we have
    $$
    \mathbb{P}\left\{\left|X^{\top} B Y\right| \geq t\right\} \leq 2 \exp \left[-c \min \left(\frac{t^2}{K^4\|B\|_F^2}, \frac{t}{K^2\|B\|}\right)\right], 
    $$
    where  $K=\max (\max _i\left\|X_i\right\|_{\psi_2},\max_j\left\|Y_j\right\|_{\psi_2})$.
\end{lemma}

% \begin{lemma}[Restricted strong convexity condition]
% \label{lemma:rsc condition}
%     \blue For any vector $u$, 
% 	$$\frac{u^\top \bZ^\top \bZ u}{n} \geqslant \kappa \|u\|_2^2 - C_5 \log{d} \sqrt{\frac{\Psi(p)}{n}} \|u\|_1 \|u\|_2 \, ,$$
% 	where $\Psi(p) \sim 1/(-4p\log{p})$ for $p$ close to $0$. Furthermore, if we replace $\bZ$ by $\hat \bZ$, then we have: 
%     $$
%     \frac{u^\top \hat \bZ^\top \hat \bZ u}{n} \geqslant \frac{\kappa}{2} \|u\|_2^2 - C_5 \log{d} \sqrt{\frac{\Psi(p)}{n}} \|u\|_1 \|u\|_2 - 3K \frac{\sqrt{2\log{d}}}{n}\|u_1\|_1^2  \, ,$$
% where $u_1$ is the last $d$ coordinates of $u$.   
% \end{lemma}

% \begin{lemma}[Restricted strong convexity condition]
% \label{lemma:rsc condition_Z_hat}
%     \blue For any vector $u$, 
% 	$$\frac{u^\top \hat \bZ^\top \hat \bZ u}{n} \geqslant \kappa \|u\|_2^2 - C_5 \log{d} \sqrt{\frac{\Psi(p)}{n}} \|u\|_1 \|u\|_2 - \frac{\sqrt{\log{d}}}{n^{3/2}}\|u\|_1^2 \, ,$$
% 	where $\Psi(p) \sim 1/(-4p\log{p})$ for $p$ close to $0$.
% \end{lemma}

\begin{lemma}[Upper bound for $A^*$ and ${A^*}^\top A^*$]
	\label{lemma:ubforA}
	(\romannumeral 1) For the $\ell_2$ norm and Frobenius norm of $A^*$, we have:
	\begin{equation*}
		\label{eq:A_op_bound}
		\Pr(\|A^*\|_2 \leqslant 2 \sqrt{np})=1 - \exp ( \log{n} - np/c),
	\end{equation*}
	\begin{equation*}
		\label{eq:A_frob_bound}
			\Pr(\| A^* \|_F^2 \leqslant 2n)=1 - 2( n^2 p)^{-1}.
	\end{equation*} 
	For the operator norm and Frobenius norm of ${A^*}^\top A^*$, we have
	\begin{equation*}
		\label{eq:ATA_op_bound}
		\Pr(\|{A^*}^\top A^*\|_2 \leqslant 4 np)=1 - \exp ( \log{n} - np/c),
	\end{equation*}
	\begin{equation*}
		\label{eq:ATA_frob_bound}
		\Pr(\|{A^*}^\top A^*\|_F^2 \leqslant 2n + 2n^2p^2)=1 - n^{-1}.
	\end{equation*} 
\end{lemma}

\begin{lemma}[Expectation of $\bZ^\top \bZ$]
	\label{lemma: expectationZTZ}
    $$
	\frac{1}{n}\bbE\left( \bZ^\top \bZ\right)  = \begin{pmatrix}
    \Sigma_X & \b0 \\
    \b0 & \Sigma_X 
    \end{pmatrix} = \begin{pmatrix}
        1 &  0 \\
        0 & 1
    \end{pmatrix} \otimes \Sigma_X \triangleq \Sigma_Z \,.
$$
\end{lemma}

%\section*{Appendix C. Proof of Theorems}
\section{Proof of Theorems}

% \subsection*{Appendix C.1. Proof of Theorem \ref{thm: general model upper bound without transfer learning}}
\subsection{Proof of Theorem \ref{thm: general model upper bound without transfer learning}}

To prove Theorem \ref{thm: general model upper bound without transfer learning}, we divide the proof into two steps. In the first step, we utilize the fact that \(\hat{\gamma}\) minimizes the objective function, obtaining an upper bound on the deviation of \(\hat{\gamma}\) from its true value \(\gamma\). However, this upper bound is related to \(\lambda\). Therefore, in the second step, we will provide the relationship between \(\lambda\), the sample size \(n\), and the network parameters \(p\), thereby completing the proof.

{\bf Step 1:} Recall that, we model the response variable as: $\by = A^* \bX \beta_0 + \bX \beta_1 + \epsilon = \bZ \gamma + \epsilon$. Furthermore, as we do not know $p$, we replace it by $\hat p$, i.e., we replace $\bZ$ by $\hat \bZ$. Therefore, our estimator $\hat \gamma$ is defined as: 
$$
\hat \gamma = \argmin_{\gamma} \left\{\frac{1}{2n}\|\by - \hat \bZ \gamma  \|_{2}^{2} + \lambda \|\gamma \|_{1}\right\}
$$
% Here, we assumes we only have one data source with $\by = A^* \bX \beta_0 + \bX \beta_1 + \epsilon = \bZ \gamma + \epsilon$. Note that the exact values of \( p\) are not known in advance, we can only use \( \hat{p} \) to substitute \( p \) for normalizing \( A \) to obtain \( \hat{A} \). This substitution alters our objective function from $(2n)^{-1}\|\by - \bZ \gamma  \|_{2}^{2} + \lambda \|\gamma \|_{1}$
% to $(2n)^{-1} \| \by - \hat{\bZ} \gamma  \|_{2}^{2} + \lambda \|\gamma\|_{1}$, where $\hat{\bZ} = (\hat{A}\bX, \bX)$.
% We denote the minimizers of this objective function as $\hat{\gamma}$, and the true parameters as $\gamma$. 
Denote $\hat{\mu} = \hat{\gamma} - \gamma$. Considering that $\hat{\gamma}$ is the minimizer of the objective function, we have the following inequalities:
\begin{equation}
	\label{eq:minimizer of obj func without transfer}
	\frac{1}{2n}\left\| \by - \hat{\bZ} \hat{\gamma}  \right\|_{2}^{2} + \lambda \left\| \hat{\gamma} \right\|_{1}
	\leqslant \frac{1}{2n}\left\|  \by - \hat{\bZ} \gamma \right\|_{2}^{2} + \lambda \left\| \gamma \right\|_{1}.
\end{equation}
Let $S$ denote the support set of $\gamma$ with cardinality $\left| S\right| = s_1 + s_2$. Because of $\by = \bZ \gamma + \epsilon$ and $\hat{\gamma} = \hat{\mu} + \gamma$, formula \eqref{eq:minimizer of obj func without transfer} can be simplifies to 
\begin{equation}
	\label{eq:obj func without transfer 1}
\begin{aligned}
\frac{1}{2n} \left\| \hat{\bZ} \hat{\mu} \right\|_{2}^{2} &\leqslant \frac{1}{n} \left\langle \epsilon + \left(\bZ-\hat{\bZ}\right) \gamma , \hat{\bZ} \hat{\mu} \right\rangle + \lambda \left\| \gamma \right\|_{1} - \lambda \left\| \hat{\gamma} \right\|_{1} \\
& = \frac{1}{n} \left\langle \hat{\bZ}^{\top} \epsilon + \hat{\bZ}^{\top} \left(\bZ-\hat{\bZ}\right) \gamma, \hat{\mu} \right\rangle + \lambda \left\| \gamma \right\|_{1} - \lambda \left\| \hat{\gamma} \right\|_{1},
\end{aligned}
\end{equation}
where the second equation is due to the duality property of the inner product.
Let the penalty parameter $\lambda$ satisfies $n^{-1}\|\hat \bZ^{\top} \epsilon + \hat{\bZ}^{\top} (\bZ-\hat{\bZ}) \gamma \|_{\infty} \leqslant \lambda/2$ with high probability. 
In this case, using H{\"o}lder inequality, the right-hand side of formula \eqref{eq:obj func without transfer 1} will take the form of ${\lambda} \| \hat{\mu}\|_1 / 2  + \lambda \| \gamma \|_{1} - \lambda \|\hat{\gamma} \|_{1}$. Since $\gamma_{S^{c}}=0$, we have $\|\gamma\|_{1}=\|\gamma_{S}\|_{1}$, $\hat{\mu} + \gamma = \hat{\gamma}$ and $\| \hat{\gamma} \|_{1} = \|\gamma + \widehat{\mu}\|_{1} = \| \gamma_{S}+\widehat{\mu}_{S} \|_{1} + \|\widehat{\mu}_{S^{c}}\|_{1} \geqslant \|\gamma_{S}\|_{1} -\|\widehat{\mu}_{S}\|_{1} + \|\widehat{\mu}_{S^{c}}\|_{1}$.
Substituting these relations into formula \eqref{eq:obj func without transfer 1} yields $(2n)^{-1} \| \hat{\bZ} \hat{\mu} \|_{2}^{2} \leqslant {3\lambda}\|\widehat{\mu}_{S}\|_{1} / 2 - {\lambda} \|\hat{\mu}_{S^{c}}\|_{1} / 2$. Using the fact $\|\widehat{\mu}_{S}\|_{1} \leqslant \sqrt{s_1 + s_0}\|\widehat{\mu}\|_{2}$ concluded from Cauchy inequality, we have 
\begin{equation}
	\label{eq: upper bound without trans befor RSC}
\frac{1}{2n} \left\| \hat{\bZ} \hat{\mu} \right\|_{2}^{2} \leqslant 3 \sqrt{s_0 + s_1} \lambda \left\|\hat{\mu} \right\|_{2} / 2.
\end{equation}
From the cone constraint \(0 \leqslant \frac{3}{2} \|\hat{\mu}_{S}\|_{1} - \frac{1}{2} \|\hat{\mu}_{S^{c}}\|_{1}\), 
we have $\|\hat{\mu}_{S^{c}}\|_{1} \leqslant 3\|\hat{\mu}_{S}\|_{1}$ and $\|\hat{\mu}\|_{1} = \|\hat{\mu}_{S}\|_{1} + \|\hat{\mu}_{S^{c}}\|_{1} \leqslant 4\|\hat{\mu}_{S}\|_{1} \leqslant 4 \sqrt{s}\|\hat{\mu}\|_{2}$. From Theorem \ref{thm:RSC}, we have
$$\kappa \|\hat{\mu}\|_2^2 - C_1 \log{d} \|\hat{\mu}\|_1 \|\hat{\mu}\|_2  \sqrt{\frac{\Psi(p)}{n}} - C_2 \frac{\sqrt{\log d}}{n} \|\hat{\mu}\|_1^2  \lesssim \sqrt{s_0 + s_1} \lambda_{\gamma} \|\hat{\mu}\|_2 \, .
$$
Denote $s = s_0+ s_1$.  Using the fact $\|\hat{\mu}\|_1 \le 4\sqrt{s}\|\hat{\mu}\|_{2}$
we have: 
\begin{align*}
& \kappa \|\hat{\mu}\|_2^2 - C_1 \log{d} \|\hat{\mu}\|_1 \|\hat{\mu}\|_2 \sqrt{\frac{\Psi(p)}{n}} - C_2 \frac{\sqrt{\log{d}}}{n}\|\hat{\mu}\|_1^2 \\
& \ge \left(\kappa - C_3 \log{d} \sqrt{\frac{s\Psi(p)}{n}} - C_4 \frac{s\sqrt{\log{d}}}{n}\right)\|\hat \mu\|_2^2 \\
& \ge \frac{\kappa}{2}\|\hat \mu\|_2^2 \,,
\end{align*}
where we use the assumption (C3) $s\log^2(d)\Psi(p)/n \leqslant s\log^2(d) / np = o(1)$. 
% and assumptions $\log d \sqrt{s \Psi(p)/n} = o(1)$ and $s \sqrt{\log d}/n= o(1)$, we have $\kappa \|\hat{\mu}\|_2^2 - C_1 \log{d} \|\hat{\mu}\|_1 \|\hat{\mu}\|_2  \sqrt{{\Psi(p)}/{n}} - C_2 {\sqrt{\log d}}/{n} \|\hat{\mu}\|_1^2 \geqslant (\kappa - C_1 \log{d}\sqrt{s {\Psi(p)}/{n}} -  C_2 s {\sqrt{\log d}}/{n}) \|\hat{\mu}\|_2^2 \geqslant \kappa\|\hat{\mu}\|_2^2/2$.
% }
% \DM{You the assumption $s\Psi(p)/n < 1$ right?} {\color{ForestGreen} Yes, actually we need assumption $\log d \sqrt{s \Psi(p)/n} = o(1)$.}
Combining this lower bound with formula \eqref{eq: upper bound without trans befor RSC} yields 
$\kappa \|\widehat{\mu}\|_{2}^{2}/2 \leqslant 3  \lambda \sqrt{s_0 + s_1} \|\widehat{\mu}\|_{2}$, so we have
\begin{equation}
	\label{eq:non-transfer result with lambda}
	\left\|\hat{\gamma} - \gamma \right\|_{2}^2 = \left\|\widehat{\mu}\right\|_{2}^2 \lesssim (s_1 + s_0) \lambda^2.
\end{equation}

{\bf Step 2: }In Step 1, we choose the penalty parameter $\lambda$ such that $n^{-1}\|\hat \bZ^{\top} \epsilon + \hat{\bZ}^{\top} (\bZ-\hat{\bZ}) \gamma \|_{\infty} \leqslant \lambda/2$ with high probability. To obtain the order of this $\lambda$, we need to find an upper bound on $n^{-1}\|\hat \bZ^{\top} \epsilon + \hat{\bZ}^{\top} (\bZ-\hat{\bZ}) \gamma \|_{\infty}$. Application of a traingle inequality along with observation $\hat \bZ - \bZ = [(\hat A - A^*)\bX \quad \mathbf{0}]$ yields:  
\begin{align}
\label{eq:lambda_upper_bound}
    & n^{-1}\|\hat \bZ^{\top} \epsilon + \hat{\bZ}^{\top} (\bZ-\hat{\bZ}) \gamma \|_{\infty} \notag \\
    \leqslant & n^{-1}\left\{\|\hat \bZ^{\top} \epsilon\|_\infty + \| \hat{\bZ}^{\top} (\bZ-\hat{\bZ}) \gamma\|_\infty\right\} \notag \\
    \leqslant & n^{-1}\left\{\|\hat A^\top \bX^\top \eps\|_\infty + \|\bX^\top \eps\|_\infty + \|\bX^\top \hat A^\top(A^* - \hat A)\bX\beta_0\|_\infty + \|\bX^\top (A^* - \hat A)\bX\beta_0\|_\infty\right\} \,. 
\end{align}
We now show that all of these terms can be upper bounded by $C\sqrt{\log{d}/n}$ with high probability for some constant $C > 0$, which will conclude that one may take $\lambda = 2C\sqrt{\log{d}/n}$. 
% {\bf Step 2:} Next, we need to provide an upper bound for $\lambda$. In fact, $\lambda$ clearly consists of two parts: (1) \(n^{-1} \| \hat{\bZ}^{\top} \epsilon \|_{\infty}\) and (2) $n^{-1} \| \hat{\bZ}^{\top} (\bZ - \hat{\bZ}) \gamma \|_{\infty}$. The first part can be expressed as $ n^{-1} \|[(\hat{A}\bX )^{\top}, \bX^{\top}]^{\top} \epsilon \|_{\infty}$, hence it is the maximum of $n^{-1} \| (\hat{A}\bX )^{\top}  \epsilon \|_{\infty} $ and $ n^{-1} \left\| \bX^{\top}  \epsilon \right\|_{\infty} $. The second part $n^{-1} \| \hat{\bZ}^{\top} (\bZ - \hat{\bZ}) \gamma\|_{\infty}$ is clearly close to 0 because of $\bZ - \hat{\bZ} \approx 0$ when $\hat{p}$ is close to true parameter $p$. Next we will proof $n^{-1} \| (\hat{A}\bX )^{\top}  \epsilon \|_{\infty} \lesssim \sqrt{\log d/n}$ , $n^{-1} \|  \bX^{\top}  \epsilon \|_{\infty} \lesssim \sqrt{\log d/n}$, and $n^{-1} \| \hat{\bZ}^{\top} (\bZ - \hat{\bZ}) \gamma \|_{\infty} \lesssim \sqrt{\log d/n} $ under some conditions to provide an upper bound for $\lambda$. 
We first present some bounds on $|\hat{p}-p|$ and norms of $\hat{A}$. 

\textbf{Step 2.1: Bound for $|\hat{p} - p|$ and $\hat{A}$.}
Recall that in the definition of $\hat A$, we scale $A_{ij}$ by $\sqrt{(n-1)\hat p}$ where $\hat p = D/n(n-1)$, $D = \sum_{i \neq j} A_{ij}$. It is immediate that $\bbE[\hat p] = p$. A simple application of Chernoff's inequality (Lemma \ref{lemma:chernoff inequality}) yields: 
$$
\bbP ( |\hat{p} - p| \geqslant \delta p )  = \bbP \{|D - n(n-1) p| \geqslant \delta n(n-1) p \}  \leqslant \exp \{-cn(n-1)p\delta^2\} \,.
$$
Taking $\delta = n^{-1/2}$, we conclude that with probability greater than $1 - \exp{(-c(n-1)p)}$ we have $|\hat p - p| \le p/\sqrt{n}$. Using this, we obtain the following: 
\begin{align*}
    \left|\hat A_{ij} - A^*_{ij}\right| \le \frac{|A_{ij}|}{\sqrt{(n-1)}}\left|\frac{1}{\sqrt{\hat p}} - \frac{1}{\sqrt{p}}\right| & \le \frac{|A_{ij}|}{\sqrt{(n-1)}} \frac{|p - \hat p|}{\sqrt{p\hat p}(\sqrt{p} + \sqrt{\hat p})} \\
    & \le \frac{|A_{ij}|}{\sqrt{n-1}}\frac{\sqrt{2}|\hat p - p|}{p^{3/2}} \le \frac{\sqrt{2}|A_{ij}|}{\sqrt{n(n-1)p}} = \frac{\sqrt{2}|A^*_{ij}|}{\sqrt{n}}\,.
\end{align*}
% Recall that $A^*$ differs from $A$ by a factor of $\sqrt{(n-1)p}$. And the difference between $\hat{A}$ and $A^*$ lies in the accuracy of the estimation of $p$, which is $\hat{A} - A^* = [  \lbrace  (n-1)\hat{p}\rbrace^{-1/2} - \lbrace  (n-1)p \rbrace ^{-1/2} ]  A$. {\color{magenta}\sout{ It is worth noting that $\hat{A} - A^*$ differs from $A$ by only a constant multiple.} Define} $ D = \sum_{i = 1}^n \sum_{j = 1}^n A_{ij} $, and $ A_{ij} \sim \operatorname{Bernoulli}(1, p)$. Using lemma \ref{lemma:chernoff inequality}, we can directly obtain an upper bound for $|\hat{p} - p|$: $ \bbP ( |\hat{p} - p| \geqslant \delta p )  = \bbP \{|D - n(n-1) p| \geqslant \delta n(n-1) p \}  \leqslant \exp \{-cn(n-1)p\delta^2\}$.
% Taking $\delta = 1/\sqrt{n}$, we know $|\hat{p} - p|\leqslant p/\sqrt{n}$ with probability $1-\exp\left\lbrace -c(n-1)p\right\rbrace  \to 1$ because of $np \to \infty$.
% This indicates that $|\hat{p} - p|$ will be sufficiently small, which gives us an opportunity to use Taylor expansion to approximate the coefficients in $\hat{A} - A^*$, specifically $ |\{(n-1)\hat{p}\} ^{-1/2} -  \{(n-1)p\}^{-1/2}| \lesssim  \{(n-1) p^3\}^{-1/2} |p - \hat{p}| \lesssim n^{-1} p^{-1/2}$, which implies $|\hat{A} - A^*| \lesssim n^{-1/2} A^*$.
% \DM{The above bound also implies $|\hat A_{ij} - A^*_{ij}| \le |A^*_{ij}|/\sqrt{n}$ as $np \uparrow \infty$. We may need this cruder upper bound.}
We now use the above inequality and Lemma \ref{lemma:ubforA} to bound the operator and Frobenious norm of $\hat A - A$. In particular, we have  
% For the bounds for norms of $\hat{A}$, due to the properties of norms and lemma \ref{lemma:ubforA}, we know that
\begin{equation}
	\label{eq:2 and F norm for hat A - A}
	\begin{aligned}
		&\left\| \hat{A} - A^* \right\|_2 \leqslant \left\| A^*\right\|_2/\sqrt{n} \lesssim \sqrt{p} , \\
		&\left\|\hat{A} - A^*\right\|_F^2 \leqslant \left\| A^*\right\|_F^2/n \lesssim 1,
	\end{aligned}
\end{equation}
and 
\begin{equation}
	\label{eq:2 and F norm for hat A}
	\begin{aligned}
		&\left\|\hat{A}\right\|_2 \leqslant \left\| A^*\right\|_2 + \left\| \hat{A} - A^*\right\|_2 \leqslant 2 \left\| A^*\right\|_2 \lesssim \sqrt{np} , \\
		&\left\|\hat{A}\right\|_F^2 \leqslant 2\left\| A^*\right\|_F^2 + 2\left\| \hat{A} - A^*\right\|_F^2 \leqslant 3 \left\| A^*\right\|_F^2 \lesssim n,
	\end{aligned}
\end{equation}
with high probability.

\textbf{Step 2.2: Bounding first two terms of Equation \eqref{eq:lambda_upper_bound}}:  
We start with bounding $n^{-1}\|\bX^\top \hat A^\top \eps\|_\infty$, for which we use Hanson-Wright inequality (Lemma \ref{lemma:hwinequalityforXY}). Note that the $j^{th}$ element of $\bX^\top \hat A^\top \eps$ is $\bX_{*j}^\top \hat A^\top \eps$, we $\bX_{*j}$ is the $j^{th}$ column of $\bX$. 
% {\color{magenta}We first discuss the bound given $\hat{A}$}.  
% Considering the $j$-th element of $\bX^{\top} \hat{A}^{\top} \epsilon$, 
Since $\bX_{*, j}$ and $\epsilon$ are both independent sub-gaussian random variables, we can use the Lemma \ref{lemma:hwinequalityforXY} and formula \eqref{eq:2 and F norm for hat A} to obtain:  
\begin{align*}
    \bbP\left(n^{-1} |{\bX_{*j}}^{T} \hat{A}^{T} \epsilon| \geqslant t \mid \hat{A}\right)  & \leqslant 2 \exp \{-c \min ( n^2 t^2/\| \hat{A}\|_{\mathrm{F}}^2, nt/\|{\hat{A}}\|_{2}) \} \\
    & \le 2 \exp \{-c\min(n t^2, t\sqrt{n/p})\}
\end{align*}
with high probability. 
% an upper bound on the probability of $n^{-1} X_j^{\top} \hat{A}^{T} \epsilon$ as $\mathbb{P}\{n^{-1} |{X_j}^{T} \hat{A}^{T} \epsilon| \geqslant t \mid \hat{A}\}  \leqslant 2 \exp \{-c \min ( n^2 t^2/\| \hat{A}\|_{\mathrm{F}}^2, nt/\|{\hat{A}}\|_{2}) \} \leqslant  2 \exp \{-c\min(n t^2, t\sqrt{n/p})\}$ with high probability, .
This implies, via a union bound: 
$$
 \bbP\left(n^{-1} \|\bX^{T} \hat{A}^{T} \epsilon \|_{\infty} \geqslant t \mid \hat{A}\right) \le \exp \{\log d - c\min(n t^2, t\sqrt{n/ p})\} \,.
$$
% $\mathbb{P} \{n^{-1} \|\bX^{T} \hat{A}^{T} \epsilon \|_{\infty} \geqslant t \mid \hat{A}\} \leqslant  \mathbb{P} \{n^{-1} |{X_1}^{T} \hat{A}^{T} \epsilon| \geqslant t \mid \hat{A}\} + \mathbb{P}\{n^{-1} |{X_2}^{T}\hat{A}^{T}\epsilon| \geqslant t \mid \hat{A}\}  + \cdots + \mathbb{P} \{n^{-1} |{X_d}^{T} \hat{A}^{T} \epsilon | \geqslant t \mid \hat{A}\}  \leqslant   \exp \{\log d - c\min(n t^2, t\sqrt{n/ p})\}$. 
Choosing $t = c \max\{\sqrt{\log{d}/n}, \log d \sqrt{p/n}\} = c \sqrt{\log{d}/n}$, with assumption (C3) $p \log d \to 0$, we conclude $n^{-1} \| (\hat{A}\bX )^{\top}  \epsilon \|_{\infty} \lesssim \sqrt{\log d/n}$, with high probability converge to 1. 
% \DM{In the first inequality, should not we condition on both $X$ and $\hat A$? Otherwise, we may need a version of HW inequality with bounds $X^\top B Y$ for two independent random variables $(X, Y)$. I think we can use that, but then we have to mention and refer to it properly. } {\red Wang L: I have added Lemma \ref{lemma:hwinequalityforXY} and its proof in the lemma section. In fact, we are using the HW inequality for independent random variables (X, Y). In fact, it is obvious because one only needs to combine (X, Y) as a new vector and construct an anti-diagonal matrix for A. Since each element of (X, Y) remains an independent sub-Gaussian random variable, the HW inequality can be applied to obtain the first inequality.}
Using a similar application of Lemma \ref{lemma:hwinequalityforXY} establishes that $n^{-1} \|  \bX^{\top}  \epsilon \|_{\infty} \lesssim \sqrt{\log d/n}$. Therefore, we can bound $ n^{-1} \|[(\hat{A}\bX )^{\top}, \bX^{\top}]^{\top} \epsilon \|_{\infty} \lesssim \sqrt{\log d/n}$.

\textbf{Step 2.3: Bounding last two terms of Equation \eqref{eq:lambda_upper_bound}: } 
% The second part of \(\lambda\) can be expressed as $
% \| \hat{\bZ}^{\top} (\bZ - \hat{\bZ}) \gamma \|_{\infty} = \max  \{ \| \bX^{\top} (A^* - \hat{A}) \bX \beta_0 \|_{\infty}, \| \bX^{\top} \hat{A}^{\top} (A^* - \hat{A}) \bX \beta_0 \|_{\infty} \}$. 
For notational convenience, define: 
\begin{align*}
    M_1 & = n^{-1}\|\bX^\top (A^* - \hat A)\bX\|_{\infty, \infty}\,,\\
    M_2 & = n^{-1}\|\bX^\top \hat A^\top(A^* - \hat A)\bX\|_{\infty, \infty}  \,.
\end{align*}
% Let $M_1$ be the maximum element of matrix $n^{-1} \bX^{\top} (A^* - \hat{A}) \bX $, and $M_2$ be the maximum element of matrix $n^{-1} \bX^{\top} \hat{A}^{\top} (A^* - \hat{A}) \bX$. 
An application of Lemma \ref{lemma:hwinequalityforXY} along with Equation \eqref{eq:2 and F norm for hat A} yields: 
\begin{align}
	\label{eq: M1 upperbound}
\bbP(M_1 \geqslant t) & = \bbP\left(n^{-1} \max_{i, j} |\bX_{*, i}^\top (A^* - \hat A)\bX_{*, j}| \ge t\right) \notag \\
& \leqslant \sum_{i, j} \bbP\left(n^{-1}|\bX_{*, i}^\top (A^* - \hat A)\bX_{*, j}| \ge t\right) \notag \\
& \leqslant \exp \left\{2\log d - \min(n^2 t^2, nt/\sqrt{p} )\right\} \,.
\end{align}
% Obviously, since $\hat{p}$ is very close to $p$, \( M_1 \) and \( M_2 \) both being very small. Specifically, by applying formula \eqref{eq:2 and F norm for hat A} and Lemma \ref{lemma:hwinequality}, we derive $\bbP(M_1 \geqslant t) \leqslant \exp \{2\log d - \min(n^2 t^2, nt/\sqrt{p} )\}$. 
Taking $ t = c\sqrt{\log d} / n $, we have $\bbP(M_1 \geqslant c\sqrt{\log d} / n) \leqslant 1/d $ with assumption (C3) $p \log d \to 0$. Under assumption (C4) $\| \beta_0\|_1 \leqslant \sqrt{n/(1+np^2)}$, we have $M_1 \| \beta_0 \|_1 \leqslant \sqrt{\log d / n}$ with probability $\ge 1 - d^{-1}$.
For $M_2$, we use Equation \eqref{eq:2 and F norm for hat A - A} to obtain the following bounds: 
% can use the same technique as in Step 2.1 and formula \eqref{eq:2 and F norm for hat A - A} to obtain:
\begin{equation}
	\label{eq:2 and F norm for hat A * A}
	\begin{aligned}
		& \left| \hat{A}^{\top} A^* - \hat{A}^{\top} \hat{A}\right|  \leqslant  \frac{1}{(n-1) p^2}\left| {p}-{\hat{p}}\right|  A^{\top} A \leqslant  \frac{1}{\sqrt{n}} {A^*}^{\top} A^* ,\\
		&\left\| \hat{A}^{\top} A^* - \hat{A}^{\top} \hat{A}\right\|_2  \leqslant \frac{1}{\sqrt{n}} np = \frac{p}{\sqrt{n}} , \\
		&\left\| \hat{A}^{\top} A^* - \hat{A}^{\top} \hat{A}\right\|_F^2  \leqslant  \frac{1}{n} (n + n^2 p^2) = 1 + np^2 , \\
	\end{aligned}
\end{equation}
Using these bounds, along with applying Lemma \ref{lemma:hwinequalityforXY}, we conclude that for any $t > 0$: 
\begin{equation}
	\label{eq:M2 upper bound}
\Pr\left( M_2 \geqslant t\right)  \leqslant \exp\left\lbrace 2\log d - \min\left(\frac{n^2 t^2 }{1 + np^2} , \frac{n^{3/2}t}{p} \right) \right\rbrace.
\end{equation}
Taking $t = c\sqrt{(1+np^2) \log d}/n$, we obtain $\bbP(M_2 \geqslant c\sqrt{(1+np^2) \log d}/n ) \leqslant 1/d $ as $n^2 \gg \log{d}$ (Assumption (C3)). 
% {\color{red} Wang L: Here we need $n^2 \gg \log d$}. 
Under Assumption (C4), $\| \beta_0\|_1 \leqslant \sqrt{n/(1+np^2)}$, which implies
$M_2 \| \beta_0 \|_1 \leqslant \sqrt{\log d / n}$ with probability $\ge 1 - d^{-1}$.

Combining the results from Step 2.2 and Step 2.3, we establish that it is possible to choose $\lambda$ such that $\lambda \leqslant C\sqrt{\log d/n}$ for some constant $C > 0$. Substituting this into formula \eqref{eq:non-transfer result with lambda}, we conclude that for some suitably chosen constant $K > 0$, 
$$
\bbP\left(\|\widehat{\gamma}-\gamma\|^2_{2} \ge K(s_0 + s_1)\frac{\log{d}}{n}\right)  \le \frac{1}{d} + \frac{1}{n} + e^{\left(\log{n} -\frac{np}{c}\right)} \,.
$$
The right-hand side goes to 0 as $(n \wedge d) \uparrow \infty$ along with the assumption (C3) $np \gg \log{n}$. This completes the proof.

% \subsection*{Appendix C.2. Proof of Theorem \ref{thm: general model upper bound with transfer learning}}
\subsection{Proof of Theorem \ref{thm: general model upper bound with transfer learning}}

We will prove Theorem \ref{thm: general model upper bound with transfer learning} in three steps. In the first step, we consider the estimation error of the Transferring Step in Algorithm \ref{al:trans_alg}, obtaining an upper bound that includes $\lambda_{\gamma}$. In the second step, we will present the relationship between $\lambda_{\gamma}$, $n_k$, and $p_k$, and then substitute the conclusions into the results obtained in the first step. In the third step, we consider the estimation error of the Debiasing Step in Algorithm \ref{al:trans_alg}.
\\
\noindent
{\bf Step 1:} Recall the definition of $\hat \gamma_A$: 
$$
\hat \gamma_A = \argmin_{\gamma} \left\{\frac{1}{2n} \sum_{k} \|( \by_k - \hat{\bZ}_k \gamma) \|_{2}^{2} + \lambda_{\gamma}\|{\gamma}\|_{1}\right\}
$$
As elaborated in Section 3.2, $\hat \gamma_A$ approximates $\gamma_A$ (defined in Equation \eqref{eq:transferstepobjectfunction}). 
% {\color{magenta}Note that the Lasso} estimator $\hat{\gamma}^\mathcal{A}$ reach the minimum value of $(2n)^{-1} \sum_{k} \|( \by_k - \hat{\bZ}_k \gamma) \|_{2}^{2}+\lambda_{\gamma}\|{\gamma}\|_{1}$.
% Denote the true parameters as $\gamma^\mathcal{A}$ and estimate error 
Set $\hat{u} = \hat{\gamma}^\mathcal{A} - \gamma^\mathcal{A}$.
Following the same approach as in formula \eqref{eq:minimizer of obj func without transfer} and \eqref{eq:obj func without transfer 1}, we can derive the following inequality: 
$$
\begin{aligned}
	\frac{1}{2n} \sum_{k} \left\|\hat{\bZ}_k \hat{u}\right\|_2^2  & \leqslant \lambda_{\gamma} \left\|\gamma^\mathcal{A} \right\|_{1}-\lambda_{\gamma}\left\|\hat{\gamma}^\mathcal{A}\right\|_{1}+\frac{1}{n}\left|\hat{u}^{\top} \sum_{k} \hat{\bZ}_k^{\top}\left(y_k - \hat{\bZ}_k \gamma^\mathcal{A} \right)\right| \\
	& \leqslant \lambda_{\gamma}\left\| \gamma^\mathcal{A} \right\|_{1}-\lambda_{\gamma}\left\|\hat{\gamma}^\mathcal{A}\right\|_{1}+\frac{\lambda_{\gamma}}{2}\left\|\hat{u}\right\|_{1}\\
	& = \lambda_{\gamma}\left\|\gamma^\mathcal{A}_S \right\|_{1} + \lambda_{\gamma}\left\|\gamma^\mathcal{A}_{S^c} \right\|_{1} - \lambda_{\gamma}\left\|\hat{\gamma}_S^\mathcal{A}\right\|_{1} - \lambda_{\gamma}\left\|\hat{\gamma}_{S^c}^\mathcal{A}\right\|_{1} + \frac{\lambda_{\gamma}}{2}\left\|\hat{u}_S\right\|_{1} + \frac{\lambda_{\gamma}}{2}\left\|\hat{u}_{S^c}\right\|_{1},
\end{aligned}
$$
where the second inequality is obtained by choosing $\lambda_{\gamma} \geqslant 2 n^{-1} \| \sum_{k} \hat{\bZ}_k^{\top} (\by_k - \hat{\bZ}_k \gamma^\mathcal{A}) \|_{\infty}$ and applying Hölder's inequality, while the third equality is derived from the properties of the \(\ell_1\) norm of vectors. 
The subscript $S$ denotes the support set of sparse vector $\gamma_0$, and $S^c$ is the complement set of $S$. It is worth noting that since the $\gamma_k$ of each source task is different $\gamma_0$, $S$ is \emph{not} the supporting set of $\gamma^\cA$, so $\gamma_{S^c}^\mathcal{A} \neq 0$. By reorganizing the terms on the right-hand side of the inequality, we obtain:  
\allowdisplaybreaks
\begin{align}
\label{eq:transfer_bound_1}
	\frac{1}{2n} \sum_{k}  \left\|\hat{\bZ}_k \hat{u}\right\|_2^2 & \leqslant \lambda_{\gamma}\left\|\gamma_S^\mathcal{A} \right\|_{1} - \lambda_{\gamma}\left\| \hat{\gamma}_S^\mathcal{A} \right\|_{1} + \frac{\lambda_{\gamma}}{2}\left\|\hat{u}_S\right\|_{1} + \lambda_{\gamma}\left\|\gamma_{S^c}^\mathcal{A} \right\|_{1} - \lambda_{\gamma}\left\|\hat{\gamma}_{S^c}^\mathcal{A} \right\|_{1} + \frac{\lambda_{\gamma}}{2}\left\|\hat{u}_{S^c}\right\|_{1} \notag \\
	& \leqslant \frac{3}{2}\lambda_{\gamma}\left\|\hat{u}_S\right\|_{1} + \lambda_{\gamma}\left\|\gamma_{S^c}^\mathcal{A} \right\|_{1} - \lambda_{\gamma}\left\|\hat{\gamma}_{S^c}^\mathcal{A} \right\|_{1} + \frac{\lambda_{\gamma}}{2}\left\|\hat{u}_{S^c}\right\|_{1} \notag \\
	& \leqslant \frac{3}{2}\lambda_{\gamma}\left\|\hat{u}_S\right\|_{1} + 2 \lambda_{\gamma}\left\|\gamma_{S^c}^\mathcal{A} \right\|_{1} - \frac{1}{2}\lambda_{\gamma}\left\|\hat{u}_{S^c}\right\|_{1} \,.
\end{align}
Here, the second inequality follows from a triangle inequality $\|\gamma^\mathcal{A}_S \|_{1} - \|\hat{\gamma}_S^\mathcal{A} \|_{1} \leqslant \|\hat{u}_S\|_{1}$, and the third inequality is derived from $\|\hat{\gamma}_{S^c}^\mathcal{A} \|_{1} \geqslant \|\hat{u}_{S^c}\|_{1} - \|\gamma^\mathcal{A}_{S^c}\|_{1}$. 
We now divide the proof into two parts, depending on whether $3\|\hat{u}_S\|_{1}/2 \geqslant 2\|\gamma_{S^c}^\mathcal{A} \|_{1}$ or not. 
\\\\
\noindent
{\bf Situation 1:} Consider the case when $3\|\hat{u}_S\|_{1}/2 \geqslant 2\|\gamma_{S^c}^\mathcal{A} \|_{1}$. From Equation \eqref{eq:transfer_bound_1}, we have: 
$$
\frac{1}{2n} \sum_{k}  \| \hat{\bZ}_k \hat{u}\|_2^2 \leqslant 3 \lambda_{\gamma}\|\hat{u}_S\|_{1} - \lambda_{\gamma}\|\hat{u}_{S^c}\|_{1}/2 \leqslant 3 \lambda_{\gamma} \sqrt{s} \|\hat{u} \|_2
$$
and the first inequality concludes $\lambda_{\gamma}\|\hat{u}_{S^c}\|_{1} \leqslant 6\|\hat{u}_S\|_{1}$, which implies $\|\hat{u}\|_{1} = \|\hat{u}_S\|_{1} + \|\hat{u}_{S^c}\|_{1} \leqslant 7 \|\hat{u}_S\|_{1} \lesssim  \sqrt{s} \|\hat{u}_S\|_{2}$. 
Using RSC condition for $\hat{\bZ}_k$, we arrive 
$$
\kappa \|\hat{u}\|_2^2 - C_1 \log{d} \|\hat{u}\|_1 \|\hat{u}\|_2  \frac{\sum_{k} \sqrt{\Psi(p_k) n_k}}{n} - C_2 \frac{\sqrt{\log{d}}}{n} \|\hat{u}\|_1^2 \lesssim \sqrt{s} \lambda_{\gamma} \|\hat{u}\|_2 \,.
$$
Under assumptions (C3) $\sqrt{\log d} s / n = o(1)$ and $\log{d} \sqrt{s} \sum_{k} \sqrt{\Psi(p_k) n_k}/n = o(1)$ which concluded by assumption (C3), we conclude $\|\hat{u}\|_2^2 \lesssim s \lambda^2_{\gamma}$.
% If the supporting set of $\gamma_k$ in each source task is close enough to $S$, the supporting set of $\gamma_0$, then $\gamma_{S^c}^\mathcal{A}$ will be close to 0, which means $3\lambda_{\gamma}\|\hat{u}_S\|_{1}/2 \geqslant 2 \lambda_{\gamma}\|\gamma_{S^c}^\mathcal{A} \|_{1}$. In the case, we conclude two fact: (1) $(2n)^{-1} \sum_{k}  \| \hat{\bZ}_k \hat{u}\|_2^2 \leqslant 3 \lambda_{\gamma}\|\hat{u}_S\|_{1} \leqslant 3 \lambda_{\gamma} \sqrt{s} \|\hat{u} \|_2$,  where the second inequality is derived from $\|\hat{u}_S \|_1 \leqslant \sqrt{s} \|\hat{u}\|_2$. (2) $0 \leqslant 3 \lambda_{\gamma}\|\hat{u}_S\|_{1} -  \lambda_{\gamma}\|\hat{u}_{S^c}\|_{1}$, implying $\|\hat{u}\|_1 = \|\hat{u}_S\|_1 + \|\hat{u}_{S^c}\|_1 \lesssim \|\hat{u}_S\|_1 \leqslant \sqrt{s} \|\hat{u}\|_2 $. 
% {\color{green} Using RSC condition for $\hat{\bZ}_k$, we arrive $\kappa \|\hat{u}\|_2^2 - C \log{d} \|\hat{u}\|_1 \|\hat{u}\|_2  \sum_{k} \sqrt{\Psi(p_k) n_k} /n \lesssim \sqrt{s} \lambda_{\gamma} \|\hat{u}\|_2$. Under assumption that $\log{d} \sqrt{s} \sum_{k} \sqrt{\Psi(p_k) n_k}/n = o(1)$ {\red which can be concluded by $\log d \sqrt{s} \ll n_k p_k$,so this condition can be satisfied easily}, we conclude $\|\hat{u}\|_2^2 \lesssim s \lambda^2_{\gamma}$.}
\\\\
\noindent
{\bf Situation 2:} Now, consider the case when $3\|\hat{u}_S\|_{1}/2 < 2\|\gamma_{S^c}^\mathcal{A} \|_{1}$. In that case, we have from Equation \eqref{eq:transfer_bound_1}: 
\begin{equation}
\label{eq:transfer_bound_2}
0 \leqslant \frac{1}{2n} \sum_{k}  \| \hat{\bZ}_k \hat{u}\|_2^2 \le 4 \lambda_{\gamma} \|\gamma_{S^c}^\mathcal{A} \|_{1}- \frac{\lambda_{\gamma}}{2}\|\hat{u}_{S^c}\|_{1} \,. 
\end{equation}
which immediately implies $\|\hat{u}_{S^c}\|_{1} \leqslant 8 \|\gamma_{S^c}^\mathcal{A} \|_{1}$. Note that we have already assumed $3\|\hat{u}_S\|_{1}/2 < 2\|\gamma_{S^c}^\mathcal{A} \|_{1}$, which is equivalent to $\|\hat{u}_S\|_{1} \leqslant (4/3)\|\gamma_{S^c}^\mathcal{A} \|_{1}$. Therefore, we can conclude: 
$$
\|\hat u\|_1 = \|\hat{u}_S\|_{1} + \|\hat{u}_{S^c}\|_{1} \leqslant 9.34 \|\gamma_{S^c}^\mathcal{A} \|_{1} = 9.34 \|\delta_{S^c}^\mathcal{A} \|_{1} \leqslant 9.34 h\,.
$$
In the above display, the second last equality follows from the fact that  $\gamma^\mathcal{A} = \gamma_0 + \delta^\mathcal{A}$ and ${\gamma_0}_{S^c} = 0$, and the last inequality follows from the fact $\|\delta_{S^c}^\mathcal{A} \|_{1} \leqslant h$. 
% Assumption ??. \DM{Refer the transfer assumption regarding $h$}. {\color{ForestGreen} we do not put it into conditions, we just say transferable source set $\mathcal{A}$ satisfing this condition.}
% If the supporting set of $\gamma_k$ in each source task is very different from the supporting set $S$, then $\| \gamma_{S^c}\| $ can be very large, which means $3\lambda_{\gamma}\|\hat{u}_S\|_{1}/2 \leqslant 2 \lambda_{\gamma} \|\gamma_{S^c}^\mathcal{A} \|_{1}$. In the case, $0 \leqslant (2n)^{-1 }\sum_{k} \| \hat{\bZ}_k \hat{u}\|_2^2 \leqslant 4 \lambda_{\gamma} \|\gamma_{S^c}^\mathcal{A} \|_{1}- \lambda_{\gamma}\|\hat{u}_{S^c}\|_{1}/2$. Thus, we have $\|\hat{u} \|_1 = \|\hat{u}_S \|_1 + \|\hat{u}_{S^c} \|_1 \lesssim \|\gamma_{S^c}^\mathcal{A} \|_{1} = \|\delta_{S^c}^\mathcal{A} \|_1 \leqslant h$, where the second inequality is due to $\|\hat{u}_{S^c}\|_{1} \leqslant 8 \|\gamma_{S^c}^\mathcal{A} \|_{1}$ and $\|\hat{u}_S\|_{1} \leqslant 4 \|\gamma_{S^c}^\mathcal{A}\|_{1}/3$, and the third equation is due to $\gamma^\mathcal{A} = \gamma_0 + \delta^\mathcal{A}$ and ${\gamma_0}_{S^c} = 0$. 
Therefore, a direct upper bound can be obtained by $\|\hat{u} \|_2 \leqslant \|\hat{u} \|_1 \leqslant 9.34 h$. Furthermore, as $\|\hat u\|_2 \leqslant \|\hat u\|_1$, we have $\|\hat u\|_2^2 \leqslant \|\hat u\|_1^2 \leqslant 100 h^2$. 
\\
\indent
On the other hand, we apply Theorem \ref{thm:RSC} in Equation \eqref{eq:transfer_bound_2}, we obtain with high probability:   
\begin{align*}
\kappa\|\hat{u}\|_2^2 - C_1 \log d \|\hat{u}\|_1 \|\hat{u}\|_2 \frac{\sum_{k} \sqrt{\Psi(p_k) n_k}}{n} - C_2 \frac{\sqrt{\log d}}{n} \|\hat{u}\|_1^2 & \leqslant \frac{1}{2n}\sum_{k} \| \hat{\bZ}_k \hat{u}\|_2^2 \\
& \leqslant 4\lambda_{\gamma}\|\gamma_{S^c}^\mathcal{A} \|_{1} \leqslant 4  \lambda_{\gamma}h  \, ,
\end{align*}
From $\|\hat u\|_2 \leqslant \|\hat u\|_1 \lesssim h$, we have now, 
$$
\|\hat u\|_2^2 \lesssim \left(h^2 \frac{\log{d}}{n}  \sum_k \sqrt{\Psi(p_k) n_k} + h^2 \frac{\sqrt{\log d}}{n} + \lambda_{\gamma}h
\right) \wedge h^2 \,.
$$
As $\lambda_{\gamma} \leqslant \sqrt{\log d}/n$, and $h \sqrt{\log d \Psi(p_k)} = o(1)$ (which we assume it in theorem), we can demonstrate the first and second terms in the parentheses is negligible. Then we have: 
$$
\|\hat u\|_2^2 \lesssim \lambda_{\gamma}h \wedge h^2 \,.
$$
Therefore, whether it is situation 1 or situation 2, we can provide an upper bound
\begin{equation}
	\label{eq: upper gammaA with lambda}
	\left\|\hat{u} \right\|_2^2 \leqslant C\left[{s \lambda_{\gamma}^2} + \left(h^2 \wedge {\lambda_{\gamma} h} \right)\right].
\end{equation}

\noindent
{\bf Step 2:} Next, we provide an upper bound on $\lambda_\gamma$. Recall that in Step 1, we have already 
mentioned that we need to choose $\lambda_\gamma$ such that  $\lambda_{\gamma} \geqslant 2 n^{-1} \| \sum_{k} \hat{\bZ}_k^{\top} (\by_k - \hat{\bZ}_k \gamma^\mathcal{A}) \|_{\infty}$. 
As $\by_k = \bZ_k \gamma_k + \epsilon_k$, we have: 
\begin{align*}
\frac1n \| \sum_{k} \hat{\bZ}_k^{\top}(\bZ_k \gamma_k -\hat{\bZ}_k \gamma^\mathcal{A} + \epsilon_k ) \|_{\infty} & \leqslant  \frac1n \| \sum_{k} \hat{\bZ}_k^{\top} \epsilon_k \|_{\infty} +  \frac1n \| \sum_{k} \hat{\bZ}_k^{\top} \bZ_k ( \gamma_k - \gamma^\mathcal{A} ) \|_{\infty}  \\
& \qquad \qquad \qquad \qquad \qquad + \frac1n \| \sum_{k} \hat{\bZ}_k^{\top}(\bZ_k -\hat{\bZ}_k ) \gamma^\mathcal{A} \|_{\infty} \,.
\end{align*}
% $\lambda_{\gamma} =  n^{-1} \| \sum_{k} \hat{\bZ}_k^{\top}(\bZ_k \gamma_k -\hat{\bZ}_k \gamma^\mathcal{A} + \epsilon_k ) \|_{\infty} \leqslant n^{-1} \| \sum_{k} \hat{\bZ}_k^{\top} \epsilon_k \|_{\infty} +  n^{-1} \| \sum_{k} \hat{\bZ}_k^{\top} \bZ_k ( \gamma_k - \gamma^\mathcal{A} ) \|_{\infty}  +  n^{-1} \| \sum_{k} \hat{\bZ}_k^{\top}(\bZ_k -\hat{\bZ}_k ) \gamma^\mathcal{A} \|_{\infty}$, 
Therefore, we need to analyze these three terms.

\textbf{Step 2.1: Bound for the first term  $n^{-1} \| \sum_{k}\hat{\bZ}_k^{\top} \epsilon_k \|_{\infty}$.} 
By definition of $\hat \bZ$, this term can be expressed as $n^{-1} \| [ \sum_k \bX_k^{\top} \epsilon_k,   \sum_k (\hat{A}_k \bX_k )^{\top} \epsilon_k ]^{\top} \|_{\infty} $, hence it is the maximum of $n^{-1} \| \sum_k \bX_k^{\top}  \epsilon_k \|_{\infty}$ and $n^{-1} \| \sum_k (\hat{A}_k\bX_k )^{\top}  \epsilon_k \|_{\infty}$. 
Define a diagonal block matrix $D_{\hat{A}}$, where the $k$-th block on the diagonal is $\hat{A}_k$, and all other blocks are zero matrices. Similarly, define $D_{{A_k^*}^\top A_k^*}$, where the $k$-th block on the diagonal is ${A_k^*}^\top A_k^*$:
$$
\begin{array}{c}
	D_{\hat{A}} = \begin{pmatrix}
		\hat{A}_1 & 0 & \cdots & 0 \\
		0 & \hat{A}_2 & \cdots & 0 \\
		\vdots & \vdots & \ddots & \vdots \\
		0 & 0 & \cdots &\hat{A}_K \\
	\end{pmatrix}
	,\qquad
	D_{{A_k^*}^\top A_k^*} = \begin{pmatrix}
		{{A}_1^*}^{\top} {A}_1^* & 0 & \cdots & 0 \\
		0 & {{A}_2^*}^{\top} {{A}_2^*} & \cdots & 0 \\
		\vdots & \vdots & \ddots & \vdots \\
		0 & 0 & \cdots & {{A}_K^*}^{\top} {{A}_K^*} \\
	\end{pmatrix}.
\end{array}
$$
Similarly, define $\bX_{\operatorname{full}} \in \mathbb{R}^{n \times d}$ as the feature matrix for all individuals, and $\epsilon_{\operatorname{full}} \in \mathbb{R}^{n}$ as noise vector for all individuals. It is easy to observe that $ n^{-1} \| \sum_{k}( \hat{A}_k \bX_k )^{\top} \epsilon_k\|_{\infty} = n^{-1} \| \bX_{\operatorname{full}}^{\top} D_{\hat{A}}^{\top} \epsilon_{\operatorname{full}} \|_{\infty}$. 
Conditioning on $\cA = \{A_1, \dots, A_k\}$, we can upper bound the $j$-th element of $n^{-1} \bX_{\operatorname{full}}^{\top} D_{\hat{A}}^{\top} \epsilon_{\operatorname{full}}$ using Lemma \ref{lemma:hwinequalityforXY} as follows: 
% Let us treat all of the standardized adjacency matrix $\hat{A}_k$ as fixed, by the independence between source tasks and lemma \ref{lemma:hwinequality}, the $j$-th element of $n^{-1} \| \bX_{\operatorname{full}}^{\top} D_{\hat{A}}^{\top} \epsilon_{\operatorname{full}} \|_{\infty}$ can be bounded by 
$$
\begin{aligned}
	\Pr \left( \left|\frac1n X_{\operatorname{full},j}^{\top} D_{\hat{A}}^{\top} \epsilon \right| \geq t \mid \cA \right)  \leqslant 2\exp{\left(-c\min\left(\frac{n^2 t^2}{\|D_{\hat{A}}\|_F^2}, \frac{nt}{\| D_{\hat{A}}\|_{2}} \right)\right)} .
\end{aligned}
$$ From lemma \ref{lemma:ubforA}, we have $
\|\hat{A}_k \|_F^2 \leqslant 2 n_k$ with probability $1 - {n_k}^{-1}$ and $\|\hat{A}_k \|_2 \leqslant 2 \sqrt{n_k p_k}$ with probability $1 - \exp(\log{n_k} - n_k p_k/c)$. By applying simple matrix algebra, we can verify that the Frobenius norm and $\ell_2$ norm of $D_{\hat{A}}$ satisfying that $ \|D_{\hat{A}}\|_F^2 = \sum_k \|\hat{A}_k\|_F^2 \leqslant 2n$ with probability $1 - \sum_{k=1}^{K} {n_k}^{-1}$ and $\|D_{\hat{A}}\|_2 = \max_k\|\hat{A}_k \|_2 \leqslant  2 \max_k(\sqrt{n_k p_k})$ with probability $1 - \sum_{k=1}^{K} \exp(\log{n_k} - n_k p_k/c)$. Then $$
\begin{aligned}
	\Pr \left( \frac1n  \left\|\bX_{\operatorname{full}}^{\top} D_{\hat{A}}^{\top} \epsilon \right\|_{\infty} \geq t \mid \hat{A}_k \right) & \leqslant 2\exp{\left(\log d -c\min\left(\frac{n^2 t^2}{\|D_{\hat{A}}\|_F^2}, \frac{nt}{\| D_{\hat{A}}\|_{2}} \right)\right)}\\
	& \leqslant 2\exp{\left(\log d -c\min\left(\frac{n^2 t^2}{n}, \frac{nt}{\max_k\left\lbrace \sqrt{n_k p_k}\right\rbrace} \right)\right)}.
\end{aligned}
$$ Choosing $
t = c \sqrt{\log{d}/n}$, under assumption (C3) $p_k \log d = o(1)$, we conclude: 
\begin{align*}
& \bbP \left(n^{-1} \| \sum_k (\hat{A}_k \bX_k )^{\top}  \epsilon_k \|_{\infty} \lesssim \sqrt{{\log{d}}/{n}}\right) \\
& \qquad \qquad \geqslant 1- d^{-1} - \sum_{k=1}^{K} \exp(\log{n_k} - n_k p_k/c) -  \sum_{k=1}^{K} {n_k}^{-1} \,.
\end{align*}
For the second part $n^{-1} \| \sum_k \bX_k^{\top} \epsilon_k \|_{\infty}$, the same method can be used to derive the same conclusion. So $n^{-1}\|\sum_{k}\hat{\bZ}_k^{\top} \epsilon_k \|_{\infty} \lesssim \sqrt{{\log{d}}/{n}}$, with high probability  converge to 1.

\textbf{Step 2.2: Bound for the second term $\|\sum_{k} \hat{\bZ}_k^{\top} \bZ_k \left( \gamma_k - \gamma^{\mathcal{A}}\right)\|_{\infty}$.} 
This can be upper bounded by: 
\begin{align*}
    & \|\sum_{k} \hat{\bZ}_k^{\top} \bZ_k \left( \gamma_k - \gamma^{\mathcal{A}}\right)\|_{\infty} \\
    & \qquad \leqslant \|\sum_{k} X_k^\top \hat{A}_k^\top A^*_k X_k  (\delta_{k,0} - \delta_{\cdot,0}) \|_{\infty} +  \|\sum_{k} X_k^\top \hat{A}_k^\top X_k (\delta_{k,1} - \delta_{\cdot,1})\|_{\infty} \\
    & \qquad \qquad + \|\sum_{k} X_k^\top {A}_k^* X_k (\delta_{k,0} - \delta_{\cdot,0})\|_{\infty} + \|\sum_{k} X_k^\top X_k (\delta_{k,1} - \delta_{\cdot,1})\|_{\infty} \,,
\end{align*}
% It can be bounded by $\|\sum_{k} X_k^\top \hat{A}_k^\top A^*_k X_k  (\delta_{k,0} - \delta_{\cdot,0}) \|_{\infty} +  \|\sum_{k} X_k^\top \hat{A}_k^\top X_k (\delta_{k,1} - \delta_{\cdot,1})\|_{\infty} + \|\sum_{k} X_k^\top {A}_k^* X_k (\delta_{k,0} - \delta_{\cdot,0})\|_{\infty} + \|\sum_{k} X_k^\top X_k (\delta_{k,1} - \delta_{\cdot,1})\|_{\infty}$, 
where $\delta_k = \gamma_k - \gamma_0$ and $\delta = \gamma^\mathcal{A} - \gamma_0$, vector $\delta_{k,0}$ and $\delta_{\cdot,0}$ is the first $d$ components of the vector $\delta_k$ and $\delta$, $\delta_{k,1}$ and $\delta_{\cdot,1}$ is the last $d$ components of the $\delta_k$ and $\delta$. 
To bound each of the terms of the right-hand side of the above equation, we use similar techniques as before. 
For the first term, we can further decompose it as: 
\begin{align*}
    & \frac1n \left\| \sum_{k} \bX_k^\top \hat{A}_k^\top A^*_k X_k  (\delta_{k,0} - \delta_{\cdot,0}) \right\|_{\infty} \\
    \leqslant & \frac1n \left\| \sum_{k} X_k^\top (\hat{A}_k - A^*_k )^\top A^*_k X_k  (\delta_{k,0} - \delta_{\cdot,0}) \right\|_{\infty} + \frac1n \left\| \sum_{k} ( X_k^\top {A^*_k}^\top A^*_k X_k - n_k \Sigma_X)  (\delta_{k,0} - \delta_{\cdot,0}) \right\|_{\infty} \\
    \leqslant & 2h (M_1 + M_2) \,,
\end{align*}
% We take the first term as an example: $n^{-1} \| \sum_{k} \bX_k^\top \hat{A}_k^\top A^*_k X_k  (\delta_{k,0} - \delta_{\cdot,0}) \|_{\infty} \leqslant  n^{-1} \| \sum_{k} X_k^\top (\hat{A}_k - A^*_k )^\top A^*_k X_k  (\delta_{k,0} - \delta_{\cdot,0}) \|_{\infty} + n^{-1} \| \sum_{k} ( X_k^\top {A^*_k}^\top A^*_k X_k - n_k \Sigma_X)  (\delta_{k,0} - \delta_{\cdot,0}) \|_{\infty} \leqslant 2h (M_1 + M_2)$, 
where $M_1$ and $M_2$ are the maximum elements of the matrices $n^{-1} \sum_{k} X_k^\top (\hat{A}_k - A^*_k)^\top A^*_k X_k$ and $ n^{-1}\sum_{k} (X_k^\top {A^*_k}^\top A^*_k X_k - n_k \Sigma_X)$, respectively, and the first inequality holds as $\sum_{k} n_k \delta_k = \sum_{k} n_k \delta$. 
It is worth noting that \(\mathrm{E} X_k^\top (A^*_k)^\top A^*_k X_k = n_k \Sigma_X\), and its verification can be found in the proof of Lemma \ref{lemma: expectationZTZ}. The upper bound of \(M_1\) can be obtained by following the approach used in formula \eqref{eq:2 and F norm for hat A * A}, so we omit it here. The upper bound of $M_2$ can be concluded by applying Lemma \ref{lemma:hwinequalityforXY} on each element of the matrix; the $i,j$-th element of $n^{-1}\sum_{k} \bX_k^\top (A^*_k)^\top A^*_k \bX_k$ can be written in the form of a quadratic form as $(\bX_{\operatorname{full},i}^{\top} D_{{A^*_k}^\top A^*_k} \bX_{\operatorname{full},j}) /n$, so
\begin{align*}
& \bbP\left(\frac1n \left| \left(\bX_{\operatorname{full}}^\top D_{{A^*_k}^\top A^*_k} \bX_{\operatorname{full}}\right)_{ij} - \Sigma_{X,ij}\right| \geqslant t\right) \\
& \qquad \leqslant \exp  \left\{-c \min \left({n^2 t^2}/{\|D_{{A^*_k}^\top A^*_k}\|_F^2}, {nt}/{\| D_{{A^*_k}^\top A^*_k} \|_{2}}\right)\right\} 
\end{align*}
and 
$$
\bbP \{| M_2 | \geqslant t\} \leqslant \exp  \left[2\log d - c\min \{n^2 t^2/\sum_k(n_k + n_k^2 p_k^2) , {nt}/{\max_k (n_k p_k)}\}\right]. $$
We take $t = C \max \{\sqrt{\log d \sum_k (n_k + n_k^2 p_k^2) }/n , {\log d \max_k (n_k p_k) }/{n}\}$. 
Combining the above conclusions, we obtain: 
$$
\left\| \frac1n\sum_{k} \bX_k^\top \hat{A}_k^\top A^*_k X_k  \left(\delta_{k,0} - \delta_{\cdot,0}\right) \right\|_{\infty} \lesssim \frac{h}{n} \max \left\{\sqrt{\log d \sum_k (n_k + n_k^2 p_k^2) } , {\log{d} \, \max_k (n_k p_k) } \right\}.
$$ 
Applying similar techniques to the remaining parts, we can obtain the second part of \(\lambda_{\gamma}\) will be bound by $h \max \{\sqrt{\log d \sum_k (n_k + n_k^2 p_k^2)}, {\log d \max_k (n_k p_k) }\} / n$.
\\\\
\noindent
\textbf{Step 2.3: Bound for the third term $n^{-1} \| \sum_{k} \hat{\bZ}_k^{\top}(\bZ_k -\hat{\bZ}_k ) \gamma^\mathcal{A} \|_{\infty}$.} The third term of $\lambda_\gamma$ can be written as $\max \{\| \sum_k X_k^{\top} (A_k^* - \hat{A}_k) X_k \beta_0^\mathcal{A} \|_{\infty} ,  \| \sum_k X_k^{\top} \hat{A}_k^{\top} (A_k^* - \hat{A}_k) X_k \beta_0^\mathcal{A} \|_{\infty}\}$. Applying similar techniques, we could obtain
$$\bbP  \left[\max_{i,j} n^{-1} |  \{\sum_k X_k^{\top} (A_k^* - \hat{A}_k) X_k\}_{ij} | \geqslant t\right] \leqslant \exp \left\{2\log d - c \min ( n^2 t^2/K, nt/\max_k \sqrt{p_k} ) \right\}$$ and 
\begin{align*}
& \bbP \left[\max_{i,j} n^{-1} |\{\sum_k X_k^{\top} \hat{A}_k^{\top} (A_k^* - \hat{A}_k) X_k\}_{ij} | \geqslant t\right] \\
& \qquad \le \exp \left[2\log d - c\min  \left\{{n^2 t^2}/{\sum_k ( 1 + n_k p_k^2) }, nt/\max_k (p_k n_k^{-1/2})\right\}\right] \,.
\end{align*}
% $$\bbP \left[\max_{i,j} n^{-1} |\{\sum_k X_k^{\top} \hat{A}_k^{\top} (A_k^* - \hat{A}_k) X_k\}_{ij} | \geqslant t\right] \leqslant \exp \left[2\log d - c\min  \left\{{n^2 t^2}/{\sum_k ( 1 + n_k p_k^2) }, nt/\max_k (p_k n_k^{-1/2})\right\}\right].$$
We could choose $t = C\sqrt{\log d \sum_{k}(1+n_k p_k^2)}/n$ to ensure that the probability above tends to zero. 
Under the assumption (C4) $\|\beta_{0k}\|_1 \leqslant \sqrt{n/\sum_k ( 1+ n_k p_k^2)}$, we can conclude that $\| \sum_{k} \hat{\bZ}_k^{\top} (\bZ_k -\hat{\bZ}_k) \gamma^\mathcal{A} \|_{\infty} \lesssim \sqrt{\log{d}/n}$.
\\\\
\noindent
\textbf{Step 2.4: Upper bound of $\lambda_\gamma$.} Combining the above, we obtain the upper bound for $\lambda_{\gamma}$ as $\sqrt{\log d / n} + h \max \{\sqrt{\log d \sum_k (n_k + n_k^2 p_k^2)}, \log d \max_k (n_k p_k)\} / n$.
%\textbf{Remark:} We consider two extreme cases. The first case is when all source tasks have the same $n_k = n/K$ and connection probability $p_k = p$. In this case, the upper bound is $\max \{\sqrt{(1/n + p^2/K) \log d} , p \log d/K\}$. The second case is when one source task has a particularly large amount of data, which mean $n_1 \approx n$, and $p_1 = p$. In this case, the upper bound is $\max \{\sqrt{(1/n + p^2) \log d} , p \log d\} $.
\\\\
\noindent
\textbf{Step 3 (Debiasing):} Here, we will focus on the Debiasing Step of the algorithm \ref{al:trans_alg}, and provide its error bound, $\| \hat{\delta} - \delta\|_2$. 
The estimator $\hat \delta$ is obtained by 
minimizing $(2n_0)^{-1} \|\by_0 - \hat{\bZ}_0 ( \hat{\gamma}^\mathcal{A} - \delta) \|_{2}^{2} + \lambda_{\delta} \|{\delta}\|_{1}$. 
% value of the debiased object function: $(2n_0)^{-1} \|\by_0 - \hat{\bZ}_0 ( \hat{\gamma}^\mathcal{A} - \delta) \|_{2}^{2} + \lambda_{\delta} \|{\delta}\|_{1}$. 
Define $\hat{v} = \hat{\delta} - \delta$. We have:  
$$ 
\frac{1}{2n_0}\left\|y_0 - \hat{\bZ}_0 \left( \hat{\gamma}^\mathcal{A} -\delta - \hat{v} \right)  \right\|_{2}^{2}+\lambda_{\delta}\left\|{\hat{\delta}}\right\|_{1} \leqslant \frac{1}{2n_0}\left\|y_0 - \hat{\bZ}_0 \left( \hat{\gamma}^\mathcal{A} -\delta \right)  \right\|_{2}^{2}+\lambda_{\delta}\left\|{{\delta}}\right\|_{1} \,.
$$ 
This, after some simple algebraic manipulation, yields: 
\begin{equation}
\label{eq:debias_1}
\frac{1}{2n_0}\left\| \hat{\bZ}_0 \hat{v}  \right\|_{2}^{2} \leqslant \lambda_{\delta}\left\|{{\delta}}\right\|_{1} - \lambda_{\delta}\left\|{\hat{\delta}}\right\|_{1} + \frac{1}{n_0}\left| \left\langle \epsilon_0 + (\bZ_0 -\hat{\bZ}_0 ) \beta_0 -\hat{\bZ}_0 \hat{u} ,  \hat{\bZ}_0 \hat{v} \right\rangle \right| \,.
\end{equation}
Further analysis of the inner product term yields
\begin{equation}
\label{eq:debias_2}
\begin{aligned}
	\frac{1}{n_0}\left| \left\langle \epsilon_0 + (\bZ_0 -\hat{\bZ}_0 ) \beta_0 -\hat{\bZ}_0 \hat{u} ,  \hat{\bZ}_0 \hat{v} \right\rangle\right| \leqslant & \frac{1}{n_0}\left| \left\langle \epsilon_0 + (\bZ_0 -\hat{\bZ}_0 ) \beta_0 ,  \hat{\bZ}_0 \hat{v} \right\rangle\right| + \frac{1}{n_0}\left| \left\langle  \hat{\bZ}_0 \hat{u} ,  \hat{\bZ}_0 \hat{v} \right\rangle\right|\\
	\leqslant & \frac{\lambda_{\delta}}{2} \left\|\hat{v} \right\|_1 + \frac{1}{n_0}\left\|\hat{\bZ}_0 \hat{u} \right\|_2^2 + \frac{1}{4 n_0}\left\|\hat{\bZ}_0 \hat{v} \right\|_2^2 ,
\end{aligned}
\end{equation}
where the first term comes from H{\"o}lder inequality by denoting $\lambda_\delta = 2\{\hat{\bZ}_0^\top \epsilon_0 + \hat{\bZ}_0^\top(\bZ_0 -\hat{\bZ}_0 ) \beta_0\}$, while the second term comes from the inequality $|ab | \leqslant c a^2/2 + b^2/2c $ and let $c=2$. 
Combining the bounds in Equation \eqref{eq:debias_1} and \eqref{eq:debias_2} and using the fact $\|{\hat{\delta}}\|_{1} \geqslant \|\hat{v}\|_1 - \| \delta \|_{1}$, we conclude: 
$$
\frac{1}{4n_0}\left\| \hat{\bZ}_0 \hat{v}   \right\|_{2}^{2} \leqslant 2\lambda_{\delta}\left\|{{\delta}}\right\|_{1} - \frac{\lambda_{\delta}}{2} \left\|\hat{v} \right\|_1 + \frac{1}{n_0}\left\|\hat{\bZ}_0 \hat{u} \right\|_2^2
$$
Next, we consider two different situations:

\noindent
{\bf Situation 1:} If $2 \lambda_{\delta} \|\delta \|_{1} \geqslant n_0^{-1} \|\hat{\bZ}_0 \hat{u} \|_2^2$, we have:  
$$
0 \leqslant \frac{1}{4n_0} \| \hat{\bZ}_0 \hat{v}\|_{2}^{2} \leqslant  4 \lambda_{\delta} \|\delta\|_{1} - \frac{\lambda_{\delta}}{2} \|\hat{v}\|_1 \,.
$$
The above inequality immediately implies: 
\begin{equation}
\label{eq:upper_bound_h_debiased}
 4 \lambda_{\delta} \|\delta\|_{1} - \frac{\lambda_{\delta}}{2} \|\hat{v}\|_1 \ge 0 \implies \| \hat{v} \|_2 \leqslant \|\hat{v} \|_1 \leqslant 8\|\delta \|_1 \leqslant 8h \,.
\end{equation}
Furthermore, using Theorem \ref{thm:RSC} we obtain: 
$$
\kappa \| \hat{v}\|_{2}^{2} - C_1 \log d \|\hat{v}\|_{1} \|\hat{v}\|_{2} \sqrt{\frac{\Psi(p_0)}{n_0}} - C_2 \frac{\sqrt{\log d}}{n_0} \|\hat{v}\|_{1}^2 \leqslant \frac{1}{4n_0} \|  \hat{\bZ}_0 \hat{v} \|_{2}^{2} \leqslant 4 \lambda_{\delta} \|\delta\|_{1}
$$
% and conclude two facts from inequality above: (1) $\| \hat{v} \|_2 \leqslant \|\hat{v} \|_1 \leqslant 8\|\delta \|_1 \leqslant 8h$ and
% (2) ${\color{green} \kappa \| \hat{v}\|_{2}^{2} - C \log d \|\hat{v}\|_{1} \|\hat{v}\|_{2} \sqrt{\Psi(p_0)/n_0} \leqslant (4n_0)^{-1} \|  \hat{\bZ}_0 \hat{v} \|_{2}^{2} \leqslant 4 \lambda_{\delta} \|\delta\|_{1}}$, where the first fact comes from the inequality $\|\hat{v}\|_2 \leqslant \|\hat{v}\|_1$ and $0 \leqslant 4 \lambda_{\delta} \|\delta\|_{1} - \lambda_{\delta} \|\hat{v}\|_1 / 2$ , and the second fact comes from the {\color{green} RES condition}. 
Hence, the above inequality implies: 
\begin{align}
\label{eq:first_bound}
    \kappa \| \hat{v}\|_{2}^{2} & \leqslant C_1 \log d \|\hat{v}\|_{1} \|\hat{v}\|_{2} \sqrt{\frac{\Psi(p_0)}{n_0}} + C_2 \frac{\sqrt{\log d}}{n_0} \|\hat{v}\|_{1}^2 + 4 \lambda_{\delta} \|\delta\|_{1} \notag \\
    & \leqslant C_1 h\log{d} \|\hat{v}\|_{2} \sqrt{\frac{\Psi(p_0)}{n_0}} + C_2 \frac{\sqrt{\log d}}{n_0}h^2 + 4 \lambda_{\delta} h \notag \\
    & \leqslant \frac{\kappa}{2}\|\hat v\|_2^2 +  h^2 \left(C_3 \log^2{d}\frac{\Psi(p_0)}{n_0} + C_2 \frac{\sqrt{\log d}}{n_0}\right)  + 4\lambda_\delta h \hspace{.5in}  \left[ ab \leqslant \frac{a^2}{2} + \frac{b^2}{2} \right]  \notag \\
    \implies \|\hat v\|_2^2 & \lesssim h^2 \left( \log^2{d}\frac{\Psi(p_0)}{n_0} + \frac{\sqrt{\log d}}{n_0} \right)  + \lambda_\delta h
\end{align}
We next claim that $h\lambda_\delta$ is of the larger order. This would be true if i) $h\sqrt{\log{d}}/n_0 \ll \lambda_\delta$ and ii) $h\log^2{d}\Psi(p_0)/n_0 \ll \lambda_\delta$. Given that $\lambda_\delta = C\sqrt{\log{d}/n_0}$, the first condition is satisfied as soon as $h \ll \sqrt{n_0}$ which is trivially true given our assumptions. For the second condition to be satisfied, we need $h \log^{3/2}{d} \Psi(p_0) \ll \sqrt{n_0}$, which is equivalent to the condition $h\sqrt{\log{d}\Psi(p_0)} \ll \sqrt{(n_0/\Psi(p_0))}/\log{d} \approx \sqrt{n_0p_0}/\log{d}$. As we have assumed $h\sqrt{\log{d} \Psi(p_0)}= o(1)$ and $\sqrt{n_0p_0}/\log{d} = \Omega(1)$ (Assumption (C3)), this condition is also easily satisfied. 
Therefore, we conclude from Equation \eqref{eq:first_bound} that $\|\hat v\|_2^2 \lesssim h\lambda_\delta$, which, in combination with Equation \eqref{eq:upper_bound_h_debiased}, yields $\|\hat v\|_2^2 \lesssim (h\lambda_\delta \wedge h^2)$. 

Since $\lambda_{\delta}$ has the same form as \(\lambda\) in Appendix C.1, the same argument can be used to show that $\lambda_{\delta} \lesssim \sqrt{\log d/n_0}$ with high probability. So under this situation, $\|\hat{v}\|_{2}^2 \lesssim  h \sqrt{\log d / n_0} \wedge h^2$.
\\\\
\noindent
{\bf Situation 2:} If $ 2 \lambda_{\delta} \|\delta \|_{1} \leqslant n_0^{-1} \|\hat{\bZ}_0 \hat{u} \|_2^2$, then we have: 
$$
0 \leqslant \frac{1}{4n_0}\|\hat{\bZ}_0 \hat{v} \|_{2}^{2} \leqslant \frac{2}{n_0}\|\hat{\bZ}_0 \hat{u} \|_2^2 - \frac{\lambda_{\delta}}{2} \|\hat{v} \|_1 
$$
An immediate conclusion from the above inequality is that $\|\hat v\|_1 \leqslant (1/2\lambda_\delta)n_0^{-1}\|\hat \bZ_0 \hat u\|_2^2$. 
Furthermore, Theorem \ref{thm:RSC} implies: 
\begin{equation}
\label{eq:RSC_debiased_2}
\kappa \|\hat v\|_2^2 -  C_1 \log d \|\hat{v}\|_{1} \|\hat{v}\|_{2} \sqrt{\frac{\Psi(p_0)}{n_0}} - C_2 \frac{\sqrt{\log d}}{n_0} \|\hat{v}\|_{1}^2  \leqslant \frac{2}{n_0}\|\hat{\bZ}_0 \hat{u} \|_2^2 - \frac{\lambda_{\delta}}{2} \|\hat{v} \|_1 \,.
\end{equation}
% As $\|\hat v\|_1 \ge \|\hat v\|_2$, we have: 
% $$
% \frac{\kappa}{2}\|\hat v\|_2^2 \le  \left(\kappa - C \log d \sqrt{\frac{\Psi(p_0)}{n_0}}\right)\|\hat v\|_2^2 \le \frac{2}{n_0}\|\hat{\bZ}_0 \hat{u} \|_2^2 - \frac{\lambda_{\delta}}{2} \|\hat{v} \|_1 \,.
% $$
% \DM{Here, as in the previous situation, I have used the assumption that $\log^2{d} \Psi(p_0)/n_0 = o(1)$. 
% Now, we argue with proper justification that $n_0^{-1} \|\hat{\bZ}_0 \hat{u} \|_2^2 \leqslant 2 \lambda_{\max}(\Sigma_Z) \left\|\hat{u} \right\|_2^2$. This would imply: 
% $$
% \frac{\kappa}{2}\|\hat v\|_2^2 \le  4 \lambda_{\max}(\Sigma_Z) \left\|\hat{u} \right\|_2^2 \implies \|\hat v\|_2^2 \lesssim \|\hat u\|_2^2 \,.
% $$
% Finally, combining the bounds from both situations, along with the assumption $\log^2{d} \Psi(p_0)/n_0 = o(1)$, we conclude: 
% $$
% \|\hat v\|_2^2 \lesssim \|\hat u\|_2^2 + \left(h\lambda_\delta \wedge h^2\right) \,.
% $$
% }
% $0 \leqslant (4n_0)^{-1} \|\hat{\bZ}_0 \hat{v} \|_{2}^{2} \leqslant 2 n_0^{-1} \|\hat{\bZ}_0 \hat{u} \|_2^2 - \lambda_{\delta} \|\hat{v} \|_1 /2$. 
From Equation \eqref{eq:RSC_debiased_2}, we have, upon applying Young's inequality: 
\begin{equation*}
\begin{aligned}
\label{eq:RSC_debiased_1}
\kappa \|\hat v\|_2^2 & \leqslant \frac{2}{n_0}\|\hat{\bZ}_0 \hat{u} \|_2^2 + C \log d \|\hat{v}\|_{1} \|\hat{v}\|_{2} \sqrt{\frac{\Psi(p_0)}{n_0}} + \frac{\sqrt{\log d}}{n_0}\|\hat{v}\|_{1}^2  - \frac{\lambda_{\delta}}{2} \|\hat{v} \|_1 \\
& \leqslant \frac{2}{n_0}\|\hat{\bZ}_0 \hat{u} \|_2^2 + \frac{\kappa}{2}\|\hat v\|_2^2 + C_1 \frac{\log^2{d} \ \Psi(p_0)}{n_0}\|\hat v\|_1^2 + C_2\frac{\sqrt{\log d}}{n_0}\|\hat{v}\|_{1}^2 - \frac{\lambda_{\delta}}{2} \|\hat{v} \|_1 \,.
\end{aligned}
\end{equation*}
Easy to verify that the forth term is negligible compared to the third term. Therefore, we can conclude $\|\hat v\|_2^2 \lesssim n_0^{-1}\|\hat \bZ_0 \hat u\|_2^2$ when
\begin{equation}
\label{eq:want to proof in debais}
\frac{\log^2{d} \ \Psi(p_0)}{n_0}\|\hat v\|_1 \asymp \lambda_\delta^2 \log{d} \ \Psi(p_0)\|\hat v\|_1 \ll  \frac{\lambda_{\delta}}{2} \Longleftarrow \lambda_\delta \|\hat v\|_1 \ll (\log{d} \ \Psi(p_0))^{-1} \,.
\end{equation}
As we have already pointed out $\|\hat v\|_1 \lesssim \lambda_\delta^{-1} n_0^{-1}\|\hat \bZ_0 \hat u\|_2^2$, and using $n_0^{-1} \|\hat{\bZ}_0 \hat{u} \|_2^2 \leqslant 2 \lambda_{\max}(\Sigma_Z)\|\hat u\|_2^2$, we have: 
$$
\lambda_\delta \|\hat v\|_1 \lesssim \|\hat u\|_2^2 \lesssim s\lambda_\gamma^2 + (h\lambda_\gamma \wedge h^2) \, .
$$
Under assumptions (in theroem)
$$
h\sqrt{\log{d} \ \Psi(p_0)} = o(1) \ \ \text{ and }\ \ s\lambda_\gamma^2 \log{d} \ \Psi(p_0) = o(1) \,.
$$
we have \eqref{eq:want to proof in debais}, which implying $\|\hat v\|_2^2 \leqslant \|\hat u\|_2^2$.

In summary, we control $\|\hat{v}\|_{2}^2$ by $\|\hat{v}\|_{2}^2 \lesssim  ( h \sqrt{\log d/n_0} \wedge h^2) \vee \|\hat{u} \|_2^2 $.
Combining the above inequalities, we obtain,
\begin{equation}
	\left\|\hat{\gamma_0} - \gamma_0 \right\|_{2}^2 \leqslant \left\|\hat{u} \right\|_2^2 + \left\|\hat{v} \right\|_2^2 \lesssim s \lambda_{\gamma}^2 + \left(  h^2 \wedge \lambda_{\gamma}  h \right)  + \left(  h^2 \wedge \lambda_{\delta}  h \right) ,
\end{equation}
where $\lambda_{\gamma} = \sqrt{\log d/n} + h \max \{\sqrt{\log d \sum_k (n_k + n_k^2 p_k^2) } , \log d \max_k (n_k p_k)\}/n $ and $\lambda_{\delta} = \sqrt{\log d/n_0}$.

\color{black}

\section{Proof of auxiliary lemmas and RSC condition}

\subsection{Proof of lemma \ref{lemma:hwinequalityforXY}}
\begin{proof}
Define $Z = (X^{\top}, Y^{\top})^{\top} \in \mathbb{R}^{n_1 + n_2}$ as the new random vector. It is easy to observe that $Z$ is a random vector with independent, mean-zero, sub-Gaussian coordinates, and its $\psi_2$-norm is $K$. Define the anti-diagonal matrix $\mathbb{B} \in \mathbb{R}^{(n_1 + n_2)\times (n_1 + n_2)}  $, where the upper-right corner is $B^{\top}$, the lower-left corner is $B$, and the other elements are 0. It can be verified that $Z^{\top} \mathbb{B} Z = 2 X^{\top} B Y$, $\mathbb{E}Z^{\top} \mathbb{B} Z = 0$, and that $\|\mathbb{B}\|_F^2 \lesssim \|B\|_F^2$ and $\|\mathbb{B}\|_2 \lesssim \|B\|_2$. By applying the Hanson-Wright inequality to $Z^{\top} \mathbb{B} Z$, we can obtain the result stated in Lemma \ref{lemma:hwinequalityforXY}. 
\end{proof}

\subsection{Proof of Theorem \ref{thm:RSC}}
\begin{proof}
    The proof of the RE condition is similar to the proof of Proposition 2 of \cite{negahban2009unified}. However, the technicalities are different due to the dependence among observations. We use a similar truncation argument; following the proof of Proposition 2 of \cite{negahban2009unified}, we define a function $\phi_\tau(x)$ which takes value $x^2$ in $[-\tau/2, \tau/2]$, $(\tau - x)^2$ in $[-\tau, -\tau/2] \cup [\tau/2, \tau]$ and $0$ outside $[-\tau, \tau]$. Then for some fixed $0 < \tau \le T$ (to be chosen later), we have: 
\begin{align}
\label{eq:RSC_bound_1}
    \frac1n \sum_i (u^\top \bZ_i)^2 \ge \frac1n \sum_i \phi_\tau\left((u^\top \bZ_i)^2\mathds{1}_{|\gamma_0^\top Z_i| \le T}\right) \triangleq \frac1n \sum_i g_{\tau, T}(u^\top Z_i)
\end{align}
Our first target is to show that the expected value of the right-hand side of Equation \ref{eq:RSC_bound_1} is further lower bounded by some constant with high probability. Toward that end, we establish some moment bounds: 
\\\\
\noindent
{\bf Step 1: }We have already established that: 
\begin{align*}
    \bbE\left[\frac{u^\top \bZ^\top \bZ u}{n}\right] & = u^\top (I_2 \otimes \Sigma_X) u \ge \kappa
\end{align*}
where $\kappa$ is a lower bound on the minimum eigenvalue of $\Sigma_X$. However, we are performing truncation here; note that, $g_{\tau, T}(u^\top Z_i) \neq (u^\top Z_i)^2$ only if either $|u^\top Z_i| > \tau/2$ or $|\gamma_0^\top Z_i| > T$. Therefore, we have the following upper bound: 
\begin{align}
\label{eq:RSC_moment_bound_1}
    & \frac1n \sum_i\bbE\left[(u^\top Z_i)^2 - g_{\tau, T}(Z_i)\right] \notag \\
    & \le \frac1n \sum_i\bbE\left[(u^\top Z_i)^2\mathds{1}_{|u\top Z_i| > \tau/2}\right] + \frac1n \sum_i\bbE\left[(u^\top Z_i)^2 \mathds{1}_{|\gamma_0^\top Z_i| > T}\right] \notag \\
    & \le \frac1n \sum_i\sqrt{\bbE\left[(u^\top Z_i)^4\right]\bbP\left(|u\top Z_i| > \tau/2\right)}+ \frac1n \sum_i\sqrt{\bbE\left[(u^\top Z_i)^4\right] \bbP\left(|\gamma_0^\top Z_i| > T\right)} \notag \\
    & \le \sqrt{\frac1n \sum_i\bbE\left[(u^\top Z_i)^4\right]\bbP\left(|u\top Z_i| > \tau/2\right)} + \sqrt{\frac1n \sum_i\bbE\left[(u^\top Z_i)^4\right] \bbP\left(|\gamma_0^\top Z_i| > T\right)}
\end{align}
To bound the right hand side of Equation \eqref{eq:RSC_moment_bound_1}, we need to bound the moments of $(u^\top Z_i)$ and $(\gamma_0^\top Z_i)$, which is proved in the following lemma: 
\begin{lemma}
    \label{lem:moment_bounds}
    Under the problem setup and assumption (C1) , there exists constants $\kappa_1$ and $\kappa_2$ such that: 
    $$
    \bbE[(v^\top Z_i)^2] \le \kappa_1  \ \ \& \ \ \bbE[(v^\top Z_i)^4] \le \kappa_2 \,,
    $$
    for any vector $v \in \reals^p$ with $\|v\|_2 \le 1$. 
\end{lemma}
\begin{proof}
    We first bound the fourth moment. Note that: 
    \begin{align*}
    & \bbE[(v^\top {\bZ}_i)^4] \\
    &= \bbE\left[\left(\frac{1}{\sqrt{(n-1)p}}\sum_{j = 1}^n a_{ij}({\bX}_j^\top v)\right)^4\right] \\
    & = \frac{1}{(n-1)^2p^2}\left\{\sum_j \bbE[A_{ij}^4(X_j^\top v)^4] + \sum_{j_1 \neq j_2} \bbE[A_{ij}^2(X_j^\top v)^2]\bbE[A_{ij}^2(X_j^\top v)^2]\right\} \\
    & =  \frac{1}{(n-1)^2p^2}\left\{np\bbE[(X^\top v)^4] + n(n-1)p^2 (\bbE[(X^\top v)^2])^2\right\} \\
    & = \frac{n}{n-1}\left\{\frac{1}{n-1}\bbE[(X^\top v)^4] + (\bbE[(X^\top v)^2])^2\right\} \\
    & \le \kappa_2 \,.
\end{align*}
for some constant $\kappa_2$ as both $\bbE[(X^\top v)^2]$ and $\bbE[(X^\top v)^4]$ are finite, since $X$ is a subgaussian random vector. The analysis for the second moment is similar: 
\begin{align*}
    \bbE[(v^\top {\bZ}_i)^2] & = \bbE\left[\left(\frac{1}{\sqrt{(n-1)p}}\sum_{j = 1}^n A_{ij}(X_j^\top v)\right)^2\right] \\
    & = \frac{1}{(n-1)p}\sum_{j = 1}^p \bbE\left[A^2_{ij}(X_j^\top v)^2\right] \\
    & = \frac{n}{n-1} \bbE[(X^\top v)^2] \le \kappa_1 \,.
\end{align*}
\end{proof}
Now, using this lemma and Chebychev's inequality, we conclude that for a large enough choice of $(T, \gamma)$, we have: 
$$
\bbE\left[\frac1n \sum_i g_{\tau, T}(u^\top Z_i)\right] \ge \frac{\kappa}{2} \,.
$$
{\bf Step 2: }
Now that we have proved that the expected value of the right-hand side of Equation \eqref{eq:RSC_bound_1} is lower bounded by some constant, 
We next define an empirical process, namely $\bZ(t)$, which is defined as: 
\begin{equation}
    \label{eq:def_emp_proc_RSC}
    Z(t) \triangleq Z(t, Z_1, \dots, Z_n) = \sup_{\|u\|_2 = 1, \|u\|_1 = t} \left|\frac1n \sum_i g_{\tau, T}(u^\top Z_i) - \bbE\left[\frac1n \sum_i g_{\tau, T}(u^\top Z_i)\right]\right|
\end{equation}
As the function $g_{\tau, T}$ is bounded by $\tau^2/4$, we start with Mcdiarmid's inequality/bounded difference inequality; for any $Z'_i \neq {\bZ}_i$, 
% We will apply bounded difference inequality (Theorem 6.2 of \cite{boucheron2003concentration}). Note that conditional of $\bX$, ${\bZ}_i$'s are independent random vectors. Furthermore, for any $1 \le i \le n$ and for any $Z'_i \neq {\bZ}_i$: 
\begin{align*}
    & Z(t, Z_1, \dots, Z_{i-1}, {\bZ}_i, \dots, Z_n) - Z(t, Z_1, \dots, Z_{i-1}, Z'_i, \dots, Z_n) \\
    & \le \frac{1}{n}\sup_{u \in \bbS_2(1) \cap \bbS_1(t)} \left|g_u(Z'_i) - \bbE[g_u(Z'_i)]\right| \le \frac{\tau^2}{2n} \hspace{.1in} [\because g_u(\cdot) \le \tau^2/4]. 
\end{align*}
Therefore, by bounded difference inequality: 
$$
\bbP\left(Z(t) \ge \bbE[Z(t) \mid \bX] + t \mid {\bX}\right) \le \exp{\left(-\frac{8nt^2}{\tau^4}\right)}
$$
As the right-hand side does not depend on the value of $\bX$, we can further conclude the following by taking expectations with respect to $\bX$ on both sides: 
\begin{equation}
\label{eq:Zt_conc_bound_1}
\bbP\left(Z(t) \ge \bbE[Z(t) \mid \bX] + t \right) \le \exp{\left(-\frac{8nt^2}{\tau^4}\right)} \,.
\end{equation}
Next, using symmetrization and Rademacher complexity bounds, we bound $\bbE[Z(t) \mid \bX]$. For notational simplicity let us define: 
\begin{align*}
V_n & = \max_{1 \le j \le p} \left|\frac{1}{\sqrt{n}}\sum_{k = 1}^n {\bX}_{kj}\right| \\
\Gamma_n & = \max_{1 \le j \le d} \frac{1}{n}\sum_{k = 1}^n {\bX}_{kj}^2 \,.
\end{align*}
Now, as we have already pointed out, conditional on $\bX$, ${\bZ}_i$'s are i.i.d. random vectors. Therefore, the symmetrization argument holds, and following the same line of argument as of \cite{negahban2009unified}, we can conclude an analog of their equation (78): 
\begin{align*}
\bbE[Z(t) \mid \bX] & \le 8K_3t \bbE_{\eps, Z}\left[\max_{1 \le j \le 2d}\left|\frac{1}{n}\sum_{i=1}^n \eps_i {\bZ}_{ij}\mathds{1}_{|{\bZ}_i^\top \gamma_0| \le T}\right| \mid \bX\right] \\
& = \frac{8K_3t}{\sqrt{n}} \bbE_{\bZ \mid \bX}\left[\bbE_{\eps\mid \bZ, \bX}\left[\max_{1 \le j \le 2d}\left|\frac{1}{\sqrt{n}}\sum_{i=1}^n \eps_i {\bZ}_{ij}\mathds{1}_{|{\bZ}_i^\top \gamma_0| \le T}\right|\right]\right]
\end{align*}
First, observe that $\{\eps_1, \dots, \eps_n\}$ are Rademacher random variables (which are also subgaussian with sub-gaussian constant being 1), and therefore, conditionally on $\bZ$, 
$$
\frac{1}{\sqrt{n}}\sum_{i=1}^n \eps_i {\bZ}_{ij}\mathds{1}_{|{\bZ}_i^\top \gamma_0| \le T} \text{ is subgaussian with norm } \sqrt{\frac{1}{n}\sum_{i=1}^n {\bZ}^2_{ij}\mathds{1}_{|{\bZ}_i^\top \gamma_0| \le T}} \le \sqrt{\frac1n \sum_{i=1}^n {\bZ}^2_{ij}} \,.
$$
Therefore, from standard probability tail bound calculation, we have: 
$$
\bbE_{\eps\mid \bZ, \bX}\left[\max_{1 \le j \le 2d}\left|\frac{1}{\sqrt{n}}\sum_{i=1}^n \eps_i {\bZ}_{ij}\mathds{1}_{|{\bZ}_i^\top \gamma_0| \le T}\right|\right] \le 
\sqrt{2\log{2d}}\max_{1 \le j  \le 2d} \sqrt{\frac1n \sum_{i=1}^n {\bZ}^2_{ij}} \,.
$$
Therefore, we have: 
\begin{equation}
\label{eq:Zt_bound_1}
\bbE[Z(t) \mid \bX] \le 8\sqrt{2}K_3t \sqrt{\frac{\log{2d}}{n}} \bbE\left[\max_{1 \le j  \le 2d} \sqrt{\frac1n \sum_{i=1}^n {\bZ}^2_{ij}} \mid \bX\right]
\end{equation}
%We next bound $\max_{1 \le j \le d}\frac{1}{n}\sum_i {\bZ}_{ij}^2$. 
We next analyze the first $d$ co-odinates of $\bZ_i$ (which is $\{X_{ij}\}_{1 \le j \le d}$) and the last $d$ coordinates of $\bZ_i$, which is $X^\top A^*_{i*}$ separately. For the first $d$ coordinates, conditional on $\bX$, we have: 
$$
\max_{1 \le j \le d} \sqrt{\frac{1}{n}\sum_{i = 1}^n \bZ_{ij}^2} =  \max_{1 \le j \le d} \sqrt{\frac{1}{n}\sum_{i = 1}^n \bX_{ij}^2} = \sqrt{\Gamma_n} \,.
$$
Now, for any $d+1 \le j \le 2d$, we have:  
% ${\bZ}_{ij} = (\sum_k A_{ik}{\bX}_{kj})/\sqrt{np}$, which is not centered conditional on $\bX$. Therefore, we first center it: 
$$
{Z}_{ij} = \frac{1}{\sqrt{np}}\sum_k A_{ik}{\bX}_{kj} = \frac{1}{\sqrt{np}}\sum_k (A_{ik} - p){\bX}_{kj} + \sqrt{\frac{p}{n}}\sum_k {\bX}_{kj} \triangleq \bar {\bZ}_{ij} + \sqrt{\frac{p}{n}}\sum_k {\bX}_{kj} \,.
$$
Using this we have: 
\begin{align}
\bbE\left[\max_{d+1 \le j  \le 2d} \sqrt{\frac1n \sum_{i=1}^n {\bZ}^2_{ij}} \mid \bX\right] & \le \bbE\left[\max_{d+1 \le j  \le 2d} \sqrt{\frac1n \sum_{i=1}^n \bar {\bZ}^2_{ij}} \mid \bX\right] + \sqrt{p} \max_{1 \le j  \le d} \left|\frac{1}{\sqrt{n}}\sum_k {\bX}_{kj}\right|  \notag \\
&= \bbE\left[\max_{d+1 \le j  \le 2d} \sqrt{\frac1n \sum_{i=1}^n \bar {\bZ}^2_{ij}} \mid \bX\right] + \sqrt{p}V_n \notag \\
\label{eq:Z_bound_1} & \le \sqrt{\bbE\left[\max_{d+1 \le j \le 2d} \frac1n \sum_{i=1}^n \bar {\bZ}^2_{ij} \mid \bX\right]}+ \sqrt{p}V_n \,.
\end{align}
We next establish an upper bound on the conditional expectation of the maximum of the mean of $\bar {\bZ}_{ij}^2$. We first claim that $\bar {\bZ}_{ij}$ is a SG($\sigma_j)$ random variable with the value of $\sigma_j$ defined in equation \eqref{eq:Z_tilde_sg_norm} below. To see this, first note that, from Theorem 2.1 of \cite{ostrovsky2014exact}, we know $(A_{ik} - p)$ is SG($\sqrt{2}Q(p)$). Therefore, we have: 
\begin{equation}
\label{eq:Z_tilde_sg_norm}
\bar {Z}_{ij} = \frac{1}{\sqrt{np}}\sum_k (A_{ik} - p){\bX}_{kj} \in \mathrm{SG}\left(\sqrt{\frac{2Q^2(p)}{p} \frac1n \sum_k {\bX}_{kj}^2}\right) \triangleq \mathrm{SG}(\sigma_j) \,.
\end{equation}
Let $\mu_j = \bbE[\bar {\bZ}^2_{ij} | \bX]$. Then, by equation (37) of \cite{honorio2014tight}, we know $\bar {\bZ}^2_{ij} - \mu_j$ is a sub-exponential random variable, in particular: 
$$
\bar {\bZ}^2_{ij} - \mu_j \in \mathrm{SE}\left(\sqrt{32}\sigma_j, 4\sigma_j^2\right) \,.
$$
Hence we have, by equation (2.18) of \cite{wainwright2019high} (we use the version for the two-sided bound here): 
\begin{equation}
\label{eq:Z_bound_concentration_1}
\bbP\left(\left|\frac1n \sum_{i = 1}^n\left(\bar {\bZ}_{ij}^2 - \mu_j\right)\right| \ge t\right) \le \exp{\left(-\frac{1}{8\sigma_j^2}\min\left\{\frac{nt^2}{8}, nt\right\}\right)} \,.
\end{equation}
Going back to \eqref{eq:Z_bound_1}, we have: 
\begin{align*}
    \bbE\left[\max_{1 \le j \le d} \frac1n \sum_{i=1}^n \bar {\bZ}^2_{ij} \mid \bX\right] & = \bbE\left[\max_{1 \le j \le d} \left\{\left(\frac1n \sum_{i=1}^n (\bar {\bZ}^2_{ij} - \mu_j)\right) + \mu_j\right\} \mid \bX\right] \\
    & \le \bbE\left[\max_{1 \le j \le d} \left|\frac1n \sum_{i=1}^n (\bar {\bZ}^2_{ij} - \mu_j)\right| \mid \bX\right] + \max_{1 \le j \le d} \mu_j
\end{align*}
Now, bound the first term using the concentration inequality \eqref{eq:Z_bound_concentration_1}. Towards that end, define $\sigma_* = \max_j \sigma_j$ and observe that $\sigma_* = \sqrt{2Q^2(p)/p}\sqrt{\Gamma_n}$. 
\begin{align*}
    & \bbE\left[\max_{1 \le j \le d} \left|\frac1n \sum_{i=1}^n (\bar {\bZ}^2_{ij} - \mu_j)\right| \mid \bX\right] \le 8 \max\{\sigma_* \sqrt{\log{d}}, \sigma_*^2 \log{d}\} \,.
    % & = \int_0^\infty \bbP\left(\max_{1 \le j \le d} \left|\frac1n \sum_{i=1}^n (\bar {\bZ}^2_{ij} - \mu_j)\right| > t\right) \ dt 
\end{align*}
Furthermore, observe that: 
\begin{align*}
    \mu_j = \bbE[\bar {\bZ}_{ij}^2 \mid \bX] & = \bbE\left[\left(\frac{1}{\sqrt{np}}\sum_k (A_{ik} - p)X_{kj}\right)^2 \mid \bX\right] = (1-p) \frac1n \sum_k X^2_{kj} \,,
\end{align*} 
which implies, $\max_{1 \le j \le d} \mu_j = (1-p) \Gamma_n$. Using these bounds in equation \eqref{eq:Z_bound_1}, we have: 
\begin{equation}
    \label{eq:Z_bound_2}
    \bbE\left[\max_{1 \le j  \le d} \sqrt{\frac1n \sum_{i=1}^n {\bZ}^2_{ij}} \mid \bX\right] \le \sqrt{\max\{\sigma_* \sqrt{\log{d}}, \sigma_*^2 \log{d}\} + (1-p) \Gamma_n} + \sqrt{p}V_n
\end{equation}
This, along, with equation \eqref{eq:Zt_bound_1},yields: 
\begin{align}
\label{eq:Zt_bound_2}
\bbE[Z(t) \mid \bX] & \le Ct\sqrt{\frac{\log{d}}{n}}\left(\sqrt{\max\{\sigma_* \sqrt{\log{d}}, \sigma_*^2 \log{d}\} + (1-p) \Gamma_n} + \sqrt{p}V_n\right) \notag \\
& \triangleq Ct\sqrt{\frac{\log{d}}{n}} g(\bX, p, d) \,.
\end{align}
Using this in the inequality \eqref{eq:Zt_conc_bound_1} yields: 
\begin{equation}
    \label{eq:Zt_conc_bound_2}
    \bbP\left(Z(t) \ge Ct\sqrt{\frac{\log{d}}{n}} g(\bX, p, d) + y \right) \le \exp{\left(-\frac{8ny^2}{\tau^4}\right)} \,.
\end{equation}
We next provide an upper bound for $g(\bX, p, d)$ term. Note that in the expression of $g(\bX, p, d)$, there are two key terms: $\Gamma_n, V_n$. Therefore, if we can obtain an upper bound on them individually, we can obtain an upper bound on $g(\bX, p, d)$. We start with $V_n$; for any fixed $j$, ${\bX}_{kj}$'s are i.i.d sub-gaussian random variable with constant $\sigma^2_X$. Therefore, we have: 
$$
\bbP\left(\left|\frac{1}{\sqrt{n}}\sum_{k = 1}^n {\bX}_{kj}\right| \ge t\right) \le 2e^{-\frac{t^2}{2\sigma^2_X}} 
$$
As a consequence, by union bound: 
$$
\bbP(V_n \ge t) = \bbP\left(\max_j \left|\frac{1}{\sqrt{n}}\sum_{k = 1}^n {\bX}_{kj}\right| \ge t\right) \le 2e^{\log{d} -\frac{t^2}{2\sigma^2_X}}
$$
Therefore, choosing $t = \sigma_X \sqrt{2c_1 \log{d}}$ (where $c_1 \ge 2$), we have: 
\begin{equation}
\label{eq:bound_Vn}
V_n \le  \sigma_X \sqrt{2c_1 \log{d}}) \ \ \text{with probability } \ge 1 - 2\exp{\left(-(c_1 -1)\log{d}\right)} \,.
% \bbP(V_n \ge \sigma_x \sqrt{2c_1 \log{d}}) \le 2e^{-(c_1 -1)\log{d}} \,.
\end{equation}
Call this event $\Omega_{n, {\bX}, 1}$. Our next target is $\Gamma_n$ which can be further upper bounded by: 
$$
\Gamma_n = \max_j \frac1n \sum_{k = 1}^n {\bX}_{kj}^2 \le \max_j \frac1n \sum_{k = 1}^n ({\bX}_{kj}^2 - \Sigma_{X, jj}) + \max_{j} \Sigma_{X, jj} \triangleq \bar \Gamma_n + \max_{j} \Sigma_{X, jj} \,.
$$
As we have assumed $\max_j \Sigma_{X, jj} \le C_1$ for some constant $C_1$, we need to bound $\bar \Gamma_n$. Here, we also use the fact that ${\bX}_{jk}^2 - \Sigma_{X, jj} \in \mathrm{SE}(\sqrt{32}\sigma_X, 4\sigma_X^2)$. Therefore, by equation (2.18) of \cite{wainwright2019high} we have: 
$$
\bbP\left(\left|\frac1n \sum_{k = 1}^n ({\bX}_{kj}^2 - \Sigma_{X, jj})\right| \ge t\right) \le 2\exp{\left(-\frac{n}{8\sigma^2_X}\min\left\{\frac{t^2}{8}, t\right\}\right)}
$$
Therefore, by union bound: 
$$
\bbP\left(\max_{1 \le j\le d}\left|\frac1n \sum_{k = 1}^n ({\bX}_{kj}^2 - \Sigma_{X, jj})\right| \ge t\right) \le 2\exp{\left(\log{d} -\frac{n}{8\sigma^2_X}\min\left\{\frac{t^2}{8}, t\right\}\right)}
$$
Choosing $t = \max_j \Sigma_{X, jj}$, we have: 
\begin{equation}
\label{eq:bound_Gamma_n}
\Gamma_n \le 2\max_j \Sigma_{X, jj} \le 2C_1 \ \ \text{with probability } \ge 1 -  2\exp{\left(\log{d} - c_2 n\right)} \,.
\end{equation}
Call this event $\Omega_{n, {\bX}, 2}$. 
Now, going back to the definition of $g(\bX, p, d)$ in equation \eqref{eq:Zt_bound_2}, we first note that, on the event $\Omega_{n,{\bX},1} \cap \Omega_{n, {\bX}, 2}$:  
$$
\sigma_* = \sqrt{\frac{2Q^2(p)}{p}\Gamma_n} \le 2\sqrt{\frac{C_1 Q^2(p)}{p}} \triangleq 2\sqrt{C_1 \Psi(p)}\,.
$$
It is immediate from the definition of $Q(p)$ that $\Psi(p) \sim 1/(-4p\log{p})$ for $p$ close to $0$. Therefore for all small $p$ and large $d$, $\sigma_*^2 \log{d} \ge 1$ and consequently $\max\{\sigma_* \sqrt{\log{d}}, \sigma_*^2 \log{d}\} = \sigma_*^2 \log{d}$. Hence, we have on the event $\Omega_{n,{\bX},1} \cap \Omega_{n, {\bX}, 2}$:
$$
g(\bX, p, d) \le C_2 \sqrt{\Psi(p) \log{d}} + 2C_1 + \sigma_X \sqrt{2c_1p \log{d}} \,.
$$
It is immediate that the dominating term is the first term, which implies: 
$$
g(\bX, p, d) \le 3C_2 \sqrt{\Psi(p) \log{d}} \,.
$$
We now use this bound in equation \eqref{eq:Zt_conc_bound_1}. Note that: 
\begin{align*}
    & \bbP\left(Z(t) \ge \bbE[Z(t) \mid \bX] + y \right) \\
    & \ge  \bbP\left(Z(t) \ge \bbE[Z(t) \mid \bX] + y, \Omega_{n,{\bX}, 1} \cap \Omega_{n, {\bX}, 2} \right)  \\
    & \ge \bbP\left(Z(t) \ge  3CC_2 t\log{d}\sqrt{\frac{\Psi(p)}{n}}  + y, \Omega_{n,{\bX}, 1} \cap \Omega_{n, {\bX}, 2} \right) \\
    & \ge \bbP\left(Z(t) \ge  3CC_2 t\log{d}\sqrt{\frac{\Psi(p)}{n}}  + y\right) + \bbP(\Omega_{n,{\bX}, 1} \cap \Omega_{n, {\bX}, 2} ) - 1
\end{align*}
Therefore, 
\begin{align}
    \label{eq:Zt_conc_bound_4}
     \bbP\left(Z(t) \ge  3CC_2 t\log{d}\sqrt{\frac{\Psi(p)}{n}}  + y\right)  & \le \exp{\left(-\frac{8ny^2}{\tau^4}\right)} + \bbP((\Omega_{n,{\bX}, 1} \cap \Omega_{n, {\bX}, 2})^c) \notag \\
     & \le \exp{\left(-\frac{8ny^2}{\tau^4}\right)} + 2\exp{\left(-(c_1 -1)\log{d}\right)} + 2\exp{\left(\log{d} - c_2 n\right)} \,.
\end{align}
Choosing $y = C_3 t \log{d}\sqrt{\Psi(p)/n}$, we have: 
\begin{align}
    \label{eq:Zt_conc_bound_3}
    & \bbP\left(Z(t) \ge  3CC_2 t\log{d}\sqrt{\frac{\Psi(p)}{n}}  + C_3t \log{d} \sqrt{\frac{\Psi(p)}{n}}\right) \notag \\
    & \le \exp{\left(-\frac{8C_3^2 t^2 \log^2{d}\Psi(p)}{\tau^4}\right)} + 2\exp{\left(-(c_1 -1)\log{d}\right)} + 2\exp{\left(\log{d} - c_2 n\right)} \,.
\end{align}

{\bf Step 4: }Our last modification, not modification per se, but an application of peeling argument. Infact we want an upper bound on the event $\cE$ defined as: 
$$
\cE = \left\{Z(t) \ge 3eCC_2 t\log{d}\sqrt{\frac{\Psi(p)}{n}}  + C_3et \log{d} \sqrt{\frac{\Psi(p)}{n}}\ \text{ for some }t \in [1, \sqrt{d}] \right\} \,.
$$
Note that $t$ denotes the $\ell_1$ norm of a a vector $u$ such that $\|u\|_2 = 1$. Therefore, $t \in [1, \sqrt{d}]$. Also recall that $Z(t)$ is the suprema of the empirical process over all vectors $u$ such that $\|u\|_2 = 1$ and $\|u\|_1 \le t$. In peeling, we write $\cE$ as the union of disjoint events. Define $\cE_j$ as: 
$$
\cE_j = \left\{Z(t) \ge 3eCC_2 t\log{d}\sqrt{\frac{\Psi(p)}{n}}  + C_3et \log{d} \sqrt{\frac{\Psi(p)}{n}}\ \text{ for some }t \in [\sqrt{d}/e^{j}, \sqrt{d}/e^{j-1}] \right\} \,.
$$
Therefore, 
$$
\cE \subseteq \cup_{j = 1}^{\lceil \frac12\log{d}\rceil} \cE_j \implies \bbP(\cE) \le \sum_{j = 1}^{\lceil \frac12\log{d}\rceil} \bbP(\cE_j) \,.
$$
Now observe that, for any $t \in [\sqrt{d}/e^{j}, \sqrt{d}/e^{j-1}]$, we have $Z(t) \le Z(\sqrt{d}/e^{j-1})$ and also 
\begin{align*}
3eCC_2 t\log{d}\sqrt{\frac{\Psi(p)}{n}}  + C_3et \log{d} \sqrt{\frac{\Psi(p)}{n}} & \ge 3eCC_2 \frac{\sqrt{d}}{e^{j}}\log{d}\sqrt{\frac{\Psi(p)}{n}}  + C_3e \frac{\sqrt{d}}{e^{j}} \log{d} \sqrt{\frac{\Psi(p)}{n}} \\
 & \ge 3CC_2 \frac{\sqrt{d}}{e^{j-1}}\log{d}\sqrt{\frac{\Psi(p)}{n}}  + C_3\frac{\sqrt{d}}{e^{j-1}} \log{d} \sqrt{\frac{\Psi(p)}{n}} \,.
\end{align*}
Therefore: 
\begin{align*}
    \bbP(\cE_j) & \le \bbP\left(Z\left(\frac{\sqrt{d}}{e^{j-1}}\right) \ge 3CC_2 \frac{\sqrt{d}}{e^{j-1}}\log{d}\sqrt{\frac{\Psi(p)}{n}}  + C_3\frac{\sqrt{d}}{e^{j-1}} \log{d} \sqrt{\frac{\Psi(p)}{n}}\right) \\
    & \le \exp{\left(-\frac{8C_3^2 d \log^2{d}\Psi(p)}{e^{2j - 2}\tau^4}\right)} + 2\exp{\left(-(c_1 -1)\log{d}\right)} + 2\exp{\left(\log{d} - c_2 n\right)} \\
    & \le  \exp{\left(-c_4 \log^2{d}\Psi(p)\right)} + 2\exp{\left(-(c_1 -1)\log{d}\right)} + 2\exp{\left(\log{d} - c_2 n\right)} 
\end{align*}
Hence: 
$$
\bbP(\cE) \le  \exp{\left(\frac12 \log{d} + 1-c_4 \log^2{d}\Psi(p)\right)} + 2\exp{\left(1-(c_1 -3/2)\log{d}\right)} + 2\exp{\left(\frac{3}{2}\log{d} + 1 - c_2 n\right)}  \,.
$$
On the event $\cE^c$ (which is a high probability event):  
$$
Z(t) \le 3eCC_2 t\log{d}\sqrt{\frac{\Psi(p)}{n}}  + C_3et \log{d} \sqrt{\frac{\Psi(p)}{n}}\ \text{ for all }t \in [1, \sqrt{d}] \,.
$$
Now let us conclude with the entire roadmap of the proof. First, following the same line of argument as of \cite{negahban2009unified} we show that 
\begin{align*}
\delta L_n(u) & \ge L_\psi(T) \frac1n \sum_i \phi_\tau\left((u^\top {\bZ}_i)^2\mathds{1}_{|{\bZ}_i^\top\gamma_0| \le T}\right) \\
& = L_\psi(T) \|u\|^2 \frac1n \sum_i \phi_\tau\left(\left(u^\top {\bZ}_i/\|u\|\right)^2\mathds{1}_{|{\bZ}_i^\top\gamma_0| \le T}\right) \\
& = L_\psi(T) \|u\|^2 \bbP_n(g_{u/\|u\|}(Z)) \\
& = L_\psi(T) \|u\|_2^2 \left\{P(g_{u/\|u\|}(Z)) + (\bbP_n - P)g_{u/\|u\|}(Z)\right\} 
\end{align*}
We have proved in Modification 1 that $P(g_{u/\|u\|}(Z)) \ge \kappa_l$. Therefore, 
$$
\delta L_n(u) \ge  L_\psi(T) \|u\|_2^2 \left\{\kappa_l + (\bbP_n - P)g_{u/\|u\|}(Z)\right\} 
$$
Now for any $u$, 
\begin{align*}
(\bbP_n - P)g_{u/\|u\|}(Z) & \le Z\left(\left\|\frac{u}{\|u\|_2}\right\|_1\right)  \\
& \le \left(3eCC_2 \log{d}\sqrt{\frac{\Psi(p)}{n}}  + C_3e \log{d} \sqrt{\frac{\Psi(p)}{n}}\right)\frac{\|u\|_1}{\|u\|_2} \,.
\end{align*}
Hence, we conclude that: 
$$
\delta L_n(u) \ge L_\psi(T) \|u\|_2^2 \left\{\kappa_l -  \left(C_4 \log{d} \sqrt{\frac{\Psi(p)}{n}}\right)\frac{\|u\|_1}{\|u\|_2}\right\} \, ,
$$
where $L_\psi(T)$ is a constant. Using the fact $\delta L_n(u) = u^\top \bZ^\top \bZ u /n$, we conclude
\begin{equation}
	\label{eq:rsc}
	\frac{u^\top \bZ^\top \bZ u}{n} \geqslant \kappa \|u\|_2^2 - C_5 \log{d} \sqrt{\frac{\Psi(p)}{n}} \|u\|_1 \|u\|_2.
\end{equation}

We next derive the upper bound on $(u^\top \hat \bZ^\top \hat \bZ u)/n$. Recall that the difference between $\hat \bZ$ and $\bZ$ is that in the former, $A_{ij}$ is scaled by $\sqrt{(n-1)\hat p}$, whereas $A_{ij}$ is scaled by $\sqrt{(n-1)p}$ in the later. Expanding the quadratic form yields:
\allowdisplaybreaks
\begin{align*}
    \frac{u^\top \hat \bZ^\top \hat \bZ u}{n} = \frac{1}{n}\sum_i (\hat Z_i^\top u)^2 & = \frac{1}{n}\sum_i (u_1^\top X_i + u_2^\top \bX^\top A_{i,*}/\sqrt{(n-1)\hat p})^2 \\
    & = \frac1n \sum_i \left(u_1^\top X_i + \sqrt{\frac{p}{\hat p}}\frac{1}{\sqrt{(n-1)p}}u_2^\top \bX^\top A_{i, *}\right)^2 \\
     & = \frac1n \sum_i \left(\sqrt{\frac{p}{\hat p}} u^\top Z_i  + \left(1- \sqrt{\frac{p}{\hat p}}\right)u_1^\top X_i\right)^2 \\
     & \ge \frac{p}{2\hat p}\frac{1}{n}\sum_i (u^\top Z_i)^2 - \left(1- \sqrt{\frac{p}{\hat p}}\right)^2\frac{1}{n} \sum_i (u_1^\top X_i)^2 \,.
\end{align*}
Here the last inequality follows from the fact that $(a+b)^2 \ge (a^2/2) - b^2$. Now, in Step 2.1 of the proof of Theorem \ref{thm: general model upper bound without transfer learning}, we show that with probability going to one, $|\hat p - p| \le p/\sqrt{n}$. This implies $(1/2) p \le \hat p \le 2p$ with high probability, which further implies $p/\hat p \ge 1/2$. 
On the other hand, we have: 
\begin{align*}
    \left|1 - \sqrt{\frac{\hat p}{p}}\right| = \frac{|\sqrt{\hat p} - \sqrt{p}|}{\sqrt{\hat p}} = \frac{|\hat p - p|}{\sqrt{\hat p}(\sqrt{\hat p} + \sqrt{p})} \le \frac{1}{\sqrt{n}}\frac{p}{\sqrt{\hat p p}} \le \sqrt{\frac{2}{n}} \,.
\end{align*}
Therefore, we conclude: 
\begin{align*}
\frac{u^\top \hat \bZ^\top \hat \bZ u}{n} & \ge \frac{u^\top\bZ^\top \bZ u}{n} - 3\sqrt{\frac{2}{n}} \frac{u_1^\top \bX^\top \bX u_1}{n} \\
& = \frac{u^\top\bZ^\top \bZ u}{n} - 3\sqrt{\frac{2}{n}} \left(u_1^\top \Sigma_X u_1 + u_1^\top \left(\frac{\bX^\top \bX }{n} - \Sigma_X\right)u_1\right) \\
& \ge \frac{u^\top\bZ^\top \bZ u}{n} - 3\sqrt{\frac{2}{n}} \left(u_1^\top \Sigma_X u_1 + \left\|\frac{\bX^\top \bX }{n} - \Sigma_X\right\|_{\infty, \infty} \|u_1\|_1^2\right) 
\end{align*}
Now, to complete proof, we use the following facts: i) $\lambda_{\min}(\Sigma_X) \ge \kappa$ (Assumption (C1) and ii) by an application of Hoeffding's inequality along with a union bound (Lemma 1 of \cite{ravikumar2011high}): 
\begin{align*}
\bbP\left(\left\|\frac{\bX^\top \bX }{n} - \Sigma_X\right\|_{\infty, \infty} \ge t\right) & \le \sum_{j, k} \bbP\left(\left|\frac1n \sum_i X_{ij}X_{ik} - \Sigma_{X, jk
}\right| \ge t\right) \\
& \le 4\exp{\left(2\log{d} - \frac{Cnt^2}{\max{\Sigma^2_{X, ii}}}\right)} 
\end{align*}
for all $t \le C_1 \max_i \Sigma_{X, ii}$. As we have assumed $\max_i \Sigma_{X, ii}$ is uniformly upper bounded (Assumption (C1)), choosing $t = K \sqrt{\log{d}/n}$ (for a suitable constant $K$ so that $CK^2 > 2\max_i \Sigma_{X, ii}^2$, we have with probability $1 - d^{-\alpha}$ 
$$
\left\|\frac{\bX^\top \bX }{n} - \Sigma_X\right\|_{\infty, \infty} \le K\sqrt{\frac{\log{d}}{n}} \,.
$$
Therefore, we conclude: 
$$
\frac{u^\top \hat \bZ^\top \hat \bZ u}{n}  \ge \frac{u^\top\bZ^\top \bZ u}{n} - 3\sqrt{\frac{2}{n}}\kappa - 3K \frac{\sqrt{2\log{d}}}{n}\|u_1\|_1^2 \,.
$$
\end{proof}

\subsection{Proof of lemma \ref{lemma:ubforA}}
\begin{proof}
{\bf Part 1: Upper bound for $\|A^*\|_F^2$.} To obtain a bound for $\| A^* \|_F$, we use Chebychev inequality. 
First, observe that: 
$$
\bbE[\|A^* \|_F^2]  = \frac{1}{(n-1)p} \bbE[  \|A\|_F^2] =\frac{1}{(n-1)p} \sum_{i \neq j}\bbE[A_{ij}^2] = n \,.
$$
For the variance of the Frobenious norm: 
\begin{align*}
\var(\| A^* \|_F^2) = \frac{1}{((n-1)p)^2} \var (\| A \|_F^2) & = \frac{1}{((n-1)p)^2} \var (\sum_{i \neq j} A^2_{ij}) \\
& = \frac{n(n-1)}{(n-1)^2 p^2} \var(A_{11}^2) \\
& \le \frac{n}{n-1} \frac{1}{p^2} \bbE[A_{11}^4] \le \frac{2}{p} \,.
\end{align*}
Therefore, by Chebychev's inequality, we have: 
$$
\bbP(|\| A^* \|_F^2 - \bbE[\| A^* \|_F^2]| \geqslant n) \leqslant \var(\| A^* \|_F^2)/n^2 \leqslant \frac{2}{n^2p} \,.
$$
Therefore, we have $\| A^* \|_F^2 \leqslant \bbE\| A^* \|_F^2 + n \leqslant 2n$ with probability $1 - 2(n^2 p)^{-1}$, thus formula \eqref{eq:A_frob_bound} could be obtained.
\\
\noindent
{\bf Part 2: Upper bound for $\| A^* \|_2$.} 
To establish a bound for $\| A^* \|_2$, first we have $
\| A^* \|_2 \leqslant \|A^* - \bbE A^*\|_2 + \|\bbE A^* \|_2$.
A bound on $\|\bbE A^*\|_2$ directly follows from the definition:  
$$
\|\bbE A^*\|_2 = \{(n-1)p\}^{-1/2} \|\bbE A\|_2 = \{(n-1)p\}^{-1/2} \|p(\mathbf{1_n}\mathbf{1_n}^\top - I_n)\|_2 \leqslant \sqrt{np} \,.
$$
% {\color{magenta}$
% \|\bbE A^*\|_2 = \{(n-1)p\}^{-1/2} \|\bbE A\|_2 = \{(n-1)p\}^{-1/2} \|p(\mathbf{1_n}\mathbf{1_n}^\top - I_n)\|_2 \leqslant \sqrt{np}$.}
Next we bound $ \|A^* - \bbE A^*\|_2 = \|A - \bbE A\|_2/\sqrt{(n-1)p}$. Using Corollary 3.12 and Remark 3.13 of \cite{bandeira2016sharp} (with $\eps = 1/2$), we have 
$$
\bbP(\|A - \bbE A\|_2 \geqslant 3\sqrt{2}\tilde \sigma + t ) \leqslant \exp(\log{n} - t^2 /c\sigma_*^2) \,,
$$
where $\tilde \sigma = \max_i \sqrt{\sum_j \var A_{ij}} = \sqrt{(n-1)p(1-p)} \leqslant \sqrt{np}$ and $\sigma_* = \max_{i, j} |A_{ij}| \le 1$. Therefore, we obtain: 
$$
\bbP (\|A - \bbE A\|_2 \geqslant 3\sqrt{2}\sqrt{np} + t) \leqslant \exp(\log n -t^2/c) \,.
$$
Taking $t = \sqrt{np}$, we have 
$$\bbP(\|A - \bbE A\|_2 \geqslant (1 + 3\sqrt{2})\sqrt{np}) \leqslant \exp(\log n -np/c)
$$
which implies
$$
\bbP(\|A^* - \bbE A^*\|_2 \geqslant (1 + 3\sqrt{2})) \leqslant \exp(\log{n} - np/c) \,.
$$
Combining these bounds, we have $\|A^*\|_2 \leqslant \sqrt{np} + (1 + 3\sqrt{2}) \leqslant 2 \sqrt{np}$ with probability $1 - \exp(\log{n} - np/c)$. With $np \gg \log n$ (Assumption (C3) ), we have $1 - \exp(\log{n} - np/c) \to 1$.
\\
\noindent
{\bf Part 3: Upper bound for $\| {A^*}^\top A^* \|_2$.} We have $\| {A^*}^\top A^* \|_2 = \|A^*\|_2^2\leqslant 4 np$ with probability $1 - \exp(\log{n} - np/c)$.
\\
\noindent
{\bf Part 4: Upper bound for $\|{A^*}^\top A^*\|_F^2$.} 
% To apply Chebychev's inequality to obtain an upper bound on \(\|{A^*}^\top A^*\|_F^2\), 
We divide the entire proof into three steps: in the first step, we calculate the expected value of \(\|{A^*}^\top A^*\|_F^2\); in the second step, we provide an upper bound on its variance; in the third step, we summarize the results from the first two steps and use Chebychev's inequality to complete the proof.
\\
\indent
{\bf Step 1: Computing Expectation.} 
% We first calculate the expectation for $\|{A^*}^\top A^*\|_F^2$. 
For any $1 \leqslant i\leqslant n$, we have: 
\allowdisplaybreaks
\begin{align*}
    \bbE {({A^*}^\top A^*)_{ij}}^2 & = \{(n-1) p\}^{-2} \bbE \{\left(\sum_{k = 1}^n A_{ki}A_{kj}\right)^2\} \\
    & = \{(n-1) p\}^{-2} [ \sum_{k = 1}^n \bbE\left\lbrace (A_{ki}A_{kj})^2\right\rbrace  + \sum_{k \neq l} \bbE\left\lbrace (A_{ki}A_{kj})(A_{li}A_{lj})\right\rbrace ] \\
    & \le \{(n-1) p\}^{-2} ( np^2 + n^2p^4)  \lesssim n^{-1} + p^2 \\
    \bbE\{({A^*}^\top A^*)_{ii}^2\} & = \{(n-1)p\}^{-2} \bbE \{(\sum_{k = 1}^n A^2_{ki})^2\} \\
    & = \{(n-1)p\}^{-2} \{\sum_k \bbE(A^4_{ki}) + \sum_{k \neq l}\bbE (A_{ki}^2 A_{li}^2)\}  \\
    & \le \{(n-1)p\}^{-2} (np + n^2 p^2 ) \lesssim (np)^{-1} + 1 \,.
\end{align*}
% $\bbE {({A^*}^\top A^*)_{ij}}^2 = \{(n-1) p\}^{-2} \bbE \{\left(\sum_{k = 1}^n A_{ki}A_{kj}\right)^2\} = \{(n-1) p\}^{-2} [ \sum_{k = 1}^n \bbE\left\lbrace (A_{ki}A_{kj})^2\right\rbrace  + \sum_{k \neq l} \bbE\left\lbrace (A_{ki}A_{kj})(A_{li}A_{lj})\right\rbrace ] \leqslant \{(n-1) p\}^{-2} ( np^2 + n^2p^4) \lesssim n^{-1} + p^2$. 
% For $1 \leqslant i = j \leqslant n$, we have $\bbE\{({A^*}^\top A^*)_{ii}^2\} = \{(n-1)p\}^{-2} \bbE \{(\sum_{k = 1}^n A^2_{ki})^2\} = \{(n-1)p\}^{-2} \{\sum_k \bbE(A^4_{ki}) + \sum_{k \neq l}\bbE (A_{ki}^2 A_{li}^2)\}  \leqslant \{(n-1)p\}^{-2} (np + n^2 p^2 ) \lesssim (np)^{-1} + 1$. 
Therefore, we have 
$$
\bbE\{\|{A^*}^\top A^*\|_F^2\} = \sum_{i=1}^n \bbE\{({A^*}^\top A^*)_{ii}^2\} + \sum_{i \neq j}\bbE\{({A^*}^\top A^*)_{ij}^2\} \lesssim n\{(np)^{-1} + 1\} + n^2 (n^{-1} + p^2 ) \lesssim n + n^2p^2 \,,
$$
where the last inequality follows from $p \geqslant n^{-1}$.

{\bf Step 2: Computing Variance: }
For the variance part, we have: 
%Next, we establish a bound on the variance,
\begin{align*}
    \var(\|{A^*}^\top A^*\|_F^2) & = \{(n-1)p\}^{-4} \var(\|A^\top A\|_F^2) \\
    & \lesssim (np)^{-4} \var\{\sum_{i, j} (A^\top A)^2_{i,j}\} \\
    & = (np)^{-4} [\sum_{i, j}\var\{(A^\top A)^2_{i,j}\}+ \sum_{(i, j) \neq (k, l)} \cov\{(A^\top A)^2_{i,j}, (A^\top A)^2_{k,l}\}] \\
    & \triangleq (np)^{-4} (T_1 + T_2) \,.
\end{align*}
% $\var(\|{A^*}^\top A^*\|_F^2) = \{(n-1)p\}^{-4} \var(\|A^\top A\|_F^2) \lesssim (np)^{-4} \var\{\sum_{i, j} (A^\top A)^2_{i,j}\} = (np)^{-4} [\sum_{i, j}\var\{(A^\top A)^2_{i,j}\} \\ + \sum_{(i, j) \neq (k, l)} \cov\{(A^\top A)^2_{i,j}, (A^\top A)^2_{k,l}\}] \triangleq (np)^{-4} (T_1 + T_2)$. 
We next bound $T_1$ and $T_2$ separately. For that, we use some basic bounds on the moments of a binomial random variable: if $X \sim \operatorname{Bernoulli}(n, p)$, then $\bbE X^k \leqslant C n^kp^k$ for all $k \in \{1, 2, 3, 4\}$, for some universal constant $C$ as long as $np \rightarrow \infty$ (e.g., see \cite{rohe2011spectral, lei2015consistency}). 
% \DM{Any references for this bound?}{\color{magenta} It is a common assumption that the expected degree tends to infinity as $n \to \infty$, as in \cite{rohe2011spectral} and \cite{lei2015consistency}.}
Observe that $(A^\top A)_{ii} \sim \operatorname{Bernoulli}(n-1, p)$ and $(A^\top A)_{ij} \sim \operatorname{Bernoulli}(n-2, p^2)$ for $i \neq j$. 
For $T_1$, we have: 
\begin{align*}
\sum_{i, j}\var\{(A^\top A)^2_{ij}\} & = \sum_{i = 1}^n \var\{(A^\top A)^2_{ii}\} + \sum_{i \neq j }\var\{(A^\top A)^2_{ij}\}  \\
& \leqslant \sum_{i = 1}^n \bbE\{(A^\top A)^4_{ii}\} + \sum_{i \neq j }\bbE\{(A^\top A)^4_{ij}\} \lesssim n^5 p^4 + n^6p^8 \,.
\end{align*}
% {\bf Step 2.1: Bound $T_1$.} For $T_1$, we have $\sum_{i, j}\var\{(A^\top A)^2_{ij}\} = \sum_{i = 1}^n \var\{(A^\top A)^2_{ii}\} + \sum_{i \neq j }\var\{(A^\top A)^2_{ij}\}  \leqslant \sum_{i = 1}^n \bbE\{(A^\top A)^4_{ii}\} + \sum_{i \neq j }\bbE\{(A^\top A)^4_{ij}\} \lesssim n^5 p^4 + n^6p^8$.

For $T_2$, note that if 
% {\bf Step 2.2: Bound $T_2$.} Next, we bound $T_2$, i.e., the covariance term. Note that if
$(i, j, k, l)$ all are different, then covariance is $0$ as the terms are independent. Therefore, we only consider the cases when there are three or two distinct indices. We first deal with the terms of two distinct indices, i.e., $\cov\{(A^\top A)^2_{ii}, (A^\top A)^2_{ij}\}$ where $ i \neq j$. There are almost $n^2$ terms are of this form. For each of these type of terms: 
\begin{align*}
    \cov\{(A^\top A)^2_{ii}, (A^\top A)^2_{ij}\} & \leqslant \bbE \{(A^\top A)^2_{ii} (A^\top A)^2_{ij}\} \\
    & =  \bbE \{(\sum_{k, k' = 1}^n A^2_{ki}A_{k'i}A_{k'j})^2\} \\
    & = \bbE\{(\sum_k A_{ki}A_{kj} + \sum_{k \neq k'} A_{ki} A_{k'i}A_{k'j})^2\} \\
    & \le  2 [  \{\bbE(\sum_k A_{ki}A_{kj})^2\} + \bbE \{(\sum_{k \neq k'} A_{ki} A_{k'i}A_{k'j})^2\} ] \lesssim n^2 p^4 + n^4p^6 \,.
\end{align*}
% $\cov\{(A^\top A)^2_{ii}, (A^\top A)^2_{ij}\} \leqslant \bbE \{(A^\top A)^2_{ii} (A^\top A)^2_{ij}\} = \bbE \{(\sum_{k, k' = 1}^n A^2_{ki}A_{k'i}A_{k'j})^2\} = \bbE\{(\sum_k A_{ki}A_{kj} + \sum_{k \neq k'} A_{ki} A_{k'i}A_{k'j})^2\}  \leqslant 2 [  \{\bbE(\sum_k A_{ki}A_{kj})^2\} + \bbE \{(\sum_{k \neq k'} A_{ki} A_{k'i}A_{k'j})^2\} ] \lesssim n^2 p^4 + n^4p^6$. 
Therefore, we have 
$$
\sum_{i \neq j}  \cov\{(A^\top A)^2_{ii}, (A^\top A)^2_{ij}\} \leqslant n^4 p^4 + n^6 p^6 \,. 
$$
Next, we bound the covariance terms of the form $\cov((A^\top A)^2_{ij}, (A^\top A)^2_{jk})$, i.e. two terms share an index with $i \neq j \neq k$. There are almost $n^3$ such terms. For each term, we have 
\begin{align*}
    \cov\{(A^\top A)^2_{ij}, (A^\top A)^2_{jk}\} & \leqslant \bbE \{(\sum_{l, l'} A_{li}A_{lj}A_{l'i}A_{l'k})^2\} \\
    & = \bbE \{(\sum_l A^2_{li}A_{lj}A_{lk} + \sum_{l \neq l'} A_{li}A_{lj}A_{l'i}A_{l'k})^2\} \\
    & \le 2 [\bbE \{(\sum_l A_{li}A_{lj}A_{lk})^2\} + \bbE \{(\sum_{l \neq l'} A_{li}A_{lj}A_{l'i}A_{l'k})^2\}]  \lesssim n^2p^6 + n^4p^8 \,.
\end{align*}
% $\cov\{(A^\top A)^2_{ij}, (A^\top A)^2_{jk}\} \leqslant \bbE \{(\sum_{l, l'} A_{li}A_{lj}A_{l'i}A_{l'k})^2\}  = \bbE \{(\sum_l A^2_{li}A_{lj}A_{lk} + \sum_{l \neq l'} A_{li}A_{lj}A_{l'i}A_{l'k})^2\} \leqslant 2 [\bbE \{(\sum_l A_{li}A_{lj}A_{lk})^2\} + \bbE \{(\sum_{l \neq l'} A_{li}A_{lj}A_{l'i}A_{l'k})^2\}]  \lesssim n^2p^6 + n^4p^8$.
As there are almost $n^3$ terms in this form, we have 
$$
\sum_{i \neq j \neq k} \cov\{(A^\top A)^2_{ij}, (A^\top A)^2_{jk}\} \lesssim n^5p^6 + n^7p^8 \,.
$$ 
This implies $T_2 \lesssim n^4p^4 + n^6p^6 + n^5p^6 + n^7p^8$. Combining the bounds on $T_1$ and $T_2$, we conclude 
$$
\var(\|A^\top A\|_F^2) \lesssim n^5 p^4 + n^6p^8 + n^4p^4 + n^6p^6 + n^5p^6 + n^7p^8 \lesssim n^5 p^4 + n^6p^6 + n^7p^8 \,,
$$
where the last inequality follows from the fact that $n^5p^4 \ge n^4 p^4$, $n^6p^6 \ge n^6p^8$ and $n^6p^6 \ge n^4p^4$ as $np \rightarrow \infty$. As a consequence, we have 
$$\var(\|{A^*}^\top A^*\|_F^2) \lesssim (np)^{-4} \var(\|A^\top A\|_F^2) \lesssim n + n^2p^2 + n^3p^4 \,.
$$

{\bf Step 3: Chebychev's inequality: }The last step involves an application of Chebychev's inequality: 
\begin{align*}
    \bbP(\|{A^*}^\top A^*\|_F^2 - \bbE\{\|{A^*}^\top A^*\|_F^2\} \geqslant n + n^2p^2) & \leqslant \var(\|{A^*}^\top A^*\|_F^2)/(n + n^2p^2)^2 \\
    & \lesssim n^{-1} \,.
\end{align*}
% $\bbP (\|{A^*}^\top A^*\|_F^2 - \bbE\|{A^*}^\top A^*\|_F^2 \geqslant t) \leqslant \var(\|{A^*}^\top A^*\|_F^2)/t^2$.
% Taking $t = n + n^2p^2$, we have $\bbP(\|{A^*}^\top A^*\|_F^2 - \bbE\{\|{A^*}^\top A^*\|_F^2\} \geqslant n + n^2p^2) \leqslant \var(\|{A^*}^\top A^*\|_F^2)/(n + n^2p^2)^2 \lesssim 1/n$.
Therefore, we have $\|{A^*}^\top A^*\|_F^2 \leqslant \bbE (\| {A^*}^\top A^* \|_F^2) + n + n^2p^2 \lesssim n + n^2p^2$ with probability $1 - n^{-1}$.
\end{proof}

\subsection{Proof of lemma \ref{lemma: expectationZTZ}}
\begin{proof}
Note that the matrix $\bZ$ can be written as: 
$$
\bZ = \begin{bmatrix}
    A^* & I
\end{bmatrix} \begin{pmatrix}
    \bX & \mathbf{0}_{n \times p} \\
    \mathbf{0}_{n \times p} & \bX 
\end{pmatrix} \,.
$$
Therefore we have: 
$$
\bZ^\top \bZ = \begin{pmatrix}
    \bX^\top {A^*}^\top  {A^*} \bX & \bX^\top {A^*}^\top \bX \\
    \bX^\top A^* \bX & \bX^\top \bX \,
\end{pmatrix}
$$
Recall that $A_{ii} = 0$ and $ A^*_{ij} \sim \{(n-1)p\}^{-1/2}\ber(p)$. First we show that $\bbE[\bX^\top A^* \bX] = 0$. Towards that end, as $A_{ii} = 0$, 
\begin{align*}
    \bbE[\bX^\top A \bX] = \bbE\left[\bX^\top \left(\sum_{i \neq j}A_{ij} e_ie_j^\top\right)\bX \right] & = \bbE\left[\sum_{i \neq j} A_{ij} (\bX^\top e_i)(e_j^\top \bX)\right] \\
    & = \bbE\left[\sum_{i \neq j} A_{ij} \bx_i \bx_j^\top \right]
\end{align*}
Now as rows of $\bX$ are independent and have mean 0 and $A \indep \bX$, we have: 
$$
\bbE[A_{ij}\bx_i\bx_j^\top] = \bbE[A_{ij}]\bbE[\bx_i] \bbE[\bx_j^\top] = 0 \,.
$$
This establishes $\bbE[\bX^\top A \bX] = 0$. For the bottom right term of $\bZ^\top \bZ$, we have: 
\begin{align*}
    \frac{1}{(n-1)p}\bbE\left[\bX^\top A^\top A \bX\right] & = \frac{1}{(n-1)p}\sum_{i} \bbE[(A^\top A)_{ii}\bx_i \bx_i^\top] + \frac{1}{(n-1)p}\sum_{i \neq j} \bbE[(A^\top A)_{ij}\bx_i \bx_j^\top] \\
    & = \frac{1}{(n-1)p}\sum_{i} \bbE[(A^\top A)_{ii}\bx_i \bx_i^\top] \hspace{0.1in} [\because \bx_i \indep \bx_j \ \text{ for } i \neq j \text{ and have mean }0] \\
    & = \Sigma_X \left(\frac{1}{(n-1)p}\sum_{i} \bbE[(A^\top A)_{ii}]\right) \\
    & = \Sigma_X \left(\frac{1}{(n-1)p}\sum_{i} \bbE\left[\sum_{j = 1}^n A_{ji}^2\right]\right) = \Sigma_X n \,.
\end{align*}
The above calculation implies: 
$$
\frac{1}{n}\bbE\left( \bZ^\top \bZ\right)  = \begin{pmatrix}
    \Sigma_X & \b0 \\
    \b0 & \Sigma_X 
\end{pmatrix} = \begin{pmatrix}
        1 &  0 \\
        0 & 1
    \end{pmatrix} \otimes \Sigma_X \triangleq \Sigma_Z \,.
$$
\end{proof}

%\section*{Appendix D: AUC for transferable source data detection}
% \section{AUC for transferable source data detection}
% \begin{figure}[H]
%     \centering
%     \includegraphics[scale=0.2]{fig_new/fig_final/auc.pdf}
%     \caption{ER}
%     \label{fig:detect1}
% \end{figure}

% \begin{figure}[H]
%     \centering
%     \includegraphics[scale=0.2]{fig_new/fig_final/auc_sbm.pdf}
%     \caption{SBM}
%     \label{fig:detect2}
% \end{figure}

\section{Trans-NCR Algorithm}

\begin{algorithm}
\caption{Trans-NCR Algorithm}
\begin{algorithmic}
\STATE \textbf{Input}: Target data ($\mathbf{y}_0,\mathbf{X}_0,A_0$), source data $\{\mathbf{y}_k,\mathbf{X}_k,A_k\}_{k=1}^K$.
\STATE \underline{\textbf{Step 1.}} Let \(\mathcal{I}\) be a random subset of \(\{1, \ldots, n_{0}\}\) such that \(|\mathcal{I}| \approx c_0 n_{0}\) with some constant $0 < c_0 < 1$. Let \(\mathcal{I}^{c}=\{1, \ldots, n_{0}\} \backslash \mathcal{I}\).
\STATE \underline{\textbf{Step 2.}} Construct \(L+1\) candidate sets \(\mathcal{A},\left\{\widehat{G}_{0}, \widehat{G}_{1}, \ldots, \widehat{G}_{L}\right\}\) such that \(\widehat{G}_{0}=\emptyset\) and \(\widehat{G}_{1}, \ldots, \widehat{G}_{L}\) are based on (\ref{eq:def_Gl}) using \(\left(\mathbf{Z}_{0, \mathcal{I}}, \mathbf{y}_{0, \mathcal{I}}\right)\) and \(\left\{\mathbf{Z}_{k}, \mathbf{y}_{k}\right\}_{k=1}^{K}\).
\STATE \underline{Step 2.1.} For \(1 \leq k \leq K\), compute the marginal statistics $\hat{R}_k=\|\widehat\Delta_k\|_2^2$. For each \(k \in\{1, \ldots, K\}\), let \(\widehat{T}_{k}\) be obtained by SURE screening such that
\[
\widehat{T}_{k}=\left\{1 \leq j \leq 2d:\left|\widehat\Delta_{kj}\right| \text{ is among the first } t_{*} \text{ largest of all }\right\}
\]
%\underline{Step 2.1.} For \(1 \leq k \leq K\), compute the marginal statistics $|\hat \delta_{kj}|$. For each \(k \in\{1, \ldots, K\}\), let \(\widehat{T}_{k}\) be obtained by SURE screening such that
%\[
%\widehat{T}_{k}=\left\{1 \leq j \leq d:\left|\widehat{\delta}_{kj}\right| \text{ is among the first } t_{*} \text{ largest of all }\right\}
%\]
for a fixed \(t_{*}=n_{*}^{\alpha}, 0 \leq \alpha<1\).
\STATE \underline{Step 2.2.} Define the estimated sparse index for the \(k\)-th auxiliary sample as 
$\widehat{R}_k=\left\|\widehat{\Delta}_{k\widehat{T}_{k}}\right\|_{2}^{2}$.
% \STATE \underline{Step 2.2.} Define the estimated sparse index for the \(k\)-th auxiliary sample as $\color{red}\widehat{R}^{(k)}=\left\|\widehat{\delta}_{\widehat{T}_{k}}\right\|_{2}^{2}$.
\STATE \underline{Step 2.3.} Compute \(\widehat{G}_{l}\)  for \(l=1, \ldots, L\).
\STATE \underline{\textbf{Step 3.}} For each \(0 \leq l \leq L\), run the Oracle Trans-Lasso algorithm with primary sample \(\left(\mathbf{Z}_{0, \mathcal{I}}, \mathbf{y}_{0, \mathcal{I}}\right)\) and auxiliary samples \(\left\{\mathbf{Z}_{k}, \mathbf{y}_{k}\right\}_{k \in \widehat{G}_{l}}\). Denote the output as \(\hat{\gamma}\left(\widehat{G}_{l}\right)\) for \(0 \leq l \leq L\).
\STATE \underline{\textbf{Step 4.}} Compute $\hat{\theta}$ as in \eqref{eq:thetahat} for some \(\lambda_{\theta}>0\).
\STATE \underline{\textbf{Step 5.}} Calculate
\[
\hat{\gamma}^{\hat{\theta}}=\sum_{l=0}^{L} \hat{\theta}_{l} \hat{\gamma}\left(\widehat{G}_{l}\right).
\]
\STATE \textbf{Output}: $\hat\gamma_{\hat\theta}$.
\end{algorithmic}
\label{al:source_selection_alg}
\end{algorithm}

\section{Additional Figures}

\begin{figure}[H]
\centering
\includegraphics[width=.25\textwidth]{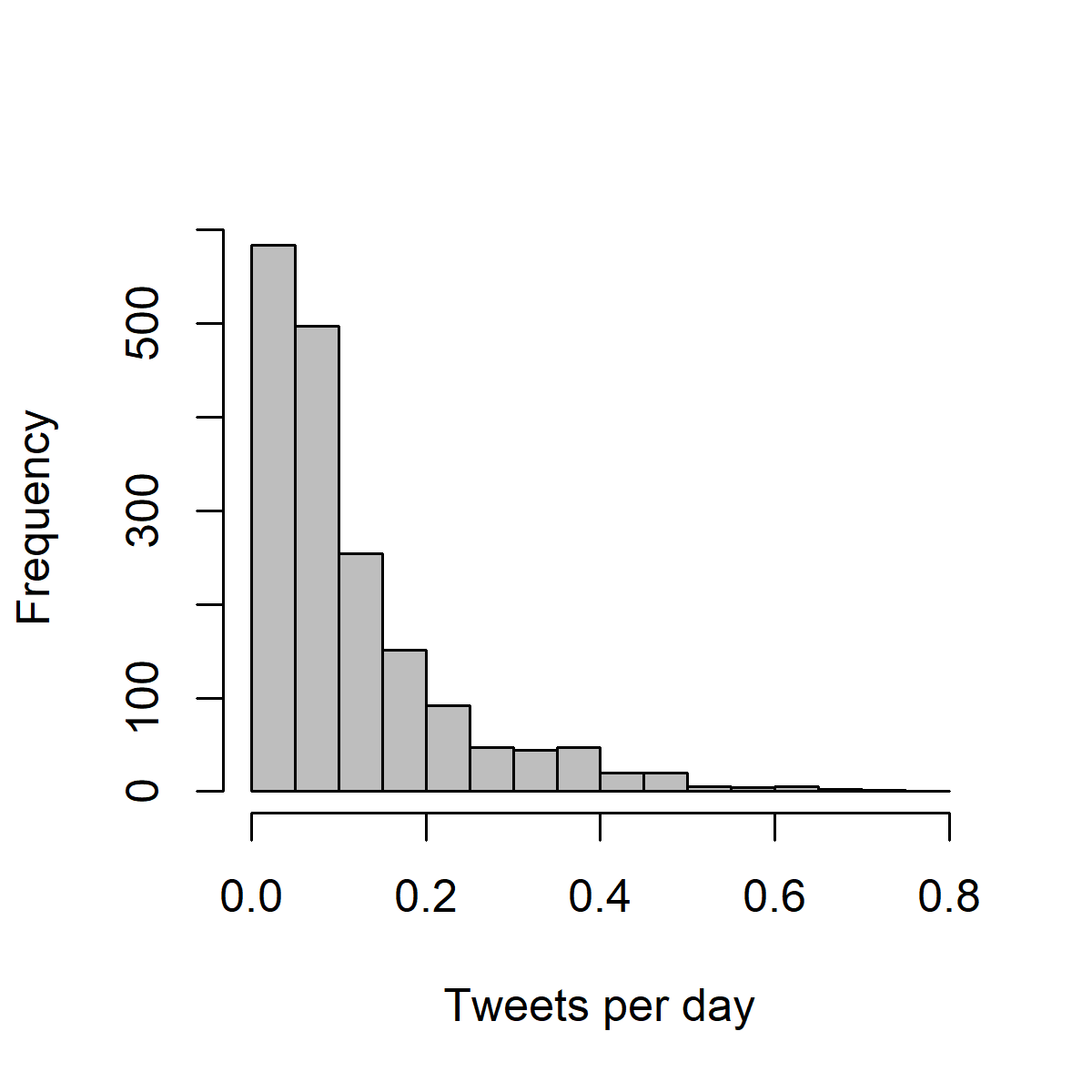}\hfill
\includegraphics[width=.25\textwidth]{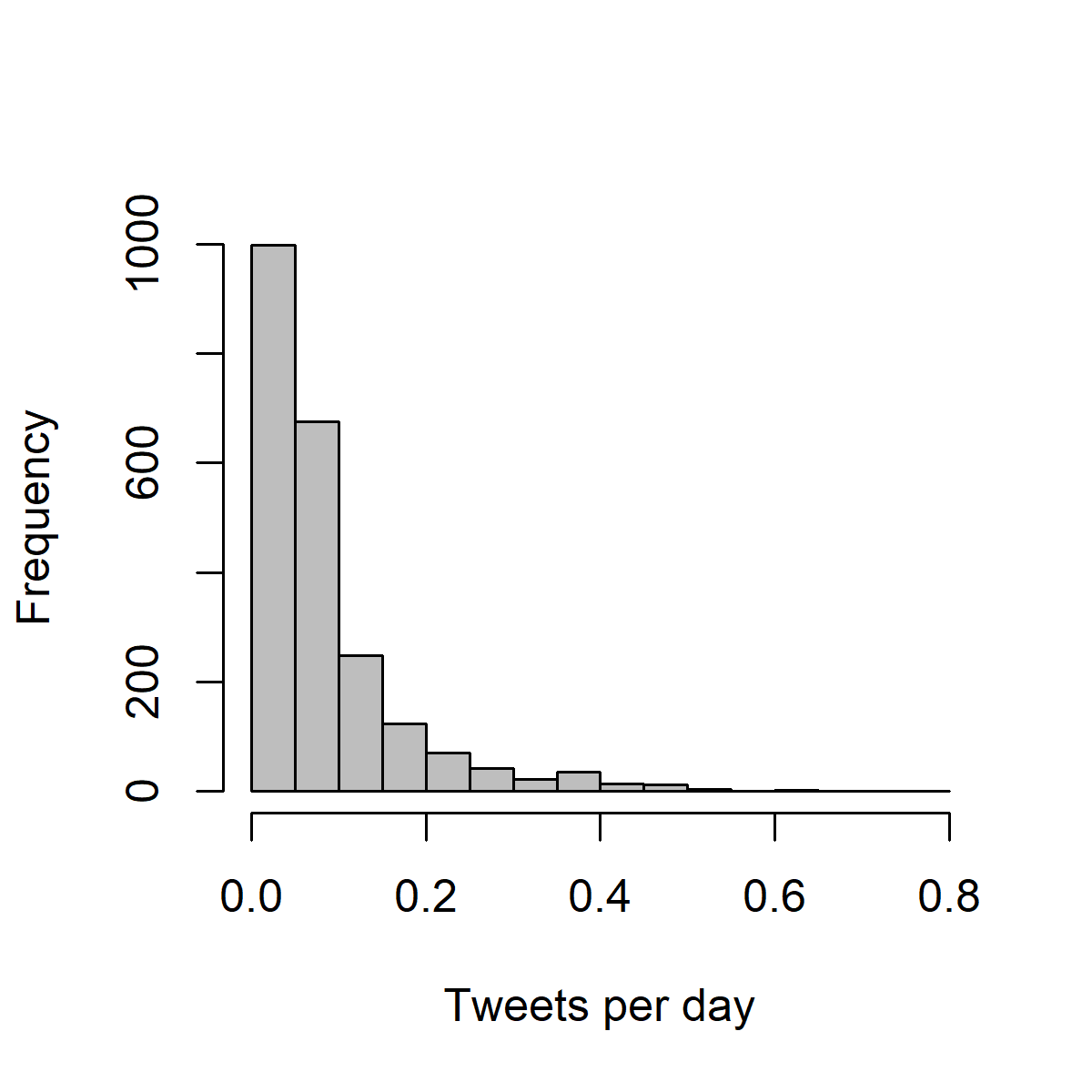}\hfill
\includegraphics[width=.25\textwidth]{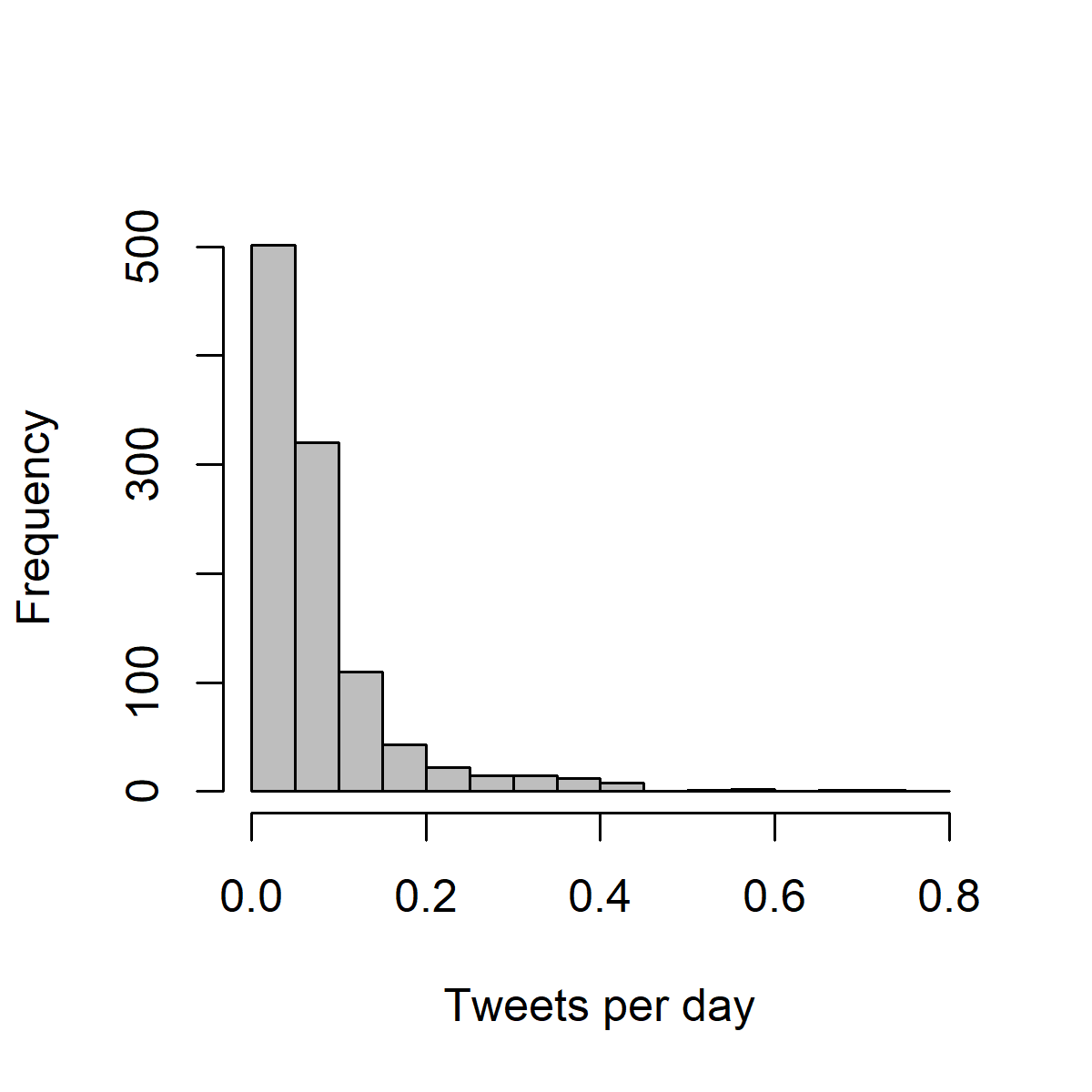}\hfill
\includegraphics[width=.25\textwidth]{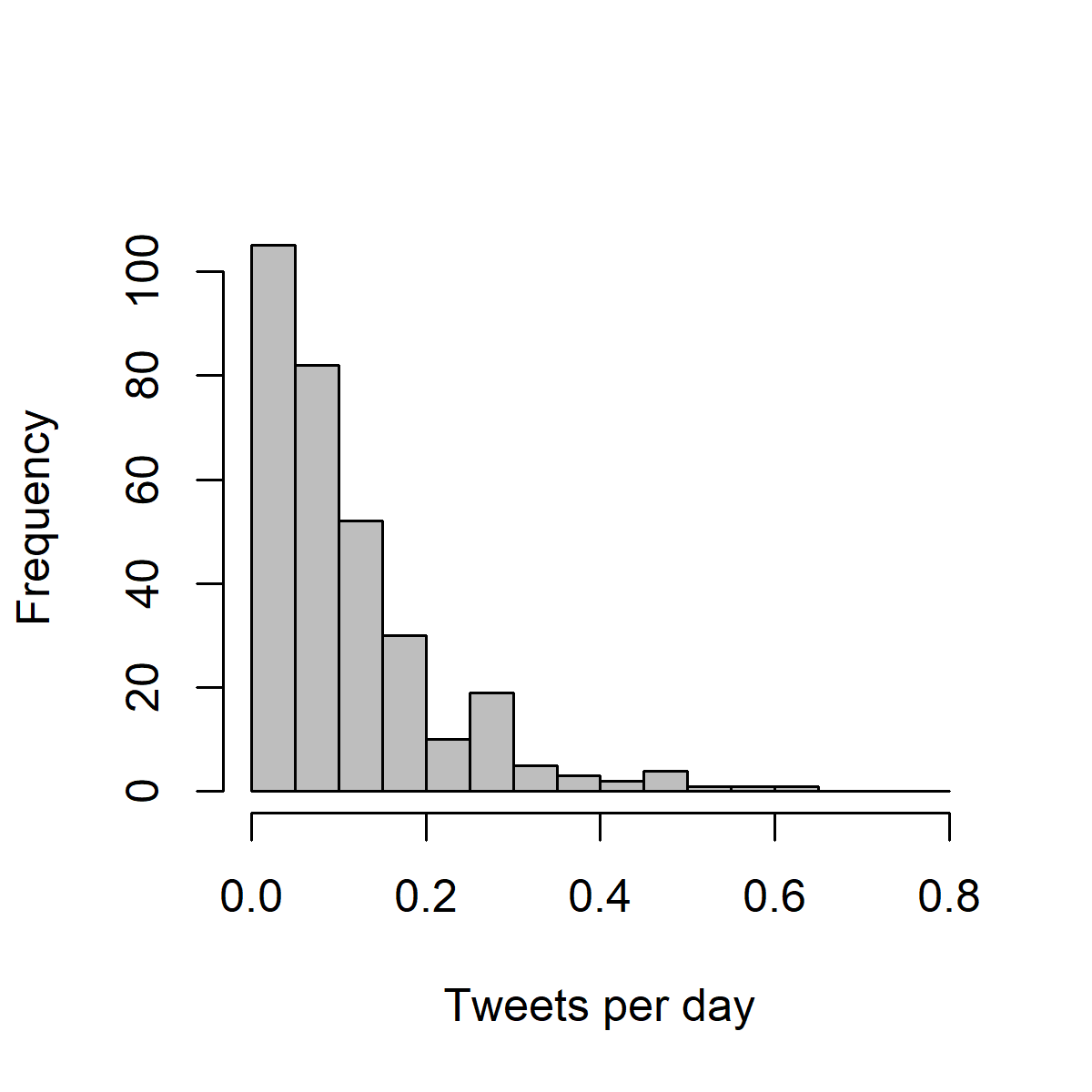}\hfill
\includegraphics[width=.25\textwidth]{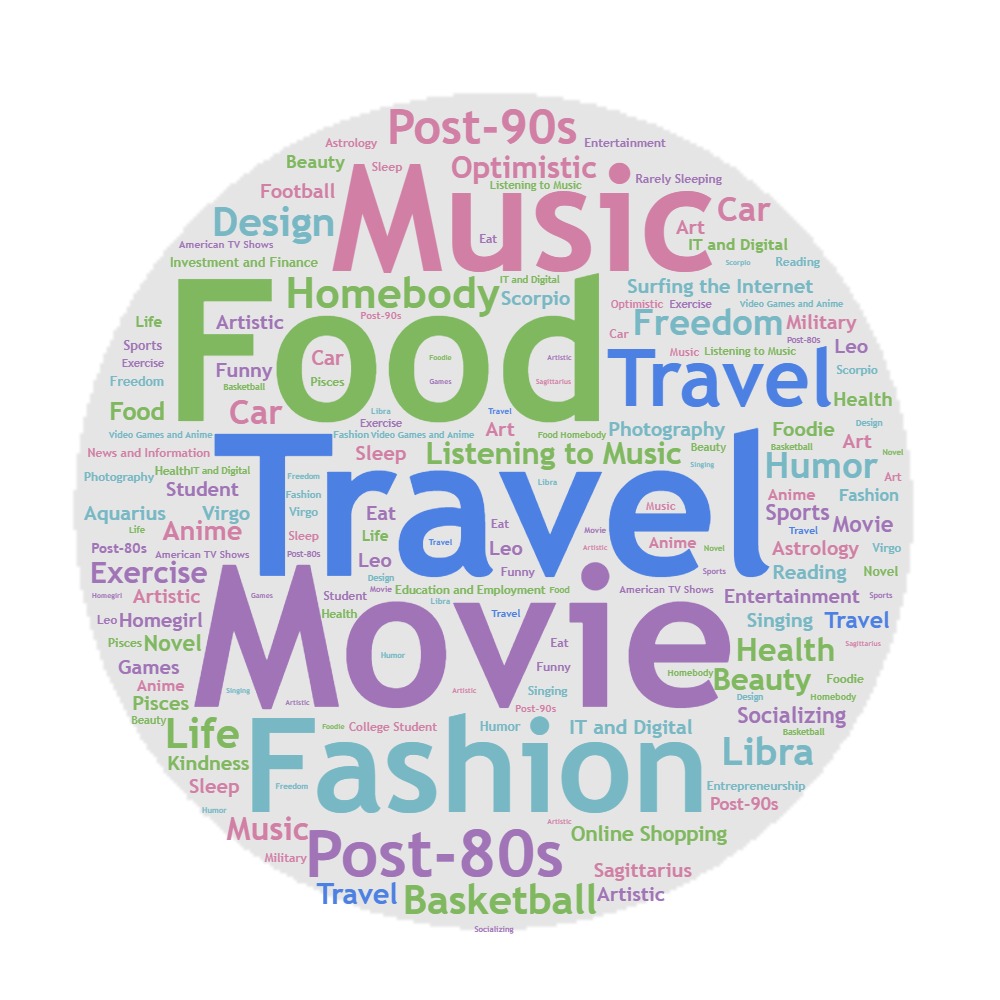}\hfill
\includegraphics[width=.25\textwidth]{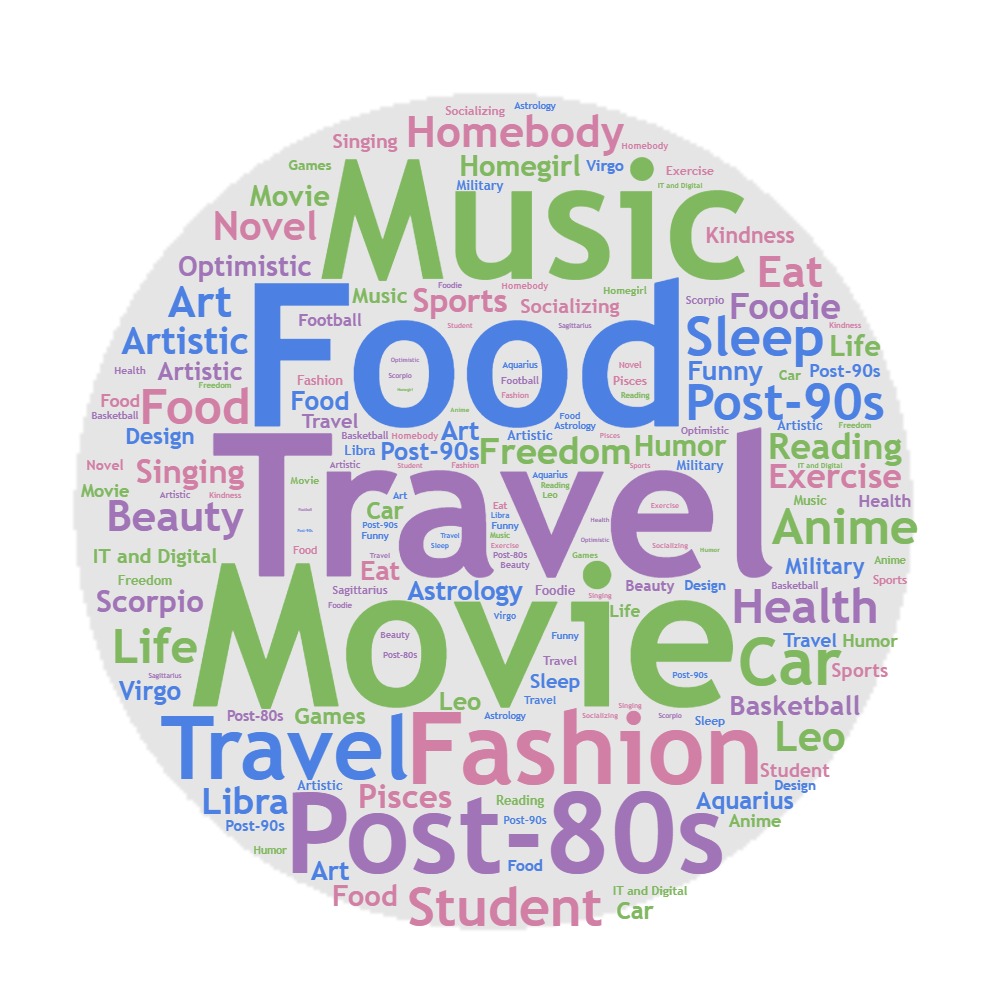}\hfill
\includegraphics[width=.25\textwidth]{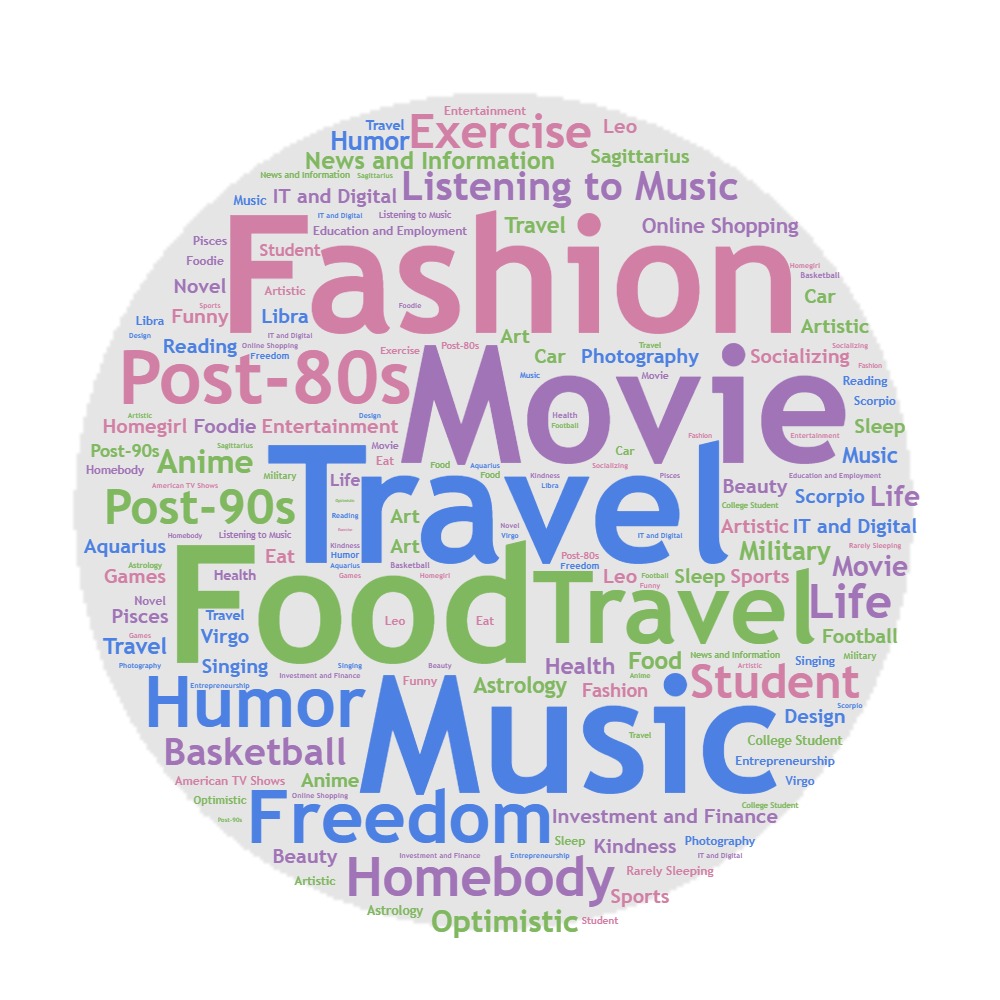}\hfill
\includegraphics[width=.25\textwidth]{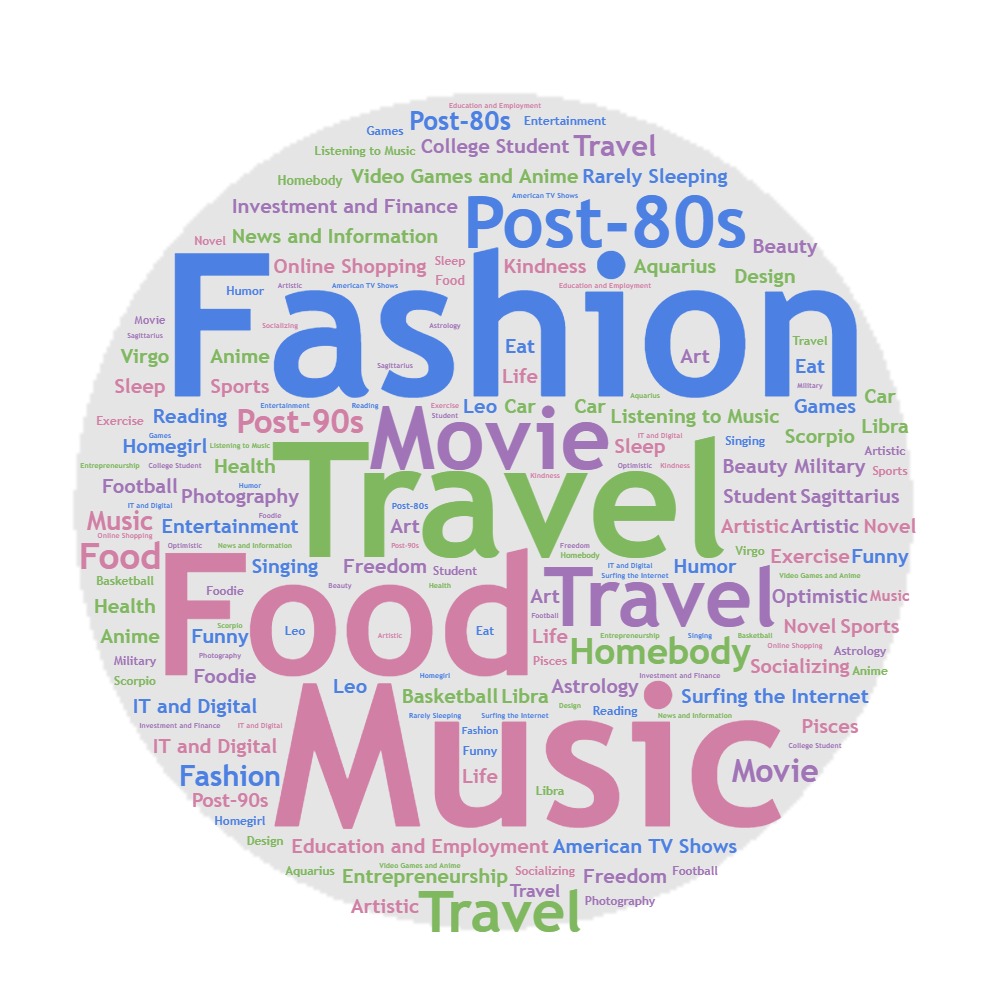}\hfill
\centering
\caption{\color{black}The upper panel displays histograms of tweets per day across different provinces, illustrating the frequency distribution. The lower panel presents word clouds representing user interests in each province, with word size indicating the relative frequency of each tag. From left to right: Beijing, Shanghai, Fujian, Liaoning.}\label{fig:hist}
\end{figure}

\bibliographystyle{apalike}
\bibliography{Mybib}

\end{document}